\documentclass{article}
\usepackage[T1]{fontenc}
\usepackage[utf8]{inputenc}
\usepackage[american]{babel}
\usepackage[nonatbib,final]{neurips_2019}

\usepackage{amsmath}
\usepackage{amsthm}
\usepackage{url}
\usepackage{booktabs}
\usepackage{amsfonts}
\usepackage{nicefrac}
\usepackage{microtype}
\usepackage{enumitem}
\usepackage{etoolbox}
\usepackage{amstext}
\usepackage{amssymb}
\usepackage{csquotes}
\usepackage{stmaryrd}
\usepackage{xparse}
\usepackage{bbm}
\usepackage{xcolor}

\usepackage{appendix}
\usepackage{graphicx}
\usepackage{epstopdf}
\usepackage{relsize}
\usepackage{subcaption}
\usepackage{xr-hyper}

\usepackage{algorithm}
\usepackage[noend]{algpseudocode}
\usepackage{setspace}

\makeatletter
\def\BState{\State\hskip-\ALG@thistlm}
\makeatother


\newcommand\mysim{\mathrel{\stackrel{\makebox[0pt]{\mbox{\normalfont\tiny i.i.d}}}{\sim}}}

\usepackage[textwidth=2cm, textsize=footnotesize]{todonotes}   

\usepackage[colorlinks,linkcolor={red!80!black},citecolor={green!80!black},urlcolor={blue!80!black},hypertexnames=false]{hyperref}

\usepackage[noabbrev,capitalize]{cleveref}
\crefformat{equation}{(#2#1#3)}
\crefrangeformat{equation}{(#3#1#4) to~(#5#2#6)}
\crefmultiformat{equation}{(#2#1#3)}{ and~(#2#1#3)}{, (#2#1#3)}{ and~(#2#1#3)}
\crefname{appsec}{Appendix}{Appendices}

\usepackage[backend=bibtex,style=numeric,maxcitenames=2,maxbibnames=8,sortcites,isbn=false,doi=false,giveninits=true]{biblatex}

\def\x{{\mathbf x}}
\def\L{{\cal L}}

\newcommand{\cP}{{{\mathcal P}}}

\newcommand{\cF}{{{\mathcal F}}}

\newcommand{\R}{{{\mathbb R}}} 
\newcommand{\E}{{{\mathbb E}}} 
\newcommand{\kH}{{{\mathcal H}}} 
\newcommand{\X}{{{\mathcal X}}} 
\newcommand{\F}{{{\mathcal F}}} 
\newcommand{\bR}{{{\mathbb R}}} 

\newcommand{\m}{\mathfrak{m}}

\newcommand{\ps}[1]{\langle #1 \rangle}

\newcommand{\agnote}[1]{\todo[color=red]{#1}}

\newcommand*\diff{\mathop{}\!\mathrm{d}}

\makeatother
\newtheorem{theorem}{Theorem}
\newtheorem{lemma}[theorem]{Lemma}

\newtheorem{proposition}[theorem]{Proposition}
\newtheorem{definition}{Definition}
\newtheorem{remark}{Remark}

\crefformat{equation}{(#2#1#3)}
\crefrangeformat{equation}{(#3#1#4) to~(#5#2#6)}
\crefmultiformat{equation}{(#2#1#3)}{ and~(#2#1#3)}{, (#2#1#3)}{ and~(#2#1#3)}

\newlist{assumplist}{enumerate}{1}
\setlist[assumplist]{label=(\textbf{\Alph*})}
\Crefname{assumplisti}{Assumption}{Assumptions}

\newlist{assumplist2}{enumerate}{1}
\setlist[assumplist2]{label=(\textbf{\Roman*})}
\Crefname{assumplist2i}{Assumption}{Assumptions}

\newlist{proplist}{enumerate}{1}
\setlist[proplist]{label=({\roman*})}
\Crefname{proplisti}{Property}{Properties}

\usepackage{filecontents}

\title{Maximum Mean Discrepancy Gradient Flow}

\author{
  Michael Arbel\\
  Gatsby Computational Neuroscience Unit\\University College London\\
  \texttt{michael.n.arbel@gmail.com}
  \And
  Anna Korba\\
  Gatsby Computational Neuroscience Unit\\University College London\\
  \texttt{a.korba@ucl.ac.uk}
  \And
  Adil Salim\\
  \phantom{xxxxxx} Visual Computing Center\phantom{xxxxxx}\\
  KAUST\\
  \texttt{adil.salim@kaust.edu.sa}
  \vspace*{-5mm}
  \And
  Arthur Gretton\\
  Gatsby Computational Neuroscience Unit\\University College London\\
  \texttt{arthur.gretton@gmail.com}
  \vspace*{-5mm}
}

\begin{document}
\maketitle

\begin{abstract}
  We construct a Wasserstein gradient flow of the maximum mean discrepancy (MMD) and study its convergence properties.
  The MMD is an integral probability metric defined for a reproducing kernel Hilbert space (RKHS), and serves as a metric on probability measures for a sufficiently rich RKHS.  We obtain conditions for convergence of the gradient flow towards a global optimum, that can be related to particle transport when optimizing neural networks.
  We also propose a way to regularize this MMD flow, based on an injection of noise in the gradient. This algorithmic fix comes with theoretical and empirical evidence.
The practical implementation of the flow is straightforward, since both the MMD and its gradient have simple closed-form expressions, which can be easily estimated with samples. 
  
\end{abstract}

\section{Introduction}

We address the problem of defining a gradient flow  on the space of probability distributions endowed with the Wasserstein metric, which transports probability mass from a starting distribtion $\nu$ to a target distribution $\mu$.   Our flow is defined on the maximum mean discrepancy (MMD) \cite{gretton2012kernel}, an integral probability metric \cite{Mueller97} which uses the unit ball in a characteristic RKHS \cite{sriperumbudur2010hilbert} as its witness function class.
Specifically, we choose the function in the witness class that has the largest difference in expectation under $\nu$ and $\mu$: this difference constitutes the MMD.
The idea of descending a gradient flow over the space of distributions can be traced back to the seminal work of \cite{jordan1998variational}, who revealed that the Fokker-Planck equation is a gradient flow of the Kullback-Leibler divergence. Its time-discretization leads to the celebrated Langevin Monte Carlo algorithm, which comes with strong convergence guarantees  (see \cite{durmus2018analysis,dalalyan2019user}), but requires the knowledge of an analytical form of the target $\mu$.
A more recent gradient flow approach, Stein Variational Gradient Descent (SVGD) \cite{liu2017stein}, also  leverages this analytical $\mu$.

The study of particle flows defined on the MMD relates to two important topics in modern machine learning. The first is in training Implicit Generative Models, notably generative adversarial networks \cite{gans}.  Integral probability metrics have been used extensively as critic functions in this setting: these include the Wasserstein distance \cite{towards-principled-gans,wgan-gp,sinkhorn-igm} and maximum mean discrepancy \cite{gen-mmd,Li:2015,Li:2017a,cramer-gan,Binkowski:2018,Arbel:2018}.  In \cite[Section 3.3]{Mroueh:2019}, a connection between IGMs and particle transport is proposed, where it is shown that gradient flow on the witness function of an integral probability metric takes a similar form to the generator update in a GAN. The critic IPM in this case is the Kernel Sobolev Discrepancy (KSD), which has an additional gradient norm constraint on the witness function compared with the MMD. It is intended as an approximation to the negative Sobolev distance from the optimal transport literature \cite{Otto:2000,Villani:2009,Peyre:2011}.  There remain certain differences between gradient flow and GAN training, however.  First, and most obviously, gradient flow can be approximated by representing $\nu$ as a set of particles, whereas in a GAN $\nu$ is the output of a generator network. The requirement that this generator network be a smooth function of its parameters causes a departure from pure particle flow. Second, in modern implementations \cite{Li:2017a,Binkowski:2018,Arbel:2018}, the kernel used in computing the critic witness function for an MMD GAN critic is parametrized by a deep network, and an alternating optimization between the critic parameters and the generator parameters is performed. Despite these differences, we anticipate that the theoretical study of MMD flow convergence will provide helpful insights into conditions for GAN convergence, and ultimately, improvements to GAN training algorithms.

Regarding the second topic, we note that the properties of gradient descent for large neural networks
have been modeled using the convergence towards a global optimum of particle transport in the population limit, when the number of particles goes to infinity  \cite{rotskoff2018neural,chizat2018global,mei2018mean,sirignano2018mean}. 
In particular, \cite{rotskoff2019global} show that gradient descent on the parameters of a neural network can also be seen as a particle transport problem, which has as its population limit a gradient flow of a  functional defined for probability distributions over the parameters of the network. 
This functional is in general non-convex, which makes the convergence analysis challenging.
The particular structure of the MMD allows us to relate its gradient flow to neural network optimization in a well-specified regression setting similar to \cite{rotskoff2019global,chizat2018global} (we make this connection explicit in \cref{subsec:training_neural_networks}).

Our main contribution in this work is to establish conditions for convergence of MMD gradient flow to its {\em global optimum}.
We give detailed descriptions of  MMD flow for both its continuous-time and discrete instantiations in \cref{sec:gradient_flow}.
In particular, the MMD flow may employ a sample approximation for the target $\mu$: unlike e.g. Langevin Monte Carlo or SVGD,
it does not require  $\mu$ in analytical form.
Global convergence is especially challenging to prove: while for functionals that are \textit{displacement convex}, the gradient flow can be shown to converge towards a global optimum \cite{ambrosio2008gradient}, the case of non-convex functionals, like the MMD, requires different tools.
 A modified gradient flow is proposed in \cite{rotskoff2019global} that uses particle birth and death to reach global optimality.
Global optimality may also be achieved simply by teleporting particles from $\nu$ to $\mu$, as occurs for the Sobolev Discrepancy flow absent a kernel regulariser \cite[Theorem 4, Appendix D]{Mroueh:2019}.
Note, however, that the regularised Kernel Sobolev Discrepancy flow does not rely on teleportation.

%

\iffalse
We note the claim of \cite[Proposition 3, Appendix B.1]{Mroueh:2019} that global convergence may be achieved for a flow on the Kernel Sobolev Discrepancy, with no barrier (the regularised KSD is related to the MMD; see Introduction). This claim requires an assumption \cite[Assumption A]{Mroueh:2019} that amounts to assuming that the algorithm is not in a local minimum, however.\agnote{maybe add a short appendix section - although  it's not urgent.}

\fi

%

Our approach takes  inspiration in particular from \cite{Bottou:2017}, where it is shown that although the $1$-Wasserstein distance is non-convex, it can be optimized up to some barrier that depends on the diameter of the domain of the target distribution.
Similarly to \cite{Bottou:2017}, we provide in \cref{sec:convergence_mmd_flow} a barrier on the gradient flow of the MMD, although
the tightness of this barrier in terms of the target diameter remains to be established.
We  obtain a further condition on the evolution of the flow to ensure global optimality, and give rates of convergence in that case, however
the condition is a strong one: it implies that the negative Sobolev distance between the target and the current particles remains bounded at all times.

We thus propose a way to regularize the MMD flow, based on a noise injection (Section \ref{sec:discretized_flow}) in the gradient, with more tractable theoretical conditions for convergence. Encouragingly, the
noise injection is shown in practice to ensure convergence in a simple illustrative case where the original MMD flow fails.
Finally, while our emphasis has been on establishing conditions for convergence, we note that MMD gradient flow has a  simple
$O(MN+N^2)$ implementation for $N$ $\nu$-samples and $M$ $\mu$-samples, and requires only evaluating the gradient of the kernel $k$ on the given samples.

\section{Gradient flow of the MMD in $W_2$}\label{sec:gradient_flow}
\subsection{Construction of the gradient flow}
In this section we introduce the gradient flow of the Maximum Mean Discrepancy (MMD) and highlight some of its properties. We start by briefly reviewing the MMD introduced in \cite{gretton2012kernel}. %
We define $\X\subset\R^d$ as the closure of a convex open set, and $\mathcal{P}_2(\X)$ as the set of probability distributions on $\X$ with finite second moment, equipped with the 2-Wassertein metric denoted $W_2$. For any $\nu \in \mathcal{P}_2(\X)$, $L_2(\nu)$ is the set of square integrable functions w.r.t. $\nu$. The reader may find a relevant mathematical background in \cref{sec:appendix_math_background}. 

\paragraph{Maximum Mean Discrepancy.}\label{subsec:MMD} Given a characteristic kernel $k : \X \times \X \to \bR$, we denote by $\kH$ its corresponding RKHS (see \cite{smola1998learning}). The space $\kH$ is a Hilbert space with inner product $\langle .,. \rangle_{\kH}$ and norm $\Vert . \Vert_{\kH}$. We will rely on specific assumptions on the kernel which are given in \cref{sec:assumptions_kernel}. In particular, \cref{assump:lipschitz_gradient_k} states that the gradient of the kernel, $\nabla k$, is Lipschitz with constant $L$.
For such kernels, it is possible to define the Maximum Mean Discrepancy as a distance on $\mathcal{P}_2(\X)$. The MMD can be written as the RKHS norm of the unnormalised \textit{witness function} $f_{\mu,\nu}$ between $\mu$ and $\nu$, which is the difference between the mean embeddings of $\nu$ and $\mu$,
\begin{align}\label{eq:witness_function}
MMD(\mu,\nu) = \Vert f_{\mu,\nu} \Vert_{\kH}, \qquad f_{\nu,\mu}(z) = \int k(x,z)\diff \nu(x) - \int k(x,z)\diff \mu(x)  \quad \forall z\in \X
\end{align}
Throughout the paper, $\mu$ will be fixed and $\nu$ can vary, hence we will only consider the dependence in $\nu$ and denote by $\F(\nu)= \frac{1}{2}MMD^2(\mu,\nu)$. 
A direct computation \cite[Appendix B]{Mroueh:2019} shows that for any %
finite measure $\chi$ such that $\nu + \epsilon \chi \in \cP_2(\X)$, we have
\begin{align}\label{prop:differential_mmd}
		\lim_{\epsilon\rightarrow 0} \epsilon^{-1}(\F(\nu +\epsilon\chi) - \F(\nu)) = \int f_{\mu,\nu}(x)d\chi(x).
	\end{align}
This means that $f_{\mu,\nu}$ is the differential of $\F(\nu)$ .  %
 Interestingly, $\F(\nu)$ admits a \textit{free-energy} expression:
\begin{align}\label{eq:mmd_as_free_energy}
	\F(\nu) = \int V(x) \diff \nu(x) +\frac{1}{2} \int W(x,y)\diff \nu(x)\diff \nu(y) + C.
\end{align}
where $V$ is a confinement potential, $W$ an interaction potential and $C$ a constant defined by:
\begin{align}\label{eq:potentials}
V(x)=-\int  k(x,x')\diff\mu(x'), \quad
W(x,x')=k(x,x'), \quad
C = \frac{1}{2} \int k(x,x')\diff\mu(x) \diff\mu(x') 
\end{align}
Formulation \cref{eq:mmd_as_free_energy} and the simple expression of the differential in \cref{prop:differential_mmd} will be key to construct a gradient flow of $\F(\nu)$, to transport particles. In \cref{eq:potentials}, $V$ reflects the potential generated by $\mu$ and acting on each particle, while $W$ reflects the potential arising from the interactions between those particles.  %
\paragraph{Gradient flow of the MMD.}\label{paragraph:flow_MMD}
We consider now the problem of transporting mass from an initial distribution $\nu_0$ to a target distribution $\mu$, by finding a continuous path $\nu_t$ starting from $\nu_0$ that  converges to $\mu$ while decreasing $\F(\nu_t)$. Such a path should be physically plausible, in that  teleportation phenomena are not allowed. For instance, the path $\nu_t = (1-e^{-t})\mu + e^{-t}\nu_0$ would constantly teleport mass between $\mu$ and $\nu_0$ although it decreases  $\F$ since $\F(\nu_t)=e^{-2t}\F(\nu_0)$
\cite[Section 3.1, Case 1]{Mroueh:2019}.
The physicality of the path is understood in terms of classical statistical physics: given an initial configuration $\nu_0$ of $N$ particles, these can move towards a new configuration $\mu$ through successive small transformations, without jumping from one location to another. %

Optimal transport theory provides a way to construct such a continuous path by means of the \textit{continuity equation}. Given a vector field $V_t$ on $\X$ and an initial condition $\nu_0$, the continuity equation is a partial differential equation which defines a path $\nu_t$ evolving under the action of the vector field $V_t$, and reads $\partial_t \nu_t = -div(\nu_t V_t)$ for all $t \geq 0$.
The reader can find more detailed discussions in \cref{subsec:wasserstein_flow} or \cite{Santambrogio:2015}. Following  \cite{ambrosio2008gradient}, a natural choice is to choose $V_t$ as the negative gradient of the differential of $\F(\nu_t)$ at $\nu_t$, since it corresponds to a gradient flow of $\F$ associated with the $W_2$ metric (see \cref{subsec:gradient_flows_functionals}). %
By \cref{prop:differential_mmd}, we know that the differential of $\F(\nu_t)$  at $\nu_t$ is given by $f_{\mu,\nu_t}$, hence $V_t(x) = -\nabla f_{\mu,\nu_t}(x)$.\footnote{Also, $V_t=\nabla V+\nabla W \star \nu_t$ (see \cref{subsec:gradient_flows_functionals}) where $\star$ denotes the classical convolution.} The
gradient flow of $\F$ is then defined by the solution $(\nu_t)_{t\geq 0}$ of
\begin{align}\label{eq:continuity_mmd}
	\partial_t \nu_t = div(\nu_t \nabla f_{\mu,\nu_t}).
\end{align}
Equation \cref{eq:continuity_mmd} is non-linear in that the vector field depends itself on $\nu_t$. This type of equation is associated in the probability theory literature to the so-called McKean-Vlasov process \cite{kac1956foundations,mckean1966class},
\begin{align}\label{eq:mcKean_Vlasov_process}
	d X_t = -\nabla f_{\mu,\nu_t}(X_t)dt \qquad X_0\sim \nu_0.
\end{align}

In fact,  \cref{eq:mcKean_Vlasov_process} defines a process $(X_t)_{t\geq 0}$ whose distribution $(\nu_t)_{t\geq 0}$ satisfies \cref{eq:continuity_mmd}, as shown in \cref{prop:existence_uniqueness}. 
$(X_t)_{t\geq 0}$ can be interpreted as the trajectory of a single particle, starting from an initial random position $X_0$ drawn from $\nu_0$. The trajectory is driven by the velocity field $-\nabla f_{\mu,\nu_t}$, and is affected by other particles. These interactions are captured by the velocity field through the dependence on the current distribution $\nu_t$ of all particles.
Existence and uniqueness of a solution to \cref{eq:continuity_mmd,eq:mcKean_Vlasov_process} are guaranteed in the next proposition, whose proof is given \cref{proof:prop:existence_uniqueness}.
\begin{proposition}\label{prop:existence_uniqueness}
	Let $\nu_0\in \mathcal{P}_2(\X)$. %
	Then, under \cref{assump:lipschitz_gradient_k}, there exists a unique process $(X_t)_{t\geq 0}$  satisfying the McKean-Vlasov equation in \cref{eq:mcKean_Vlasov_process} such that $X_0 \sim \nu_0$. Moreover, the distribution $\nu_t$ of $X_t$ is the unique solution of \cref{eq:continuity_mmd} starting from $\nu_0$, and defines a gradient flow of $\F$. 
\end{proposition}
Besides existence and uniqueness of the gradient flow of $\F$, one expects $\F$ to decrease along the path $\nu_t$ and ideally to converge towards $0$. The first property, stated in the next proposition, is rather easy to get and is the object of \cref{prop:decay_mmd}, similar to the result for KSD flow in \cite[Section 3.1]{Mroueh:2019}.
\begin{proposition}\label{prop:decay_mmd}
	Under \cref{assump:lipschitz_gradient_k}, $\F(\nu_t)$ is decreasing in time and satisfies:
	\begin{align}\label{eq:time_evolution_mmd}
		\frac{d\F(\nu_t)}{dt}= - \int \Vert \nabla f_{\mu,\nu_t}(x) \Vert^2 \diff \nu_t(x).  
	\end{align}
\end{proposition}
This property results from \cref{eq:continuity_mmd} and the energy identity in \cite[Theorem 11.3.2]{ambrosio2008gradient} and is proved in \cref{proof:prop:decay_mmd}. %
From \cref{eq:time_evolution_mmd}, $\F$ can be seen as a Lyapunov functional for the dynamics defined by \cref{eq:continuity_mmd}, since it is decreasing in time. Hence, the continuous-time gradient flow introduced in \cref{eq:continuity_mmd} allows to formally consider the notion of gradient descent on $\mathcal{P}_2(\X)$ with $\F$ as a cost function.
A time-discretized version of the flow naturally follows, and is provided in the next section.
\subsection{Euler scheme}
We consider here a forward-Euler scheme of \cref{eq:continuity_mmd}. For any $T: \X \rightarrow \X$ a measurable map, and $\nu \in \mathcal{P}_2(\X)$, we denote the pushforward measure by $T_{\#}\nu$ (see \cref{subsec:wasserstein_flow}). Starting from $\nu_0\in \mathcal{P}_2(\X)$ and using a step-size $\gamma>0$, a sequence $\nu_n\in \mathcal{P}_2(\X)$ is given by iteratively applying
\begin{align}\label{eq:euler_scheme}
	\nu_{n+1} = (I - \gamma \nabla f_{\mu,\nu_n})_{\#}\nu_n.
\end{align} 
For all $n\ge 0$, equation \cref{eq:euler_scheme} is the distribution of the process defined by
\begin{align}\label{eq:euler_scheme_particles}
	X_{n+1} = X_n - \gamma \nabla f_{\mu,\nu_n}(X_n) \qquad X_0\sim \nu_0.
\end{align}
The asymptotic behavior of \cref{eq:euler_scheme} as $n\rightarrow \infty$ %
will be the object of \cref{sec:convergence_mmd_flow}. For now, we provide a guarantee that the sequence $(\nu_n)_{n\in \mathbb{N}}$ approaches $(\nu_t)_{t\geq 0}$ as the step-size $\gamma\rightarrow 0$.
\begin{proposition}\label{prop:convergence_euler_scheme}
	Let $n\ge0$. Consider $\nu_n$ defined in \cref{eq:euler_scheme}, and the interpolation path $\rho_t^{\gamma}$ defined as: $\rho_t^{\gamma} = (I-(t- n\gamma) \nabla f_{\mu,\nu_n})_{\#}\nu_n$, $\forall t\in [n\gamma,(n+1)\gamma)$. Then, under \cref{assump:lipschitz_gradient_k},  $\forall \;T>0$,%
	\begin{align}
		W_2(\rho_t^{\gamma},\nu_t)\leq \gamma C(T) \quad \forall t\in [0,T]
	\end{align}
	where $C(T)$ is a constant that depends only on $T$.
\end{proposition} 
 A proof of \cref{prop:convergence_euler_scheme} is provided in \cref{proof:prop:convergence_euler_scheme} and relies on standard techniques to control the discretization error of a forward-Euler scheme.
\cref{prop:convergence_euler_scheme} means that $\nu_n$ can be linearly interpolated giving rise to a path $\rho_t^{\gamma}$ which gets arbitrarily close to $\nu_t$ on bounded intervals. Note that as $T \rightarrow \infty$ the bound $C(T)$ it is expected to blow up. However, this result is enough to show that \cref{eq:euler_scheme} is indeed a discrete-time flow of $\F$. In fact, provided that $\gamma$ is small enough, $\F(\nu_n)$ is a decreasing sequence, as shown in \cref{prop:decreasing_functional}.
\begin{proposition}\label{prop:decreasing_functional}
 	Under \cref{assump:lipschitz_gradient_k}, and for $\gamma \leq 2/3L$, the sequence $\F(\nu_n)$ is decreasing, and
	\begin{align*}
	\F(\nu_{n+1})-\F(\nu_n)\leq -\gamma (1-\frac{3\gamma}{2}L )\int \Vert \nabla f_{\mu, \nu_n}(x)\Vert^2 \diff \nu_n(x), \quad\forall n\geq 0.
	\end{align*}
\end{proposition}
\cref{prop:decreasing_functional}, whose proof is given in \cref{proof:prop:decreasing_functional}, is a discrete analog of \cref{prop:decay_mmd}. %
In fact, \cref{eq:euler_scheme} is intractable in general as it requires the knowledge of $\nabla f_{\mu,\nu_n}$ (and thus of $\nu_n$) exactly at each iteration $n$. Nevertheless, we present in \cref{sec:sample_based} a practical algorithm using a finite number of samples which is provably convergent towards \cref{eq:euler_scheme} as the sample-size increases. We thus begin by studying the convergence properties of the time discretized MMD flow \cref{eq:euler_scheme} in the next section.%

\section{Convergence properties of the MMD flow}\label{sec:convergence_mmd_flow}
We are interested in analyzing the asymptotic properties of the gradient flow of $\F$. %
Although we know from \cref{prop:decay_mmd,prop:decreasing_functional} that $\F$ decreases in time, it can very well  converge to local minima. One way to see this is by looking at the equilibrium condition for \cref{eq:time_evolution_mmd}. As a non-negative and decreasing function, $t \mapsto \F(\nu_t)$  is guaranteed to converge towards a finite limit $l\ge0$, which implies in turn that the r.h.s. of \cref{eq:time_evolution_mmd} converges to $0$. If $\nu_t$ happens to converge towards some distribution $\nu^{*}$, %
it is possible to show that the equilibrium condition \cref{eq:equilibrium_condition} must hold \cite[Prop. 2]{mei2018mean}
,
\begin{align}\label{eq:equilibrium_condition}
\int \left\Vert \nabla f_{\mu,\nu^{*}}(x) \right\Vert^2 \diff \nu^{*}(x) =0.  
\end{align}
Condition \cref{eq:equilibrium_condition} does not necessarily imply that $\nu^{*}$ is a global optimum unless when the loss function has a particular structure \cite{Chizat:2018a}. For instance, this would hold if the kernel is linear in at least one of its dimensions.
However, when a characteristic kernel is required (to ensure the MMD is a distance), such a structure can't be exploited.
Similarly, the claim that KSD flow converges globally,  \cite[Prop. 3, Appendix B.1]{Mroueh:2019}, requires an assumption \cite[Assump. A]{Mroueh:2019} that excludes local minima which are not global (see \cref{subsec:equilibrium_condition}; recall  KSD is related to MMD).
Global convergence of the flow is harder to obtain, and will be the topic of this section. The main challenge is the lack of convexity of $\F$ w.r.t. the Wassertein metric. %
We show that $\F$ is merely $\Lambda$-convex, and that standard optimization techniques  only provide a loose bound on its asymptotic value.
We next exploit a Lojasiewicz type inequality to prove  convergence to the global optimum provided that a particular quantity remains bounded at all times. 

\subsection{Optimization in a ($W_2$) non-convex setting}
\label{subsection:barrier_optimization}
The \textit{displacement convexity} of a functional $\F$ is an important criterion in characterizing the convergence of its Wasserstein gradient flow.
Displacement convexity states that $t\mapsto \F(\rho_t)$ is a convex function whenever $(\rho_t)_{t\in[0,1]}$ is a path of minimal length between two distributions $\mu$ and $\nu$ (see \cref{def:displacement_convexity}). %
Displacement convexity should not be confused with \textit{mixture convexity}, which corresponds to the usual notion of convexity. As a matter of fact, $\F$ is mixture convex in that it satisfies: $\F(t\nu +(1-t)\nu')\leq t\F(\nu)+(1-t)\F(\nu')$ for all $t\in [0,1]$ and $\nu,\nu'\in\mathcal{P}_2(\X)$ (see \cref{lem:mixture_convexity}). Unfortunately, \textit{$\F$ is not displacement convex}. Instead, $\F$ only satisfies a weaker notion of displacement convexity called $\Lambda$-displacement convexity, given in \cref{def:lambda-convexity} (\cref{subsec:lambda_convexity}). 
\begin{proposition}
	\label{prop:lambda_convexity} Under \cref{assump:diff_kernel,assump:lipschitz_gradient_k,assump:bounded_fourth_oder}, $\F$ is $\Lambda$-displacement convex, and satisfies
		\begin{equation}
	\F(\rho_{t})\leq(1-t)\F(\nu)+t\F(\nu')-\int_0^1 \Lambda(\rho_s, v_s ) G(s,t)\diff s
	\end{equation}
for all $\nu, \nu'\in \mathcal{P}_2(\X)$ and any \textit{displacement geodesic} $(\rho_t)_{t\in[0,1]}$ from $\nu$ to $\nu'$ with velocity vectors $(v_t)_{t \in[0,1]}$. The functional $\Lambda$ is defined for any pair $(\rho,v)$ with $\rho\in \mathcal{P}_2(\X)$ and $\Vert v\Vert \in L_2(\rho)$,   
	\begin{align}\label{eq:lambda}
		\Lambda(\rho,v) = \left\Vert \int v(x).\nabla_x k(x,.) \diff \rho(x) \right\Vert^2_{\mathcal{H}} - \sqrt{2}\lambda d \F(\rho)^{\frac{1}{2}}  \int \left\Vert  v(x) \right\Vert^2 \diff \rho(x),
	\end{align}
where $(s,t)\mapsto G(s,t)=  s(1-t) \mathbbm{1}\{s\leq t\}+ t(1-s) \mathbbm{1}\{s\geq t\}$ and $\lambda$ is defined in \cref{assump:bounded_fourth_oder}.%
	\end{proposition}
\cref{prop:lambda_convexity} can be obtained by computing the second time derivative of $\F(\rho_t)$, which is then lower-bounded by $\Lambda(\rho_t,v_t)$ (see \cref{proof:prop:lambda_convexity}).
In \cref{eq:lambda}, the map $\Lambda$ is a difference of two non-negative terms: thus $\int_0^1 \Lambda(\rho_s, v_s ) G(s,t)\diff s$ can become negative, and displacement convexity does not hold in general. 
\cite[Theorem 6.1]{Carrillo:2006} provides a convergence when only $\Lambda$-displacement convexity holds as long as either the potential or the interaction term is convex enough. In fact, as mentioned in \cite[Remark 6.4]{Carrillo:2006}, the convexity of either term could compensate for a lack of convexity of the other.
Unfortunately, this cannot be applied for MMD since both terms involve the same kernel but with opposite signs. Hence, even under convexity of the kernel, a concave term appears and cancels the effect of the convex term. Moreover, the requirement that the kernel be positive semi-definite makes it hard to construct interesting convex kernels.
However, it is still possible to provide an upper bound on the asymptotic value of $\F(\nu_n)$ when $(\nu_n)_{n \in \mathbb{N}}$ are obtained using \cref{eq:euler_scheme}. This bound is given in \cref{th:rates_mmd}, and depends on a scalar $ K(\rho^n) :=  \int_0^1\Lambda(\rho_s^n,v_s^n)(1-s)\diff s$, where $(\rho_s^n)_{s\in[0,1]}$ is a \textit{constant speed displacement geodesic} from $\nu_n$ to the optimal value $\mu$, with velocity vectors $(v_s^n)_{s \in [0,1]}$ of  constant norm.  
\begin{theorem}\label{th:rates_mmd}
	Let $\bar{K}$ be the average of $(K(\rho^j))_{0\leq j \leq n}$. %
	 Under \cref{assump:diff_kernel,assump:lipschitz_gradient_k,assump:bounded_fourth_oder} and if $\gamma \leq 1/3L$,%
\begin{align}
\F(\nu_n) \leq  \frac{W_2^2(\nu_0,\mu)}{2 \gamma n} -\bar{K}.
\end{align}
\end{theorem}
\cref{th:rates_mmd} is obtained using techniques from optimal transport and optimization. It relies on \cref{prop:lambda_convexity} and \cref{prop:decreasing_functional} to prove an \textit{extended variational inequality} (see \cref{prop:evi}), and concludes using a suitable Lyapunov function. A full proof is given in \cref{proof:th:rates_mmd}.
When $\bar{K}$ is non-negative, one recovers the usual convergence rate as $O(\frac{1}{n})$ for the gradient descent algorithm. However, $\bar{K}$ can be negative in general, and would therefore act as a barrier on the optimal value that $\F(\nu_n)$ can achieve when $n\rightarrow \infty$. In that sense, the above result is similar to \cite[Theorem 6.9]{Bottou:2017}. 
 \cref{th:rates_mmd} only provides a loose bound, however. In \cref{sec:Lojasiewicz_inequality} we show global convergence, under the boundedness at all times $t$ of a specific distance between $\nu_t$ and $\mu$.

\subsection{A condition for global convergence}\label{sec:Lojasiewicz_inequality}%
The lack of convexity of $\F$, as shown in \cref{subsection:barrier_optimization}, suggests that a finer analysis of the convergence should be performed. One strategy is to provide estimates for the dynamics in \cref{prop:decay_mmd} using differential inequalities which can be solved using the Gronwall's lemma (see \cite{oguntuase2001inequality}). Such inequalities  are known in the optimization literature as Lojasiewicz inequalities (see \cite{Bolte:2016}), and upper-bound $\F(\nu_t)$ by the absolute value of its time derivative $\int \Vert \nabla f_{\mu,\nu_t}(x) \Vert^2 \diff \nu_t(x)$.
The latter is the squared \textit{weighted Sobolev semi-norm} of $f_{\mu,\nu_t}$ (see \cref{subsection:Lojasiewicz}), also written  $\Vert f_{\mu,\nu_t} \Vert_{\dot{H}(\nu_t)}$. Thus one needs to find a relationship between $\F(\nu_t) = \frac{1}{2} \Vert f_{\mu,\nu_t} \Vert_{\mathcal{H}}^2 $ and $\Vert f_{\mu,\nu_t} \Vert_{\dot{H}(\nu_t)}$. For this purpose, we consider the \textit{weighted negative Sobolev distance} on $\mathcal{P}_2(\X)$, defined by duality using $\Vert . \Vert_{\dot{H}(\nu)}$ (see also \cite{Peyre:2011}).
\begin{definition}\label{def:neg_sobolev}
	Let $\nu\in \mathcal{P}_2(\x)$, with its corresponding \textit{weighted Sobolev semi-norm} $ \Vert . \Vert_{\dot{H}(\nu)} $. %
	The \textit{weighted negative Sobolev distance} $\Vert p - q \Vert_{\dot{H}^{-1}(\nu)}$ between any $p$ and $q$ in $\mathcal{P}_2(\x)$  is defined as
\begin{align}\label{eq:neg_sobolev}
	\Vert p - q \Vert_{\dot{H}^{-1}(\nu)} = \sup_{f\in L_2(\nu), \Vert f \Vert_{\dot{H}(\nu)} \leq 1 } \left\vert \int f(x)\diff p(x) - \int f(x)\diff q(x) \right\vert 
\end{align}	
with possibly infinite values.
\end{definition}
Equation \cref{eq:neg_sobolev} plays a fundamental role in dynamic optimal transport. %
It can be seen as the minimum kinetic energy needed to advect the mass $\nu$ to $q$ (see \cite{Mroueh:2019}). It is shown in \cref{proof:prop:lojasiewicz} that
\begin{align}\label{eq:inequality_neg_sobolev}
	\Vert f_{\mu,\nu_t} \Vert^2_{\mathcal{H}} \leq \Vert f_{\mu,\nu_t} \Vert_{\dot{H}(\nu_t)} \Vert  \mu -\nu_t\Vert_{\dot{H}^{-1}(\nu_t)}  .
\end{align}
Provided that $\Vert \mu - \nu_t \Vert_{\dot{H}^{-1}(\nu_t)} $ remains bounded by some positive constant $C$ at all times, \cref{eq:inequality_neg_sobolev} leads to  a functional version of Lojasiewicz inequality for $\F$. 
It is then possible to use the general strategy explained earlier to prove the convergence of the flow to a global optimum:
\begin{proposition}\label{prop:lojasiewicz}
		Under \cref{assump:lipschitz_gradient_k},
		\begin{proplist}
			\item If $\Vert \mu - \nu_t \Vert^2_{\dot{H}^{-1}(\nu_t)} \leq C$, for all $t\geq 0$, then: $\mathcal{F}(\nu_t)\leq \frac{C}{C\mathcal{F}(\nu_0)^{-1} + 4t}$,
			\item If $\Vert \mu - \nu_n \Vert^2_{\dot{H}^{-1}(\nu_n)} \leq C$ for all $n\geq 0$, then: $\mathcal{F}(\nu_n)\leq \frac{C}{C\mathcal{F}(\nu_0)^{-1} + 4\gamma(1-\frac{3}{2}\gamma L) n}$.
		\end{proplist}
\end{proposition}
Proofs of \cref{prop:lojasiewicz}  (i) and (ii)  are direct consequences of \cref{prop:decay_mmd,prop:decreasing_functional} and the bounded energy assumption: see  \cref{proof:prop:lojasiewicz}. The fact that \cref{eq:neg_sobolev} appears in the context of Wasserstein flows of $\F$ is not a coincidence. Indeed, \cref{eq:neg_sobolev} is a linearization of the Wasserstein distance (see \cite{Peyre:2011,Otto:2000} and \cref{subsec:Lojasiewicz_different_metrics}). Gradient flows of $\F$ defined under different metrics would involve other kinds of distances instead of \cref{eq:neg_sobolev}. 
For instance, \cite{rotskoff2019global} consider gradient flows under a hybrid metric (a mixture between the Wasserstein distance and KL divergence), where convergence rates can then be obtained provided that the chi-square divergence $\chi^2(\mu\Vert \nu_t)$ remains bounded. As shown in \cref{subsec:Lojasiewicz_different_metrics}, $\chi^2(\mu\Vert \nu_t)^{\frac{1}{2}}$ turns out to linearize $KL(\mu\Vert \nu_t)^{\frac{1}{2}}$ when $\mu$ and $\nu_t$ are close. Hence, we conjecture that gradient flows of $\F$ under a metric $d$ can be shown to converge when the linearization of the metric remains bounded. This can be verified on simple examples for $\Vert \mu - \nu_t \Vert_{\dot{H}^{-1}(\nu_t)} $ as discussed in \cref{subsec:simple_example}. However, it remains hard to guarantee this condition in general. One possible approach could be to regularize $\F$ using an estimate of \cref{eq:neg_sobolev}. Indeed, \cite{Mroueh:2019} considers the gradient flow of a regularized version of the negative Sobolev distance which can be written in closed form, and shows that this decreases the MMD. Combing both losses could improve the overall convergence properties of the MMD, albeit at additional computational cost. In the next section, we propose a different approach to improve the convergence, and a particle-based algorithm to approximate the MMD flow in practice.

\section{A practical algorithm to descend the MMD flow}\label{sec:discretized_flow}
\subsection{A noisy update as a regularization}\label{sec:noisy_flow}

We showed in \cref{subsection:barrier_optimization} that $\F$ is a non-convex functional, and derived a condition in \cref{sec:Lojasiewicz_inequality} to reach the global optimum. We now address the case where such a condition does not necessarily hold, and  provide a regularization of the gradient flow to help achieve global optimality in this scenario. Our starting point will be the equilibrium condition in \cref{eq:equilibrium_condition}. If an equilibrium $\nu^*$ that satisfies \cref{eq:equilibrium_condition} happens to have a positive density, then $f_{\mu,\nu^{*}}$ would be constant everywhere. This in turn would mean that $f_{\mu,\nu^{*}}=0$ when the RKHS does not contain constant functions, as for a gaussian kernel \cite[Corollary 4.44]{Steinwart:2008a}. Hence, $\nu^*$ would be a global optimum since $\F(\nu^{*})=0$. The limit distribution $\nu^*$  might be singular, however, and can even be a dirac distribution \cite[Theorem 6]{mei2018mean}. Although the gradient $\nabla f_{\mu,\nu^{*}}$ is not identically $0$ in that case,  \cref{eq:equilibrium_condition} only evaluates it on the support $\nu^{*}$, on which $\nabla f_{\mu,\nu^{*}}=0$ holds. Hence a possible fix would be to make sure that the unnormalised witness gradient is also evaluated at points outside of the support of $\nu^{*}$. 
Here, we propose to regularize the flow by injecting noise into the gradient during updates of \cref{eq:euler_scheme_particles}, %
\begin{align}\label{eq:discretized_noisy_flow}
	X_{n+1} = X_{n} -\gamma \nabla f_{\mu,\nu_n}(X_n+ \beta_n U_n), \qquad n\geq 0,
\end{align}
where $U_n$ is a standard gaussian variable and $\beta_n$ is the noise level at $n$. Compared to \cref{eq:euler_scheme}, the sample here  is  first blurred before evaluating the gradient.
Intuitively, if $\nu_n$ approaches a local optimum $\nu^{*}$, $ \nabla f_{\mu,\nu_n}$ would be small on the support of $\nu_n$ but it might be much larger outside of it, hence evaluating $\nabla f_{\mu,\nu_n}$ outside the support of $\nu_n$ can help in escaping the local minimum. The stochastic process \cref{eq:discretized_noisy_flow} is different from adding a diffusion term to \cref{eq:continuity_mmd}. The latter case would %
 correspond to regularizing $\F$ using an entropic term as in \cite{mei2018mean,csimcsekli2018sliced} (see also \cref{subsec:kl_flow} on the Langevin diffusion) and was shown to converge to a global optimum that is in general different from the global minmum of the un-regularized loss. Eq. \cref{eq:discretized_noisy_flow} is also different from \cite{craig2016blob,carrillo2019blob}, where $\F$ (and thus its associated velocity field) is regularized by convolving the interaction potential $W$ in \cref{eq:potentials} with a mollifier. The optimal solution of a regularized version of the functional $\F$ will be generally different from the non-regularized one, however, which is not desirable in our setting. 
Eq. \cref{eq:discretized_noisy_flow} is more  closely related to the \textit{continuation methods} \cite{Gulcehre:2016a,Gulcehre:2016,Chaudhari:2017}  and \textit{graduated optimization} \cite{Hazan:2015} used for non-convex optimization in Euclidian spaces, which inject noise into the gradient of a loss function $F$ at each iteration. The key difference is the dependence of $f_{\mu,\nu_n}$ of $\nu_n$,  which is inherently due to functional optimization.    
We show in \cref{thm:convergence_noisy_gradient} that \cref{eq:discretized_noisy_flow} attains the global minimum of $\F$ provided that the level of the noise is well controlled, with the proof given in \cref{proof:thm:convergence_noisy_gradient}.
\begin{proposition}\label{thm:convergence_noisy_gradient}
	Let $(\nu_n)_{n\in \mathbb{N}}$ be defined by \cref{eq:discretized_noisy_flow} with an initial $\nu_0$. Denote $\mathcal{D}_{\beta_n}(\nu_n)=\mathbb{E}_{x\sim \nu_n, u\sim g}[\Vert \nabla f_{\mu,\nu_n}(x+\beta_n u) \Vert^2]$ with $g$ the density of the standard gaussian distribution.	Under \cref{assump:lipschitz_gradient_k,assump:Lipschitz_grad_rkhs}, and for a choice of $\beta_n$ such that
	\begin{equation}\label{eq:control_level_noise}
	8\lambda^2\beta_n^2 \F(\nu_n) \leq \mathcal{D}_{\beta_n}(\nu_n),
	\end{equation}
	\begin{flalign}\label{eq:decreasing_loss_iterations}
\text{the following inequality holds: }\quad\quad	\F(\nu_{n+1}) - \F(\nu_n  ) \leq -\frac{\gamma}{2}(1-3\gamma L)\mathcal{D}_{\beta_n}(\nu_n), &&
	\end{flalign}
	where $\lambda$ and $L$ are defined in \cref{assump:lipschitz_gradient_k,assump:Lipschitz_grad_rkhs} and depend only on the choice of the kernel. Moreover if  $\sum_{i=0}^n \beta_i^2 \rightarrow \infty,$ then
	\begin{equation}
	\F(\nu_n)\leq \F(\nu_0) e^{-4\lambda^2\gamma(1-3\gamma L)\sum_{i=0}^n \beta^2_i}.
	\end{equation}
\end{proposition}
A particular case where $\sum_{i=0}^n \beta_i^2 \rightarrow \infty$ holds is when $\beta_n$ decays as $1/\sqrt{n}$ while still satisfying \cref{eq:control_level_noise}. In this case, convergence occurs in polynomial time.
At each iteration, the level of the noise needs to be adjusted such that the gradient is not too blurred. This ensures that each step decreases the loss functional. However, $\beta_n$ does not need to decrease at each iteration: it could increase adaptively whenever needed. For instance, when  the sequence gets closer to a local optimum, it is helpful to increase the level of the noise to probe the gradient in regions where its value is not flat.
	Note that for $\beta_n = 0$  in \cref{eq:decreasing_loss_iterations} , we recover a similar bound to \cref{prop:decreasing_functional}.

\subsection{The sample-based approximate scheme}\label{sec:sample_based}

We now provide a practical algorithm to implement %
the noisy updates in the previous section,
which employs a discretization in space.
The update \cref{eq:discretized_noisy_flow} involves computing expectations of the gradient of the kernel $k$ w.r.t the target distribution $\mu$ and the current distribution $\nu_n$ at each iteration $n$. This suggests a simple approximate scheme, based on samples from these two distributions, where %
at each iteration $n$, we model a system of $N$ interacting particles $(X_n^i)_{1\leq i\leq N}$  and their empirical distribution in order to approximate $\nu_n$. 
More precisely, given i.i.d. samples $(X^i_0)_{1\leq i\leq N}$ and $(Y^{m})_{1\leq m\leq M}$ from $\nu_0$ and $\mu$ and a step-size $\gamma$, the approximate scheme iteratively updates the $i$-th particle as
\begin{align}\label{eq:euler_maruyama}
X_{n+1}^{i} = X_n^i -\gamma \nabla f_{\hat{\mu},\hat{\nu}_n}(X_n^i+\beta_n U_n^i),
\end{align}
where $U_{n}^{i}$ are i.i.d standard gaussians and $\hat{\mu},\,\hat{\nu}_n$ denote the empirical distributions of $(Y^{m})_{1\leq m\leq M}$ and $(X^i_n)_{1\leq i\leq N}$, respectively. It is worth noting that for $\beta_n=0$, \cref{eq:euler_maruyama} is equivalent to gradient descent over the particles $(X_n^{i})$ using a sample based version of the MMD.
Implementing \cref{eq:euler_maruyama} is straightforward as it only requires to evaluate the gradient of $k$ on the current particles and target samples. Pseudocode is provided in \cref{euclid}. 
The overall computational cost of the algorithm at each iteration is $O((M+N)N)$ with $O(M+N)$ memory. The computational cost becomes  $O(M+N)$ when the kernel is approximated using random features, as  is the case for regression with neural networks (\cref{subsec:training_neural_networks}). 
This is in contrast to the cubic cost of the flow of the KSD \cite{Mroueh:2019}, which requires solving a linear system at each iteration. The cost can also be compared to the algorithm in  \cite{csimcsekli2018sliced}, which involves computing empirical CDF and quantile functions of random projections of the particles.

The approximation scheme in \cref{eq:euler_maruyama} is a particle version of \cref{eq:discretized_noisy_flow}, so one would expect it to converge towards its population version  \cref{eq:discretized_noisy_flow} as $M$ and $N$ goes to infinity. This is shown below.
\begin{theorem}\label{prop:convergence_euler_maruyama}
 Let $n\ge 0$ and $T>0$. Let $\nu_n$ and $\hat{\nu}_n$ defined by \cref{eq:euler_scheme} and \cref{eq:euler_maruyama} respectively. Suppose \cref{assump:lipschitz_gradient_k} holds and that $\beta_n<B$ for all $n$, for some $B>0$. Then for any $\frac{T}{\gamma}\geq n$:
\[
\mathbb{E}\left[W_{2}(\hat{\nu}_{n},\nu_{n})\right]\leq \frac{1}{4}\left(\frac{1}{\sqrt{N}}(B+var(\nu_{0})^{\frac{1}{2}})e^{2LT}+\frac{1}{\sqrt{M}}var(\mu)^{\frac{1}{2}})\right)(e^{4LT}-1)
\]
\end{theorem}
\cref{prop:convergence_euler_maruyama} controls the propagation of the chaos at each iteration, and uses techniques from \cite{jourdain2007nonlinear}. Notice also that these rates remain true when no noise is added to the updates, i.e. for the original flow when $B=0$. A proof is provided in \cref{proof:propagation_chaos}. The dependence in $\sqrt{M}$ underlines the fact that our procedure could be interesting as a sampling algorithm when one only has access to $M$ samples of $\mu$ (see \cref{subsec:kl_flow} for a more detailed discussion).

\textbf{Experiments}
\begin{figure}[ht]
	\centering
	\includegraphics[width=1.\linewidth]{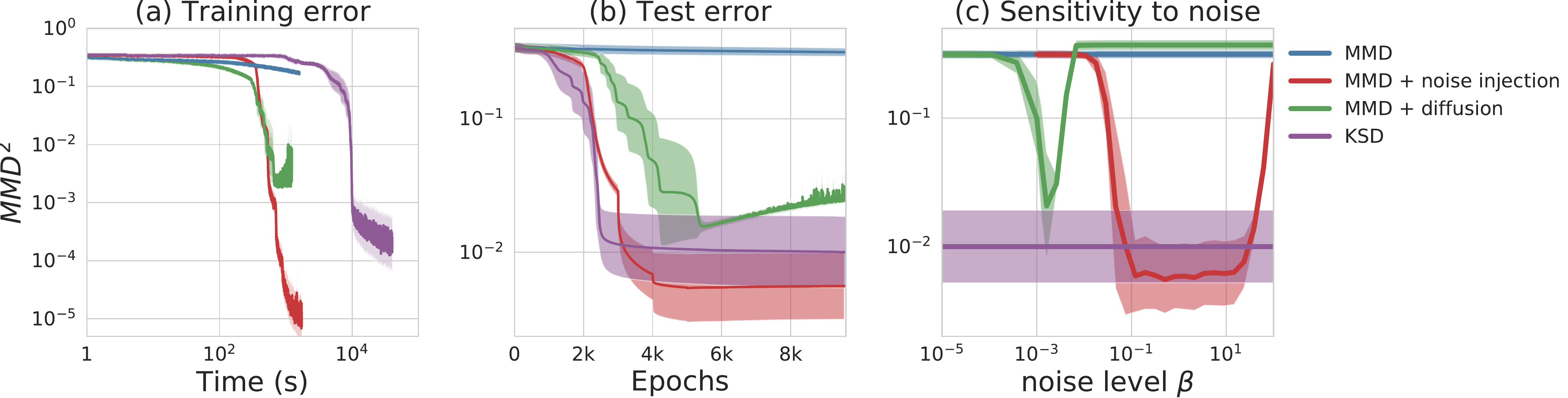}
	\caption{Comparison between different training methods for student-teacher ReLU networks with  gaussian output non-linearity and synthetic data uniform on a hyper-sphere. In blue, \cref{eq:euler_maruyama} is used without noise $\beta_n=0$ while in red noise is added with the following schedule: $\beta_0>0$ and $\beta_n$ is decreased by half after every $10^3$ epochs. In green, a diffusion term is added to the particles with noise level kept fixed during training ($\beta_n=\beta_0$). In purple, the KSD is used as a cost function instead of the MMD. In all cases, the kernel is estimated using random features (RF) with a batch size of $10^2$. Best step-size was selected for each method from $\{10^{-3},10^{-2},10^{-1}\}$ and was used for $10^4$ epochs on a dataset of $10^3$ samples (RF). Initial parameters of the networks are drawn from i.i.d. gaussians: $\mathcal{N}(0,1)$ for the teacher and $\mathcal{N}(10^{-3},1)$ for the student. Results are averaged over 10 different runs. } 
	\label{fig:experiments_student_teacher}
\end{figure}
\cref{fig:experiments_student_teacher} illustrates the behavior of the proposed algorithm \cref{eq:euler_maruyama} in a simple setting and compares it with three other methods:  MMD without noise injection (blue traces), MMD with diffusion (green traces) and KSD (purple traces, \cite{Mroueh:2019}). Here, a student network is trained to produce the outputs of a teacher network using gradient descent. More details on the experiment are provided in \cref{sec:experiments_neural_network}. As discussed in \cref{subsec:training_neural_networks}, this setting can be seen as a \textit{stochastic} version of the MMD flow since the kernel is estimated using random features at each iteration (\cref{eq:random_features_kernel} in  \cref{sec:experiments_neural_network}).  
Here, the MMD flow fails to converge towards the global optimum. Such behavior is consistent with the observations in \cite{Chizat:2018a} when the parameters are initialized from a gaussian noise with relatively high variance (which is the case here). On the other hand, adding noise to the gradient seems to lead to global convergence. Indeed, the training error decreases below $10^{-5}$ and leads to much better validation error. While adding a small diffusion term (green) help convergence, the noise-injection (red) still outperforms it. This also holds for 
KSD  (purple) which leads to a good solution (b) although at  a much higher computational cost (a).
Our noise injection method (red) is also robust to the amount of noise and achieves best performance over a wide region (c).
 On the other hand,  MMD + diffusion (green) performs well only for much smaller values of noise that are located in a narrow region.
  This is expected since adding a diffusion changes the optimal solution, unlike the injection where the global optimum of the MMD remains a fixed point of the algorithm.
  
Another illustrative experiment on a simple flow between Gaussians is given in \cref{sec:experiments_gaussian}.
\section{Conclusion}
We have introduced MMD flow, a novel flow over the space of distributions, with a practical space-time discretized implementation and a regularisation scheme to improve convergence. We provide theoretical results, highlighting intrinsic properties of the regular MMD flow, and guarantees on convergence based on  recent results in optimal transport, probabilistic interpretations of PDEs, and particle algorithms. Future work will focus on a deeper understanding of  regularization for  MMD flow,
and its application in sampling and optimization for large neural networks.

\printbibliography

\clearpage

\appendix

This appendix is organized as follows. In \cref{sec:appendix_math_background}, the mathematical background needed for this paper is given. In \cref{sec:assumptions_kernel}, we state the main assumptions used in this work. \cref{sec:appendix_gradient_flow} is dedicated to the construction of the gradient flow of the MMD. \cref{sec:appendix_convergence} provides proofs for the convergence results in \cref{sec:convergence_mmd_flow}. \cref{sec:appendix_algorithms} is dedicated to the modified gradient flow based on noise injection. In \cref{subsec:training_neural_networks}, we discuss the connexion with optimization of neural networks. \cref{sec:experiments} provides details about the experiments. Finally, some auxiliary  results are provided in \cref{sec:auxiliary_results}.

\section{Mathematical background}\label{sec:appendix_math_background}
We define $\X\subset\R^d$ as the closure of a convex open set, and $\mathcal{P}_2(\X)$ as the set of probability distributions on $\X$ with finite second moment, equipped with the 2-Wassertein metric denoted $W_2$. For any $\nu \in \mathcal{P}_2(\X)$, $L_2(\nu)$ is the set of square integrable functions w.r.t. $\nu$.
\subsection{Maximum Mean Discrepancy and Reproducing Kernel Hilbert Spaces}\label{sec:rkhs}
We recall here fundamental definitions and properties of reproducing kernel Hilbert spaces (RKHS) (see \cite{smola1998learning}) and Maximum Mean Discrepancies (MMD). 
Given a positive semi-definite kernel $(x,y)\mapsto k(x,y)\in \R$ defined for all $x,y\in\X$, we denote by $\kH$ its corresponding RKHS (see \cite{smola1998learning}). The space $\kH$ is a Hilbert space with inner product $\langle .,. \rangle_{\kH}$ and corresponding norm $\Vert . \Vert_{\kH}$. A key property of $\kH$ is the reproducing property: for all $f \in \kH, f(x) = \langle f, k(x, .)\rangle_{\kH}$. Moreover, if $k$ is $m$-times differentiable w.r.t. each of its coordinates, then any $f\in \kH$  is $m$-times differentiable  and $\partial^{\alpha}f(x)=\langle f, \partial^{\alpha} k(x,.) \rangle_{\kH}$ where $\alpha$ is any multi-index with $\alpha \leq m$ \cite[Lemma 4.34]{Steinwart:2008a}. When $k$ has at most quadratic growth, then for all $\mu\in \mathcal{P}_2(\X)$, $\int k(x,x) \diff \mu(x) <\infty$. In that case, for any $\mu\in \mathcal{P}_2(\X)$,  $ \phi_{\mu} := \int k(.,x)\diff \mu(x)$ is a well defined element in $\kH$ called the mean embedding of $\mu$. The kernel $k$ is said to be characteristic when such mean embedding is injective, that is any mean embedding is associated to a unique probability distribution. When $k$ is characteristic, it is possible to define a distance between distributions in $\mathcal{P}_2(\X)$ called the Maximum Mean Discrepancy:
\begin{align}
	MMD(\mu,\nu) = \Vert \phi_{\mu} - \phi_{\nu}\Vert_{\kH} \qquad \forall \; \mu,\nu \in \mathcal{P}_2(\X).
\end{align}
The difference between the mean embeddings of $\mu$ and $\nu$ is an element in $\kH$ called the unnormalised witness function between $\mu$ and $\nu$:  $f_{\mu,\nu} = \phi_{\nu} - \phi_{\mu}$. The MMD can also be seen as an \textit{Integral Probability Metric}:
\begin{align}
MMD(\mu,\nu) = \sup_{g\in \mathcal{B}} \int g\diff\mu - \int g \diff\nu
\end{align}
where $\mathcal{B} = \{ g\in \kH : \; \Vert g\Vert_{\kH}\leq 1 \}$  is the unit ball in the RKHS.

\subsection{$2$-Wasserstein geometry}\label{subsec:wasserstein_flow}
For two given probability distributions $\nu$ and $\mu$ in $\mathcal{P}_2(\X)$, we denote by $\Pi(\nu,\mu)$ the set of possible couplings between $\nu$ and $\mu$. In other words $\Pi(\nu,\mu)$ contains all possible distributions $\pi$ on $\X\times \X$ such that if $(X,Y) \sim \pi $ then $X \sim \nu $ and $Y\sim \mu$. The $2$-Wasserstein distance on $\mathcal{P}_2(\X)$ is defined by means of an optimal coupling between $\nu$ and $\mu$ in the following way:
\begin{align}\label{eq:wasserstein_2}
W_2^2(\nu,\mu) := \inf_{\pi\in\Pi(\nu,\mu)} \int \left\Vert x - y\right\Vert^2 \diff \pi(x,y) \qquad \forall \nu, \mu\in \mathcal{P}_2(\X)
\end{align}
It is a well established fact that such optimal coupling $\pi^*$ exists \cite{Villani:2009,Santambrogio:2015} . Moreover, it can be used to define a path $(\rho_t)_{t\in [0,1]}$ between $\nu$ and $\mu$ in $\mathcal{P}_2(\X)$. For a given time $t$ in $[0,1]$ and given a sample $(x,y)$ from $\pi^{*}$, it is possible to construct a sample $z_t$ from $\rho_t$ by taking the convex combination of $x$ and $y$: $z_t = s_t(x,y)$ where $s_t$ is given by:
\begin{equation}\label{eq:convex_combination}
s_t(x,y) = (1-t)x+ty \qquad \forall x,y\in \X, \; \forall t\in [0,1].
\end{equation}
The function $s_t$ is well defined since $\X$ is a convex set. More formally, $\rho_t$ can be written as the projection or push-forward of the optimal coupling $\pi^{*}$ by $s_t$:  
\begin{equation}\label{eq:displacement_geodesic}
\rho_t = (s_t)_{\#}\pi^{*}
\end{equation}
We recall that for any $T: \X \rightarrow \X$ a measurable map, and any $\rho \in \mathcal{P}(\X)$, the push-forward measure $T_{\#}\rho$ is characterized by:
\begin{align}
\int_{y \in \X} \phi(y) \diff T_{\#}\rho(y) =\int_{x \in \X}\phi(T(x)) \diff \rho(x) \text{ for every measurable and bounded function $\phi$.}
\end{align}
It is easy to see that \cref{eq:displacement_geodesic} satisfies the following boundary conditions at $t=0,1$:
\begin{align}\label{eq:boundary_conditions}
\rho_0 = \nu \qquad \rho_1 = \mu.
\end{align}
Paths of the form of \cref{eq:displacement_geodesic} are called \textit{displacement geodesics}. They can be seen as the shortest paths from $\nu$ to $\mu$ in terms of mass transport (\cite{Santambrogio:2015} Theorem 5.27). It can be shown that there exists a \textit{velocity vector field} $(t,x)\mapsto V_t(x)$ with values in $\R^d$ such that $\rho_t$ satisfies the continuity equation:
\begin{equation}\label{eq:continuity_equation}
\partial_t \rho_t + div(\rho_t V_t ) = 0 \qquad \forall t\in[0,1].
\end{equation}
This equation expresses two facts, the first one is that $-div(\rho_t V_t)$ reflects the infinitesimal changes in $\rho_t$ as dictated by the vector field (also referred to as velocity field) $V_t$, the second one is that the total mass of $\rho_t$ does not vary in time as a consequence of the divergence theorem. Equation \cref{eq:continuity_equation} is well defined in the distribution sense even when $\rho_t$ does not have a density. At each time $t$, $V_t$ can be interpreted as a tangent vector to the curve $(\rho_t)_{t\in[0,1]}$ so that the length $l((\rho_t)_{t\in[0,1]})$ of the curve $(\rho_t)_{t\in[0,1]}$ would be given by:
\begin{equation}
l((\rho_t)_{t\in[0,1]})^2 = \int_0^1 \Vert V_t \Vert^2_{L_2(\rho_t)} \diff t \quad \text{ where } \quad 
\left\Vert V_t \right\Vert^2_{L_2(\rho_t)} =  \int \left\Vert V_t(x) \right\Vert^2 \diff \rho_t(x)
\end{equation}
This perspective allows to provide a dynamical interpretation of the $W_2$ as the length  of the shortest path from $\nu$ to $\mu$ and is summarized by the celebrated Benamou-Brenier formula (\cite{benamou2000computational}):
\begin{align}\label{eq:benamou-brenier-formula}
W_2(\nu,\mu) = \inf_{(\rho_t,V_t)_{t\in[0,1]}} l((\rho_t)_{t\in[0,1]})
\end{align}
where the infimum is taken  over all couples  $\rho$ and $v$ satisfying  \cref{eq:continuity_equation}  with boundary conditions given by \cref{eq:boundary_conditions}. If $(\rho_t,V_t)_{t\in[0,1]}$ satisfies  \cref{eq:continuity_equation} and \cref{eq:boundary_conditions} and realizes the infimum in \cref{eq:benamou-brenier-formula}, it is then simply called a geodesic between $\nu$ and $\mu$; moreover it is called a constant-speed geodesic if, in addition, the norm of $V_t$ is constant for all $t\in[0,1]$. As a consequence, \cref{eq:displacement_geodesic} is a constant-speed displacement geodesic. 
\begin{remark}
	Such paths should not be confused with another kind of paths called \textit{mixture geodesics}. The mixture geodesic $(m_t)_{t\in[0,1]}$ from $\nu$ to $\mu$ is obtained by first choosing either $\nu$ or $\mu$ according to a Bernoulli distribution of parameter $t$ and then sampling from the chosen distribution:
	\begin{align}\label{eq:mixture_geodesic}
	m_t = (1-t)\nu + t\mu \qquad \forall t \in [0,1].
	\end{align}
	Paths of the form \cref{eq:mixture_geodesic} can be thought as the shortest paths between two distributions when distances on $\mathcal{P}_2(\X)$ are measured using the MMD (see \cite{Bottou:2017} Theorem 5.3). We refer to \cite{Bottou:2017} for an overview of the notion of shortest paths in probability spaces and for the differences between mixture geodesics and displacement geodesics.
	Although, we will be interested in the MMD as a loss function, we will not consider the geodesics that are naturally associated to it and will rather consider the displacement geodesics defined in \cref{eq:displacement_geodesic} for reasons that will become clear in \cref{subsec:lambda_convexity}.
\end{remark}

\subsection{Gradient flows on the space of probability measures}\label{subsec:gradient_flows_functionals}
Consider a real valued functional $\F$ defined over $\mathcal{P}_2(\x)$. We call $\frac{\partial{\F}}{\partial{\nu}}$ if it exists, the unique (up to additive constants) function such that $\frac{d}{d\epsilon}\F(\nu+\epsilon  (\nu'-\nu))\vert_{\epsilon=0}=\int\frac{\partial{\F}}{\partial{\nu}}(\nu) (\diff \nu'-\diff \nu) $ for any $\nu' \in \mathcal{P}_2(\X)$. The function $\frac{\partial{\F}}{\partial{\nu}}$ is called the first variation of $\F$ evaluated at $\nu$. 
We consider here functionals $\F$ of the form:
\begin{equation}\label{eq:lyapunov}
\F(\nu)=\int U(\nu(x)) \nu(x)dx + \int V(x)\nu(x)dx + \int W(x,y)\nu(x)\nu(y)dxdy
\end{equation}
where  $U$ is the internal potential, $V$ an external potential and $W$ an
interaction potential. The formal gradient flow equation associated to such functional can be written (see \cite{carrillo2006contractions}, Lemma 8 to 10):
\begin{equation}\label{eq:continuity_equation1}
\frac{\partial \nu}{\partial t}= div( \nu \nabla \frac{\partial \F}{\partial \nu})=div( \nu \nabla (U'(\nu) + V + W * \nu))
\end{equation}
where $div$ is the divergence operator and $\nabla \frac{\partial \F}{\partial \nu}$ is the strong subdifferential of $\F$ associated to the $W_2$ metric (see \cite{ambrosio2008gradient}, Lemma 10.4.1). Indeed, for some generalized notion of gradient $\nabla_{W_2}$, and for sufficiently regular $\nu$ and $\F$, the r.h.s. of \cref{eq:continuity_equation1} can be formally written as $-\nabla_{W_2}\F(\nu)$.
The dissipation of energy along the flow is then given by: 
\begin{align}\label{eq:dissipation_energy}
\frac{d \F(\nu_t)}{dt} =-D(\nu_t) \quad \text{ with } D(\nu)= \int \Vert\nabla \frac{\partial \F(\nu_t(x))}{\partial \nu}\Vert ^2 \nu_t(x)dx
\end{align}
Such expression can be obtained by the following formal calculations:
	\begin{equation*}
	\frac{d\F(\nu_t)}{dt}=\int \frac{\partial \F(\nu_t)}{\partial \nu_t} \frac{\partial \nu_t}{\partial t}=\int \frac{\partial \F(\nu_t)}{\partial \nu} div( \nu_t \nabla \frac{\partial \F(\nu_t)}{\partial \nu})=-\int \|\nabla \frac{\partial \F(\nu_t)}{\partial \nu}\|^2d \nu_t.
	\end{equation*}

\subsection{Displacement convexity}\label{subsec:lambda_convexity}
Just as for Euclidian spaces, an  important criterion to characterize the convergence of the Wasserstein gradient flow of a functional $\F$ is given by displacement convexity (see \cite[Definition 16.5 (1st bullet point)]{Villani:2004})):

\begin{definition}\label{def:displacement_convexity}[Displacement convexity] 
We say that a functional $\nu\mapsto\mathcal{F}(\nu)$ is displacement convex
	if for any $\nu$ and $\nu'$ and a constant speed geodesic $(\text{\ensuremath{\rho_{t}}})_{t \in [0,1]}$
	between $\nu$ and $\nu'$ with velocity vector field $(V_{t})_{t \in [0,1]}$ as defined by \cref{eq:continuity_equation},
	the following holds:
	\begin{equation}\label{eq:displacement_convex}
		\F(\rho_{t})\leq(1-t)\F(\nu_{0})+t\F(\nu_{1}) \qquad\forall\; t\in[0,1].
	\end{equation}
\end{definition}
\cref{def:displacement_convexity} can be relaxed to a more general notion of convexity called $\Lambda$-displacement convexity (see \cite[Definition 16.5 (3rd bullet point)]{Villani:2009}). We first define an admissible functional $\Lambda$:
\begin{definition}\label{def:conditions_lambda}[Admissible $\Lambda$ functional]
	Consider a functional $(\rho,v)\mapsto \Lambda(\rho,v) \in \R$  defined for any probability distribution $\rho\in \mathcal{P}_2(\X)$ and any square integrable vector field $v$ w.r.t $\rho$. We say that $\Lambda$ is admissible, if it satisfies:
	\begin{itemize}
	\item For any $\rho \in \mathcal{P}_2(\X)$,  $v\mapsto \Lambda(\rho,v)$ is a quadratic form.
	\item For any geodesic $(\rho_t)_{0\leq t\leq 1}$ between two distributions $\nu$ and $\nu'$ with corresponding vector fields $(V_t)_{t \in [0,1]}$ it holds that $\inf_{0\leq t\leq 1}\Lambda(\rho_t,V_t)/\Vert V_t\Vert_{L_{2}(\rho_t)}^{2}>-\infty$ 
\end{itemize}
\end{definition}
We can now define the notion of $\Lambda$-convexity:
\begin{definition}\label{def:lambda-convexity}[$\Lambda$ convexity]
	We say that a functional $\nu\mapsto\mathcal{F}(\nu)$ is $\Lambda$-convex
	if for any $\nu,\nu'\in \mathcal{P}_2(\X)^2$ and a constant speed geodesic $(\text{\ensuremath{\rho_{t}}})_{t \in [0,1]}$
	between $\nu$ and $\nu'$ with velocity vector field $(V_{t})_{t \in [0,1]}$ as defined by \cref{eq:continuity_equation},
	the following holds:
	\begin{equation}\label{eq:lambda_displacement_convex}
		\F(\rho_{t})\leq(1-t)\F(\nu_{0})+t\F(\nu_{1})-\int_{0}^{1}\Lambda(\rho_{s},V_{s})G(s,t)ds \qquad\forall\; t\in[0,1].
	\end{equation}
	where $(\rho,v)\mapsto\Lambda(\rho,v)$ satisfies \cref{def:conditions_lambda},
	and $G(s,t)=s(1-t) \mathbb{I}\{s\leq t\}
	+t(1-s) \mathbb{I}\{s\geq t\}$.
	A particular case is when $\Lambda(\rho,v)= \lambda \int \left\Vert v(x) \right\Vert^2 \diff \rho(x)   $ for some $\lambda\in \R$. In that case, \cref{eq:lambda_displacement_convex} becomes:
\begin{align}\label{eq:semi-convexity}
	\F(\rho_{t})\leq(1-t)\F(\nu_{0})+t\F(\nu_{1})-\frac{\lambda}{2}t(1-t)W_2^2(\nu_0,\nu_1)  \qquad\forall\; t\in[0,1].
\end{align}
\end{definition}
\cref{def:displacement_convexity} is a particular case of \cref{def:lambda-convexity}, where in \cref{eq:semi-convexity} one has $\lambda =0$.

\subsection{Comparison with the Kullback Leilber divergence flow}\label{subsec:kl_flow}

\textit{Continuity equation and McKean Vlasov process.}	A famous example of a free energy \cref{eq:lyapunov} is the Kullback-Leibler divergence, defined for $\nu, \mu \in \mathcal{P}(\X)$ by
	$KL(\nu,\mu)=\int log(\frac{\nu(x)}{\mu(x)})\nu(x)dx$. Indeed, $KL(\nu, \mu)=\int U(\nu(x))dx + \int V(x) \nu(x)dx$ with $U(s)=s\log(s)$ the entropy function and $V(x)=-log(\mu(x))$. In this case, $\nabla \frac{\partial \F}{\partial \nu}= \nabla \log(\nu) + \nabla V=  \nabla \log(\frac{\nu}{\mu})$ and equation \cref{eq:continuity_equation1} leads to the classical Fokker-Planck equation
	\begin{equation}\label{eq:Fokker-Planck}
	\frac{\partial{\nu}}{\partial t}= div(\nu \nabla V )+ \Delta \nu,
	\end{equation}
	where $\Delta$ is the Laplacian operator. It is well-known (see for instance \cite{jordan1998variational}) that the distribution of the Langevin diffusion in \cref{eq:langevin_diffusion} satisfies \cref{eq:Fokker-Planck},
	\begin{equation}\label{eq:langevin_diffusion}
	dX_t= -\nabla \log \mu (X_t)dt+\sqrt{2}dB_t.
	\end{equation}
	Here, $(B_t)_{t\ge0}$ is a $d$-dimensional Brownian motion. While the entropy term in the $KL$ functional prevents the particles from "crashing" onto the mode of $\mu$, this role could be played by the interaction energy $W$ defined in \cref{eq:potentials} for the MMD. Indeed, consider for instance the gaussian kernel $k(x,x')=e^{-\|x-x'\|^2}$. It is convex thus attractive at long distances ($\|x-x'\|>1$) but repulsive at small distances so repulsive. 
	
\textit{Convergence to a global minimum.} The solution to the Fokker-Planck equation describing the gradient flow of the $KL$ can be shown to converge towards $\mu$ under mild assumptions. This follows from the displacement convexity of the $KL$ along the Wasserstein geodesics. Unfortunately the MMD is not displacement convex in general, as shown in \cref{subsection:barrier_optimization} or \cref{subsec:appendix_lambda_convexity}. This makes the task of proving the convergence of the gradient flow of the MMD to the global optimum $\mu$ much harder. %

\textit{Sampling algorithms derived from gradient flows}. Two settings are usually encountered in the sampling literature: \textit{density-based}, i.e. the target $\mu$ is known up to a constant, or \textit{sample-based}, i.e. only a set of samples $X \sim \mu$ is accessible.
	The Unadjusted Langevin Algorithm (ULA), which involves a time-discretized version of the Langevin diffusion falls into the first category since it requires the knowledge of $\nabla \log \mu$. In a sample-based setting, it may be difficult to adapt the ULA algorithm, since this would require to estimate $\nabla \log(\mu)$ based on a set of samples of $\mu$, before plugging this estimate in the update of the algorithm. This problem, sometimes referred to as \textit{score estimation} in the literature, has been the subject of a lot of work but remains hard especially in high dimensions (see \cite{sutherland2017efficient},\cite{li2018gradient},\cite{shi2018spectral}). In contrast, the discretized flow (in time and space) of the MMD presented in \cref{sec:sample_based} is naturally adapted to the sample-based setting. 

\section{Main assumptions}\label{sec:assumptions_kernel}

We state here all the assumptions on the kernel $k$ used to prove all the results:%
\begin{assumplist}
	\item \label{assump:lipschitz_gradient_k} $k$ is continuously differentiable on $\X$ with $L$-Lipschitz gradient: $\Vert  \nabla k(x,x') - \nabla k(y,y')\Vert \leq L(\Vert  x-y\Vert + \Vert x'-y' \Vert ) $ for all $x,x',y,y' \in \X$.
	\item \label{assump:diff_kernel} $k$ is twice differentiable on $\X$.
	\item \label{assump:bounded_fourth_oder} $\Vert Dk(x,y) \Vert\leq \lambda  $ for all $x,y\in \X$, where $Dk(x,y)$ is an $\mathbb{R}^{d^2}\times \mathbb{R}^{d^2}$ matrix with entries given by $\partial_{x_{i}}\partial_{x_{j}}\partial_{x'_{i}}\partial_{x_{j}'}k(x,y)$.%
	\item \label{assump:Lipschitz_grad_rkhs} $ \sum_{i=1}^d\Vert  \partial_i k(x,.) - \partial_i k(y,.)\Vert^2_{\kH} \leq \lambda^2 \Vert  x-y\Vert^2 $ for all $x,y\in \X$.%

\end{assumplist}

\section{Construction of the gradient flow of the MMD}\label{sec:appendix_gradient_flow}

\subsection{Continuous time flow}

Existence and uniqueness of a solution to \cref{eq:continuity_mmd,eq:mcKean_Vlasov_process} is guaranteed under Lipschitz regularity of $\nabla k$.%
\begin{proof}[Proof of \cref{prop:existence_uniqueness}]\label{proof:prop:existence_uniqueness}[Existence and uniqueness]
Under \cref{assump:lipschitz_gradient_k}, the map $(x,\nu)\mapsto \nabla f_{\mu,\nu}(x)=\int \nabla k(x,.)d \nu - \int \nabla k(x,.) d \mu$ is Lipschitz continuous on $\X \times \mathcal{P}_2(\X)$ (endowed with the product of the canonical metric on $\X$ and $W_2$ on $\mathcal{P}_2(\X)$), see \cref{prop:grad_witness_function}. Hence, we benefit from standard existence and uniqueness results of McKean-Vlasov processes (see \cite{jourdain2007nonlinear}). Then, it is straightforward to verify that the distribution of \cref{eq:mcKean_Vlasov_process} is solution of \cref{eq:continuity_mmd} by ItÃŽ's formula (see \cite{ito1951stochastic}). The uniqueness of the gradient flow, given a starting distribution $\nu_0$, results from the $\lambda$-convexity of $\F$ (for $\lambda=3L$) which is given by \cref{lem:lambda_convexity_bis}, and  \cite[Theorem 11.1.4]{ambrosio2008gradient}. The existence derive from the fact that the sub-differential of $\F$ is single-valued, as stated by \cref{prop:differential_mmd}, and that any $\nu_0$ in $\mathcal{P}_2(\X)$ is in the domain of $\F$. One can then apply \cite[Theorem 11.1.6 and Corollary 11.1.8]{ambrosio2008gradient}.
\end{proof}

\begin{proof}[Proof of \cref{prop:decay_mmd}]\label{proof:prop:decay_mmd}[Decay of the MMD]
Recalling the discussion in \cref{subsec:gradient_flows_functionals}, the time derivative of $\F(\nu_t)$ along the flow is formally given by \cref{eq:dissipation_energy}. But we know from \cref{prop:differential_mmd} that the strong differential $\nabla  \frac{\delta \F(\nu)}{\delta\nu}$ is given by $\nabla f_{\mu,\nu}$. Therefore, one formally obtains the desired expression by exchanging the order of derivation and integration, performing an integration by parts and using the continuity equation (see \eqref{eq:dissipation_energy}). We refer to \cite{Mroueh:2019} for similar calculations. 
One can also obtain directly the same result using the energy identity in \cite[Theorem 11.3.2]{ambrosio2008gradient} which holds for $\lambda$-displacement convex functionals. The result applies here since, by \cref{lem:lambda_convexity_bis}, we know that $\F$ is $\lambda$-displacement convex with $\lambda = 3L$. 
 \end{proof}

\subsection{Time-discretized flow}\label{appendix:subsec:convegence_time_discrete}
We  prove that \cref{eq:euler_scheme} approximates \cref{eq:continuity_mmd}. To make  the dependence on the step-size $\gamma$ explicit, we will write: $\nu_{n+1}^{\gamma} =(I-\gamma\nabla f_{\mu,\nu_n^{\gamma}})_{\#}\nu_{n}^{\gamma}$ (so $\nu_n^{\gamma}=\nu_n$ for any $n \ge 0$). We start by introducing an auxiliary  sequence $\bar{\nu}_{n}^{\gamma}$ built by iteratively applying $\nabla f_{\mu,\nu_{\gamma n}}$ where $\nu_{\gamma n}$  is the solution of \cref{eq:continuity_mmd} at time $t= \gamma n$: 
 \begin{equation}\label{eq:intermed_process_time}
 \bar{\nu}_{n+1}^{\gamma} =(I-\gamma\nabla f_{\mu,\nu_{\gamma n}})_{\#}\bar{\nu}_{n}^{\gamma}
  \end{equation}
 with $\bar{\nu}_{0}=\nu_{0}$. Note that the latter sequence involves the continuous-time process $\nu_t$ of \cref{eq:continuity_mmd} with $t=\gamma n$. Using $\nu_n^{\gamma}$, we also consider the interpolation path $\rho_{t}^{\gamma}=(I-(t-n\gamma)\nabla f_{\mu,\nu_{n}^{\gamma}})_{\#}\nu_{n}^{\gamma}$ for all $t\in[n\gamma,(n+1)\gamma)$ and $n\in \mathbb{N}$, which is the same as in \cref{prop:convergence_euler_scheme}.
\begin{proof}[Proof of \cref{prop:convergence_euler_scheme}]\label{proof:prop:convergence_euler_scheme}
Let $\pi$ be an optimal coupling between $\nu_{n}^{\gamma}$ and $\nu_{\gamma n}$,
and $(x,y)$ a sample from $\pi$. For $t\in[n\gamma,(n+1)\gamma)$ we write $y_{t} =y_{n\gamma}-\int_{n\gamma}^{t}\nabla f_{\mu,\nu_{s}}(y_u)\diff u$ and $x_{t}  =x-(t-n\gamma)\nabla f_{\mu,\nu_{n}^{\gamma}}(x)$ where $y_{n\gamma}= y$. We also introduce the approximation error $  E(t,n\gamma):=y_{t}-y+(t-n\gamma)\nabla f_{\mu,\nu_{\gamma n}}(y)$ for which we know by  \cref{lem:Taylor-expansion} that $\mathcal{E}(t,n\gamma):=\mathbb{E}[E(t,n\gamma)^2]^{\frac{1}{2}}$ is upper-bounded by $(t-n\gamma)^{2}C$ for some positive constant $C$ that depends only on $T$ and the Lipschitz constant $L$. This allows to write:
\begin{align*}
W_{2}(\rho_{t}^{\gamma},\nu_{t}) & \leq\mathbb{E}\left[\left\Vert y-x+(t-n\gamma)(\nabla f_{\mu,\nu_{n}^{\gamma}}(x)-\nabla f_{\mu,\nu_{\gamma n}}(y))+E(t,n\gamma)\right\Vert^{2}\right]^{\frac{1}{2}}\\
 & \leq W_{2}(\nu_{n}^{\gamma},\nu_{\gamma n})+4L(t-n\gamma)W_{2}(\nu_{n}^{\gamma},\nu_{\gamma n})+\mathcal{E}(t,n\gamma)\\
 & \leq(1+4\gamma L)W_{2}(\nu_{n}^{\gamma},\nu_{\gamma n})+(t-\gamma n)^2C\\ &\leq(1+4\gamma L)\left( W_{2}(\nu_{n}^{\gamma},\bar{\nu}_{n}^{\gamma})+W_{2}(\nu_{\gamma n},\bar{\nu}_{n}^{\gamma}) \right)+\gamma^{2}C \\
 & \leq\gamma\left[\left(1+4\gamma L\right)M(T)+\gamma C\right]
\end{align*}
The second line is obtained using that $\nabla f_{\mu,\nu_{\gamma n}}(x)$ is jointly $2L$-Lipschitz in $x$  and $\nu$ (see \cref{prop:grad_witness_function}) and by the fact that $W_{2}(\nu_{n}^{\gamma},\nu_{\gamma n}) = \mathbb{E}_{\pi}[\Vert y-x\Vert^2]^{\frac12}$. The third one is obtained using $t-n \gamma\le \gamma$. For the last inequality, we used \cref{lem:euler_error_1,lem:euler_error_2}  where $M(T)$ is a constant that depends only on $T$. Hence for $\gamma\leq\frac{1}{4L}$ we get $W_{2}(\rho_{t}^{\gamma},\nu_{t})\leq\gamma(\frac{C}{4L}+2M(T)).$
\end{proof}

\begin{lemma}\label{lem:euler_error_1}
For any $n\geq0$:
\begin{equation*}
W_{2}(\nu_{\gamma n},\bar{\nu}_{n}^{\gamma})\le\gamma\frac{C}{2L}(e^{n\gamma 2L}-1)
\end{equation*}
\end{lemma}
\begin{proof}
Let $\pi$ be an optimal coupling between $\bar{\nu}_{n}^{\gamma}$
and $\nu_{\gamma n}$ and  $(\bar{x}$, $x)$ a joint sample
from $\pi$. Consider also the joint sample $(\bar{y},y)$ obtained from  $(\bar{x}$ ,$x)$ by applying the gradient flow of $\F$ in continuous time to get $y := x_{(n+1)\gamma}=x_{n \gamma}-\int_{n\gamma}^{(n+1)\gamma}\nabla f_{\mu,\nu_{s}}(x_u)\diff u$ with $x_{n\gamma} = x$ and by taking a discrete step from $\bar{x}$ to write $\bar{y}=\bar{x}-\gamma\nabla f_{\mu,\nu_{\gamma n}}(\bar{x})$. It is easy to see that $y\sim\nu_{\gamma(n+1)}$ (i.e. a sample from the continous process \cref{eq:continuity_mmd} at time $t=(n+1)\gamma$) and $\bar{y}\sim\bar{\nu}_{n+1}^{\gamma}$ (i.e. a sample from \cref{eq:intermed_process_time}). Moreover, we introduce the approximation error $E((n+1)\gamma,n\gamma):=y-x+\gamma\nabla f_{\mu,\nu_{\gamma n}}(x)$ for which we know by \cref{lem:Taylor-expansion} that $\mathcal{E}((n+1)\gamma,n\gamma):=\mathbb{E}[E((n+1)\gamma,n\gamma)^2]^{\frac{1}{2}}$ is upper-bounded by $\gamma^{2}C$ for some positive constant $C$ that depends only on $T$ and the Lipschitz constant $L$. Denoting by $a_{n}=W_{2}(\nu_{\gamma n},\bar{\nu}_{n}^{\gamma})$, one can therefore write:
\begin{align*}
a_{n+1}\leq & \mathbb{E_{\pi}}\left[\left\Vert x-\gamma\nabla f_{\mu,\nu_{\gamma n}}(x)-\bar{x}+\gamma\nabla f_{\mu,\nu_{\gamma n}}(\bar{x})+E((n+1)\gamma,n\gamma)\right\Vert^{2}\right]^{\frac{1}{2}}\\
\leq & \mathbb{E_{\pi}}\left[\left\Vert x-\bar{x}\right\Vert^{2}\right]^{\frac{1}{2}}+\gamma\mathbb{E_{\pi}}\left[\left\Vert\nabla f_{\mu,\nu_{\gamma n}}(x)-\nabla f_{\mu,\nu_{\gamma n}}(\bar{x}))\right\Vert^{2}\right]^{\frac{1}{2}}+\gamma^{2}C
\end{align*}
Using that $\nabla f_{\mu,\nu_{\gamma n}}$ is $2L$-Lipschitz by \cref{prop:grad_witness_function} and
recalling that $\mathbb{E}_{\pi}\left[\Vert x-\bar{x}\Vert^{2}\right]^{\frac{1}{2}}=W_{2}(\nu_{\gamma n},\bar{\nu}_{n}^{\gamma})$, we get the recursive inequality $a_{n+1}\leq(1+2 \gamma L)a_{n}+\gamma^{2}C$. Finally, 
using \cref{lem:Discrete-Gronwall-lemma} and recalling that $a_{0}=0$, since by
definition $\bar{\nu}_{0}^{\gamma}=\nu_{0}^{\gamma}$, we conclude that $a_{n}\leq\gamma\frac{C}{2L}(e^{n\gamma 2L}-1)$.
\end{proof}
\begin{lemma}\label{lem:euler_error_2}
For any $T>0$ and $n$ such that $n\gamma\leq T$
\begin{equation}
W_{2}(\nu_{n}^{\gamma},\bar{\nu}_{n}^{\gamma})\leq\gamma\frac{C}{8L^2}(e^{4TL}-1)^{2}
\end{equation}
\end{lemma}
\begin{proof}
Consider now an optimal coupling $\pi$ between $\bar{\nu}_{n}^{\gamma}$
and $\nu_{n}^{\gamma}$. Similarly to \cref{lem:euler_error_1}, we
denote by $(\bar{x},x)$ a joint sample from $\pi$ and $(\bar{y},y)$
is obtained from $(\bar{x},x)$  by applying the discrete updates  : $\bar{y}=\bar{x}-\gamma\nabla f_{\mu,\nu_{\gamma n}}(\bar{x})$ and $y=x-\gamma\nabla f_{\mu,\nu_{n}^{\gamma}}(x)$. We again have that $y\sim\nu_{n+1}^{\gamma}$ (i.e. a sample from the time discretized process \cref{eq:euler_scheme}) and $\bar{y}\sim\bar{\nu}_{n+1}^{\gamma}$ (i.e. a sample from \cref{eq:intermed_process_time}).
Now, denoting by $b_{n}=W_{2}(\nu_{n}^{\gamma},\bar{\nu}_{n}^{\gamma})$, it is easy to see from the definition of $\bar{y}$ and $y$ that
we have:
\begin{align*}
b_{n+1} & \leq\mathbb{E_{\pi}}\left[\left\Vert x-\gamma\nabla f_{\mu,\nu_{n}^{\gamma}}(x)-\bar{x}+\gamma\nabla f_{\mu,\nu_{\gamma n}}(\bar{x})\right\Vert^{2}\right]^{\frac{1}{2}}\\
&\leq (1+2\gamma L)  \mathbb{E_{\pi}}\left[\left\Vert x-\bar{x}\right\Vert^2\right]^{\frac{1}{2}} + 2\gamma L W_2(\nu_n^{\gamma},\nu_{\gamma n}))\\
 & \leq (1+ 4\gamma L)b_n + \gamma L W_2(\bar{\nu}_n^{\gamma},\nu_{\gamma n})
\end{align*}
The second line is obtained recalling that $\nabla f_{\mu,\nu}(x)$ is $2L$-Lipschitz in both $x$
and $\nu$ by \cref{prop:grad_witness_function}. The third line follows by triangular inequality and using $\mathbb{E_{\pi}}\left[\left\Vert x-\bar{x}\right\Vert^2\right]^{\frac{1}{2}}= W_2(\nu_n^{\gamma},\bar{\nu}_n^{\gamma}) = b_n$, since $\pi$ is an optimal coupling between $\bar{\nu}_{n}^{\gamma}$
and $\nu_{n}^{\gamma}$. 
By \cref{lem:euler_error_1}, we have $W_2(\bar{\nu}_n^{\gamma},\nu_{\gamma n})\leq\gamma\frac{C}{2L}(e^{2n\gamma L}-1)$, hence, for any $n$ such that $n\gamma\leq T$ we get the recursive inequality
\[b_{n+1}\leq(1+4\gamma L)b_{n}+(C/2L)\gamma^{2}(e^{2TL}-1).\]
Finally, using again \cref{lem:Discrete-Gronwall-lemma}, it follows that $b_{n}\leq\gamma\frac{C}{8L^2}(e^{4TL}-1)^{2}$.
\end{proof}

\begin{lemma}\label{lem:Taylor-expansion}[Taylor expansion]
	Consider the process $ \dot{x}_t = - \nabla f_{\mu,\nu_t}(x_t) $, and denote by $\mathcal{E}(t,s) = \mathbb{E}[ \Vert x_t - x_s +(t-s)\nabla f_{\mu,\nu_s}(x_s) \Vert ^2 ]^{\frac{1}{2}} $ for $0\leq s \leq t \leq T$. Then one has:
	\begin{align}
		\mathcal{E}(t,s)\leq  2L^2 r_0 e^{LT}(t-s)^2
	\end{align}
	with $r_0 = \mathbb{E}_{(x,z)\sim \nu_0 \otimes \mu}[\Vert x-z  \Vert]$
\end{lemma}
\begin{proof}
	By definition of $x_t$ and $\mathcal{E}(t,s)$ one can write:
	\begin{align*}
		\mathcal{E}(t,s) 
		&=
		\mathbb{E}\left[\left\Vert \int_{s}^t (\nabla f_{\mu,\nu_s}(x_s) - \nabla f_{\mu,\nu_u}(x_u))\diff u  \right\Vert^2 \right]^{\frac{1}{2}} \\
		&\leq
		\int_{s}^t  \mathbb{E}\left[\left\Vert (\nabla f_{\mu,\nu_s}(x_s) - \nabla f_{\mu,\nu_u}(x_u))  \right\Vert^2 \right]^{\frac{1}{2}} \diff u\\
		&\leq
		2L\int_{s}^t  \mathbb{E}\left[(\left\Vert  x_s - x_u \right\Vert + W_2(\nu_s,\nu_u))^2 \right]^{\frac{1}{2}} \diff u \leq 4L\int_{s}^t   \mathbb{E}\left[\left\Vert  x_s - x_u \right\Vert^2\right]^{\frac{1}{2}}\diff u
	\end{align*}
	Where we used an integral expression for $x_t$ in the first line then applied a triangular inequality for the second line. The last line is obtained recalling that $\nabla f_{\mu,\nu}(x)$ is jointly $2L$-Lipschitz in $x$ and $\nu$ by \cref{prop:grad_witness_function} and that $W_2(\nu_s,\nu_u) \leq \mathbb{E}\left[\left\Vert  x_s - x_u \right\Vert^2\right]^{\frac{1}{2}}$. Now we use again an integral expression for $x_u$ which further gives:
	\begin{align*}
		\mathcal{E}(t,s) \leq & 4L \int_{s}^t \mathbb{E}\left[\left\Vert \int_s^u  \nabla f_{\mu,\nu_l}(x_l) \diff l \right\Vert^2 \right]^{\frac{1}{2}}\diff u\\
		\leq & 4L \int_{s}^t \int_s^u  \mathbb{E}\left[ \left\Vert \mathbb{E}\left[ \nabla_1 k(x_l,x_l') - \nabla_1 k(x_l,z) \right]   \right\Vert^2  \right]^\frac{1}{2}\diff l\diff u\\
		\leq &
		4L^2 \int_{s}^t \int_s^u   \mathbb{E}\left[\left\Vert x_l' - z \right\Vert\right] \diff l \diff u
	\end{align*}
	Again, the second line is obtained using a triangular inequality and recalling the expression of $\nabla f_{\mu,\nu}(x)$ from \cref{prop:grad_witness_function}. The last line uses that $\nabla k$ is $L$-Lipschitz by \cref{assump:lipschitz_gradient_k}. Now we need to make sure that $\Vert x_l' - z \Vert$ remains bounded at finite times. For this we will first show that $ r_t = \mathbb{E}[\Vert x_t - z \Vert]$ satisfies an integro-differential inequality: 
	\begin{align*}
		r_t\leq & \mathbb{E}\left[\left\Vert x_0  - z -\int_0^t \nabla f_{\mu,\nu_s}(x_s) \diff s \right\Vert\right]\\
		\leq &
		r_0 +\int_0^t \mathbb{E}\left[\left\Vert \nabla_1 k(x_s,x_s')- \nabla_1 k(x_s,z) \right\Vert\right] \diff s\leq 
		r_0  + L\int_0^t r_s \diff s
	\end{align*}
	Again, we used an integral expression for $x_t$ in the first line, then a triangular inequality recalling the expression of $\nabla f_{\mu,\nu_s}$. The last line uses again that $\nabla k$ is $L$-Lipschitz. By Gronwall's lemma it is easy to see that $r_t \leq r_0e^{Lt}$ at all times. Moreover, for all $t\leq T$ we have a fortiori  that $r_t \leq r_0 e^{LT}$.
	Recalling back the upper-bound on $\mathcal{E}(t,s)$ we have finally:
	\[
	\mathcal{E}(t,s)\leq 4L^2 r_0 e^{LT} \int_{s}^t \int_s^u \diff l \diff u = 2L^2 r_0 e^{LT}(t-s)^2
	\] 
\end{proof}

We show now that \cref{eq:euler_scheme} decreases the functional $\F$. In all the proofs, the step-size $\gamma$ is fixed.
\begin{proof}[Proof of \cref{prop:decreasing_functional}]\label{proof:prop:decreasing_functional}
	Consider a path between $\nu_n$ and $\nu_{n+1}$ of the form $\rho_t	=(I-\gamma t\nabla f_{\mu,\nu_n})_{\#}\nu_n$. We know by \cref{prop:grad_witness_function} that $\nabla f_{\mu,\nu_n}$ is $2L$ Lipschitz, thus by \cref{lem:derivative_mmd_augmented} and using $\phi(x) = -\gamma \nabla f_{\mu,\nu_n}(x)$, $\psi(x) = x$ and $q = \nu_n$  it follows that $\F(\rho_t)$ is differentiable and hence absolutely continuous. Therefore one can write:
	\begin{align}\label{eq:taylor_expansion_decreasing}
	\mathcal{F}(\rho_1)-\mathcal{F}(\rho_0) = \dot{\mathcal{F}}(\rho_0)+  \int_0^1 \dot{\F}(\rho_t)- \dot{\F}(\rho_0)dt.
	\end{align}
	Moreover, \cref{lem:derivative_mmd_augmented} also allows to write:
	\begin{align*}
		\dot{\mathcal{F}}(\rho_0) = -\gamma \int \Vert \nabla f_{\mu,\nu_n}(x) \Vert^2 d\nu_n(x); \qquad \vert \dot{\F}(\rho_t)- \dot{\F}(\rho_0)\vert \leq 3L t\gamma^2 \int \Vert \nabla f_{\mu,\nu_n}(X) \Vert^2 d\nu_n(X).
	\end{align*}
	where $t\le 1$. Hence, the result follows directly by applying the above expression to \cref{eq:taylor_expansion_decreasing}.
\end{proof}

\section{Convergence of the gradient flow of the MMD}\label{sec:appendix_convergence}

\subsection{Equilibrium condition}\label{subsec:equilibrium_condition}
We discuss here the equilibrium condition \cref{eq:equilibrium_condition} and relate it to \cite[Assumption A]{Mroueh:2019}. Recall that \cref{eq:equilibrium_condition} is given by: $
\int \Vert \nabla f_{\mu,\nu^{*}}(x) \Vert^2 \diff \nu^{*}(x) = 0$. Under some mild  assumptions on the kernel which are states in \cite[Appendix C.1]{Mroueh:2019} it is possible to write \cref{eq:equilibrium_condition} as:
\[
\int \Vert \nabla f_{\mu,\nu^{*}}(x) \Vert^2 \diff \nu^{*}(x) = \langle f_{\mu,\nu^{*}} ,   D_{\nu^{*}}  f_{\mu,\nu^{*}}\rangle_{\kH}  = 0
\]
where $D_{\nu^{*}}$ is a Hilbert-Schmidt operator given by: %
\[
D_{\nu^{*}} = \int \sum_{i=1}^d \partial_i k(x,.)\otimes \partial_i k(x,.) \diff \nu^{*}(x)
\]
Hence \cref{eq:equilibrium_condition} is equivalent to say that $f_{\mu,\nu^{*}}$ belongs to the null space of $D_{\nu^{*}}$. In \cite[Theorem 2]{Mroueh:2019}, a similar equilibrium condition is derived by considering the time derivative of the MMD along the KSD gradient flow:
\[
\frac{1}{2} \frac{d}{dt} MMD^2(\mu,\nu_t) = - \lambda \langle f_{\mu,\nu_t}, (\frac{1}{\lambda}I - (D_{\nu_t} +\lambda I )^{-1})f_{\mu,\nu_t} \rangle_{\kH} 
\] 
The r.h.s is shown to be always negative and thus the MMD decreases in time. Hence, as $t$ approaches $\infty$, the r.h.s tends to $0$ since the MMD converges to some limit value $l$. This provides the equilibrium condition:
\[
\lambda \langle f_{\mu,\nu^{*}}, (\frac{1}{\lambda}I - (D_{\nu^{*}} +\lambda I )^{-1})f_{\mu,\nu{*}} \rangle_{\kH} = 0
\] 
It is further shown in  \cite[Lemma 2]{Mroueh:2019} that the above equation is also equivalent to having $f_{\mu,\nu^{*}}$ in the null space of $D_{\nu^{*}}$ in the case when $D_{\nu^{*}}$ has finite dimensions. We generalize this statement to infinite dimension in \cref{prop:null_space_diff_operator}. 
In \cite[Assumption A]{Mroueh:2019}, it is simply assumed that if $f_{\mu,\nu^{*}} \neq0$ then  $D_{\nu^{*}} f_{\mu,\nu^{*}} \neq 0 $ which exactly amounts to assuming that local optima which are not global don't exist.
 
\begin{proposition}\label{prop:null_space_diff_operator}
	 \[\langle f_{\mu,\nu^{*}}, (\frac{1}{\lambda}I - (D_{\nu^{*}} +\lambda I )^{-1})f_{\mu,\nu{*}} \rangle_{\kH} = 0 \iff f_{\mu,\nu^{*}} \in null(D_{\nu^{*}})
	 \]
\end{proposition}
\begin{proof}
	This follows simply by recalling $D_{\nu^{*}}$ is a symmetric non-negative Hilbert-Schmidt operator it has therefore an eigen-decomposition of the form:
	\[
	D_{\nu^{*}}  = \sum_{i=1}^{\infty} \lambda_i e_i \otimes e_i
 	\]
 	where $e_i$ is an ortho-norrmal basis of $\kH$ and $\lambda_i$ are non-negative. Moreover, $f_{\mu,\nu^{*}}$ can be decomposed in $(e_i)_{1\leq i}$ in the form:
 	\[
 	f_{\mu,\nu^{*}} = \sum_{i=0}^{\infty} \alpha_i e_i
 	\]
 	where $\alpha_i$ is a squared integrable sequence. It follows that $\langle f_{\mu,\nu^{*}}, (\frac{1}{\lambda}I - (D_{\nu^{*}} +\lambda I )^{-1})f_{\mu,\nu{*}} \rangle_{\kH}$ can be written as:
 	\[
 	\langle f_{\mu,\nu^{*}}, (\frac{1}{\lambda}I - (D_{\nu^{*}} +\lambda I )^{-1})f_{\mu,\nu{*}} \rangle_{\kH} = \sum_{i=1}^{\infty} \frac{\lambda_i}{\lambda_i+\lambda} \alpha_i^2
 	\]
 	Hence, if $f_{\mu,\nu^{*}}\in null(D_{\nu^{*}})$ then $\langle f_{\mu,\nu^{*}}, D_{\nu^{*}}f_{\mu,\nu^{*}}\rangle_{\kH}= 0$, so that $\sum_{i=1}^{\infty} \lambda_i \alpha_i^2 = 0$. Since $\lambda_i$ are non-negative, this implies that $\lambda_i \alpha_i^2= 0$ for all $i$. Therefore, it must be that $\langle f_{\mu,\nu^{*}}, (\frac{1}{\lambda}I - (D_{\nu^{*}} +\lambda I )^{-1})f_{\mu,\nu{*}} \rangle_{\kH}  = 0$.
 	Similarly, if $\langle f_{\mu,\nu^{*}}, (\frac{1}{\lambda}I - (D_{\nu^{*}} +\lambda I )^{-1})f_{\mu,\nu{*}} \rangle_{\kH}  =0 $ then $\frac{\lambda_i\alpha_i^2}{\lambda_i + \lambda} = 0$ hence $\langle f_{\mu,\nu{*}}, D_{\nu^{*}} f_{\mu,\nu{*}} \rangle_{\kH} = 0$. This means that $f_{\mu,\nu{*}}$ belongs to $null(D_{\nu^*})$.
\end{proof}

\subsection{$\Lambda$-displacement convexity of the MMD}\label{subsec:appendix_lambda_convexity}

We provide now a proof of \cref{prop:lambda_convexity}:
\begin{proof}[Proof of \cref{prop:lambda_convexity}]\label{proof:prop:lambda_convexity}[$\Lambda$- displacement convexity of the MMD]
To prove that $\nu\mapsto \F(\nu)$ is $\Lambda$-convex
we need to compute the second time derivative $\ddot{\F}(\rho_{t})$
where $(\rho_{t})_{t \in [0,1]}$ is a displacement geodesic between two probability
distributions $\nu_{0}$ and $\nu_{1}$ as defined in \cref{eq:displacement_geodesic}. Such geodesic always exists and can be written as $\rho_t = (s_t)_{\#}\pi$ with $s_t = x + t(y-x)$ for all $t\in [0,1]$ and $\pi$ is an optimal coupling between $\nu_0$ and $\nu_1$ (\cite{Santambrogio:2015}, Theorem 5.27). We denote by $V_t$ the corresponding velocity vector as defined in \cref{eq:continuity_equation}. Recall that $\F(\rho_t) = \frac{1}{2} \Vert f_{\mu,\rho_t}\Vert^2_{\mathcal{H}}$, with $f_{\mu,\rho_t}$ defined in \cref{eq:witness_function}. %
We start by computing the first derivative of $ t\mapsto \F(\rho_t) $. Since \cref{assump:diff_kernel,assump:lipschitz_gradient_k} hold, \cref{lem:second_derivative_augmented_mmd} applies for $\phi(x,y) = y-x$, $\psi(x,y) = x$ and $q = \pi$, thus we know that $\ddot{\F}(\rho_t)$ is well defined and given by:
\begin{align}\label{eq:hessian}
\begin{split}
	\ddot{\F}(\rho_t) =&\mathbb{E}\left[ (y-x)^T\nabla_1 \nabla_2 k(s_t(x,y),s_t(x',y'))(y'-x')\right]\\
&+ \mathbb{E}\left[ (y-x)^T( H_1 k(s_t(x,y),s_t(x',y'))-H_1 k(s_t(x,y),z))(y-x)\right]
\end{split}
\end{align}
Moreover, \cref{assump:bounded_fourth_oder} also holds which means by \cref{lem:second_derivative_augmented_mmd} that the second term in \cref{eq:hessian} can be lower-bounded by $-\sqrt{2}\lambda d\F(\rho_t)\mathbb{E}[ \Vert y-x \Vert^2]$ so that:
\[
\ddot{\F}(\rho_t) =\mathbb{E}\left[ (y-x)^T\nabla_1 \nabla_2 k(s_t(x,y),s_t(x',y'))(y'-x')\right] - \sqrt{2}\lambda d\F(\rho_t) \mathbb{E}[ \Vert y-x \Vert^2]
\]
Recall now that $(\rho_t)_{t \in [0,1]}$ is a constant speed geodesic with velocity vector $(V_t)_{t\in [0,1]}$ thus by a change of variable, one further has:%
\[
\ddot{\F}(\rho_t) \geq\int\left[ V_t^T(x)\nabla_1 \nabla_2 k(x,x')V_t(x')\right]\diff \rho_t(x) - \sqrt{2}\lambda d\F(\rho_t) \int \Vert V_t(x) \Vert^2 \diff \rho_t(x).
\]
Now we can introduce the function $\Lambda(\rho,v) = \langle v ,( C_{\rho} -\sqrt{2}\lambda d \F(\rho)^{\frac{1}{2}} I) v \rangle_{L_2(\rho)}$
which is defined for any pair $(\rho,v)$ with  $\rho\in \mathcal{P}_2(\X)$ and $v$ a square integrable vector field in $L_2(\rho)$ and where $C_{\rho}$ is a non-negative operator given by  $(C_{\rho}v)(x)=\int\nabla_{x}\nabla_{x'}k(x,x')v(x')d\rho(x')$ for any $x \in \X$.
This allows to write $\ddot{\F}(\rho_t) \geq \Lambda(\rho_t,V_t)$.
It is clear that $\Lambda(\rho,.)$  is a quadratic form on $L_2(\rho)$ and satisfies the requirement in \cref{def:conditions_lambda}.
Finally, using \cref{lem:integral_lambda_convexity} and \cref{def:lambda-convexity} we conclude that $\F$ is $\Lambda$-convex. Moreover, by the reproducing property we also know that for all $\rho \in \mathcal{P}_2(\X)$:
\[ \mathbb{E}_{\rho}\left[ v(x)^T \nabla_1 \nabla_2 k(x,x') v(x') \right]   = \mathbb{E}_{\rho}\left[\left\langle  v(x)^T \nabla_1 k(x,.),  v(x')^T \nabla_1k(x',.) \right\rangle_{\kH}\right].
\] 
By Bochner integrability of $v(x)^T \nabla_1 k(x,.)$ it is possible to exchange the order of the integral and the inner-product \cite[Theorem 6]{Retherford:1978}. This leads to the expression $\Vert \mathbb{E}[v(x)^T \nabla_1 k(x,.)]\Vert^2_{\kH}$. Hence $\Lambda(\rho,v)$ has a second expression of the form:
\[
\Lambda(\rho,v) = \left\Vert \mathbb{E}_{\rho}\left[v(x)^T \nabla_1 k(x,.)\right]\right\Vert^2_{\kH} - \sqrt{2}\lambda d \F(\rho)^{\frac{1}{2}}\mathbb{E}_{\rho}\left[\left\Vert v(x)\right\Vert^2 \right].
\]
\end{proof}

We also provide a result showing $\Lambda$ convexity for $\F$ only under \cref{assump:lipschitz_gradient_k}:

\begin{lemma}[$\Lambda$-displacement convexity]\label{lem:lambda_convexity_bis}
	Under \cref{assump:lipschitz_gradient_k}, for any $\nu,\nu'\in\mathcal{P}_2(\X)$ and any constant speed geodesic $\rho_t$ from $\nu$ to $\nu'$, $\F$ satisfies for all $0\leq t\leq 1$:
	\[
	\F(\rho_t) \leq (1-t)\F(\nu) + t\F(\nu') + 3L W_2^2(\nu,\nu') \qquad 
	\]
\end{lemma}

\begin{proof}
Let $\rho_t$ be a constant speed geodesic of the form $\rho_t = s_t{\#}\pi$ where $\pi$ is an optimal coupling between $\nu$ and $\nu'$ and $s_t(x,y) = x + t(y-x)$.
Since \cref{assump:lipschitz_gradient_k} holds, one can apply \cref{lem:derivative_mmd_augmented} with $\psi(x,y) =x $, $\phi(x,y)= y-x$ and  $q = \pi$. Hence, one has that $\F(\rho_t)$ is differentiable and its differential satisfies:
	\[
	\vert \dot{\F}(\rho_t) - \dot{\F}(\rho_s) \vert  \leq 3L\vert t-s \vert \int \Vert y-x\Vert^2\diff \pi(x,y)   
	\]
	This implies that $\dot{\F}(\rho_t)$ is Lipschitz continuous and therefore is differentiable for almost all $t\in [0,1]$ by Rademacher's theorem. Hence, $\ddot{\F}(\rho_t)$ is well defined for almost all $t\in [0,1]$. Moreover, from the above inequality it follows that $\ddot{\F}(\rho_t)\geq - 3L \int \Vert y-x\Vert^2\diff \pi(x,y) = -3LW_2^2(\nu,\nu')$ for almost all $t\in [0,1]$. Using \cref{lem:integral_lambda_convexity} it follows directly that $\F$ satisfies the desired inequality.
\end{proof}

\subsection{Descent up to a barrier}

To provide a proof of \cref{th:rates_mmd}, we need the following preliminary results. Firstly, an upper-bound on a scalar product involving $\nabla f_{\mu, \nu}$ for any $\mu, \nu \in \mathcal{P}_2(\X)$ in terms of the loss functional $\F$, is obtained using the $\Lambda$-displacement convexity of $\F$ in \cref{lem:grad_flow_lambda_version}. Then, an EVI (Evolution Variational Inequality) is obtained in \cref{prop:evi} on the gradient flow of $\F$ in $W_2$. The proof of the theorem is given afterwards.

\begin{lemma}	\label{lem:grad_flow_lambda_version}
	Let $\nu$ be a distribution in $\mathcal{P}_2(\X)$ and $\mu$ the target distribution such that $\F(\mu)=0$.  Let $\pi$ be an optimal coupling between $\nu$ and $\mu$, and $(\rho_t)_{t \in [0,1]}$ the displacement geodesic defined by \cref{eq:displacement_geodesic} with its corresponding velocity vector  $(V_t)_{t\in[0,1]}$ as defined in \cref{eq:continuity_equation}. Finally let $\nabla f_{\nu,\mu}(X)$ be the gradient of the unnormalised witness function between $\mu$ and $\nu$. The following inequality holds: %
	\begin{align*}
	\int \nabla f_{\mu, \nu}(x).(y-x) d\pi(x,y)
	\leq
	\F(\mu)- \F(\nu) -\int_0^1 \Lambda(\rho_s,V_s)(1-s)ds
	\end{align*}
	where $\Lambda$ is defined \cref{prop:lambda_convexity}.
\end{lemma}
\begin{proof}
	Recall that for all $t\in[0,1]$, $\rho_t$ is given by $\rho_t = (s_t)_{\#}\pi$ with $s_t = x + t(y-x)$. By $\Lambda$-convexity of $\mathcal{F}$ the following inequality holds:
	\begin{align*}
	\mathcal{F}(\rho_{t})\leq (1-t)\mathcal{F}(\nu)+t \mathcal{F}(\mu) - \int_0^1 \Lambda(\rho_s,V_s)G(s,t)ds
	\end{align*}
	Hence by bringing $\mathcal{F}(\nu)$ to the l.h.s and dividing by $t$ and then taking its limit at $0$ it follows that:
	\begin{align}\label{eq:first_order_lambda}
	\dot{\F}(\rho_t)\vert_{t=0}\leq \mathcal	{F}(\mu)-\mathcal{F}(\nu)-\int_0^1 \Lambda(\rho_s,V_s)(1-s)ds.	
	\end{align}
	where $\dot{\F}(\rho_t)=d\F(\rho_t)/dt$ and since $\lim_{t \rightarrow 0}G(s,t)=(1-s)$.
	Moreover, under \cref{assump:lipschitz_gradient_k}, \cref{lem:derivative_mmd_augmented} applies for $\phi(x,y) = y-x$, $\psi(x,y)= x$ and $q = \pi$. It follows therefore that $\dot{\F}(\rho_t)$ is differentiable with time derivative given by:
	$\dot{\F}(\rho_t) = \int \nabla f_{\mu,\rho_t}(s_t(x,y)).(y-x)\diff \pi(x,y)$. Hence at $t=0$ we get: $\dot{\F}(\rho_t)\vert_{t=0} = \int \nabla f_{\mu,\nu}(x).(y-x)\diff \pi(x,y)$ which shows the desired result when used in \cref{eq:first_order_lambda}.
\end{proof}

\begin{proposition}\label{prop:evi}
	Consider the sequence of distributions $\nu_n$ obtained from \cref{eq:euler_scheme}. For $n\ge 0$, consider the scalar
	 $ K(\rho^n) :=  \int_0^1\Lambda(\rho_s^n,V_s^n)(1-s)\diff s$ where $(\rho_s^n)_{0\leq s\leq 1}$ is a \textit{constant speed displacement geodesic} from $\nu_n$ to the optimal value $\mu$ with velocity vectors $(V_s^n)_{0\leq s\leq 1}$. If $\gamma \leq 1/L$, where $L$ is the Lispchitz constant of $\nabla k$ in \cref{assump:lipschitz_gradient_k}, then:
	\begin{align}
	2\gamma(\F(\nu_{n+1})-\F(\mu))
	\leq 
	W_2^2(\nu_n,\mu)-W_2^2(\nu_{n+1},\mu)-2\gamma K(\rho^n).
	\label{eq:evi}
	\end{align}
\end{proposition}

\begin{proof}
	Let $\Pi^n$ be the optimal coupling between $\nu_n$ and $\mu$, then the optimal transport between $\nu_n$ and $\mu$ is given by:
	\begin{align}
	W_2^2(\mu,\nu_n)=\int \Vert X-Y \Vert^2 d\Pi^n(\nu_n,\mu)
	\end{align}
	Moreover, consider $Z=X-\gamma \nabla f_{\mu, \nu_n}(X)$ where $(X,Y)$ are samples from $\pi^n$. It is easy to see that $(Z,Y)$ is a coupling between $\nu_{n+1}$ and $\mu$, therefore, by definition of the optimal transport map between $\nu_{n+1}$ and $\mu$ it follows that:
	\begin{align}\label{eq:optimal_upper-bound}
	W_2^2(\nu_{n+1},\mu)\leq \int \Vert X-\gamma \nabla f_{\mu, \nu_n}(X)-Y\Vert^2 d\pi^n(\nu_n,\mu)
	\end{align}
	By expanding the r.h.s in \cref{eq:optimal_upper-bound}, the following inequality holds:
	\begin{align}\label{eq:main_inequality}
	W_2^2(\nu_{n+1},\mu)\leq W_2^2(\nu_{n},\mu) -2\gamma \int \langle \nabla f_{\mu, \nu_n}(X), X-Y \rangle d\pi^n(\nu_n,\mu)+ \gamma^2D(\nu_n)
	\end{align}
	where $D(\nu_n) = \int \Vert \nabla f_{\mu, \nu_n}(X)\Vert^2 d\nu_n $.
	By \cref{lem:grad_flow_lambda_version} it holds that:
	\begin{align}\label{eq:flow_upper-bound}
	-2\gamma \int  \nabla f_{\mu, \nu_n}(X).(X-Y) d\pi(\nu,\mu)
	\leq
	-2\gamma\left(\F(\nu_n)- \F(\mu) +K(\rho^n)\right)
	\end{align}
	where $(\rho^n_t)_{0\leq t \leq 1}$ is a constant-speed geodesic from $\nu_n$ to $\mu$ and $K(\rho^n):=\int_0^1 \Lambda(\rho^n_s,v^n_s)(1-s)ds$. %
	Note that when $K(\rho^n)\leq 0$ it falls back to the convex setting.
	Therefore, the following inequality holds:
	\begin{align}
	W_2^2(\nu_{n+1},\mu)\leq W_2^2(\nu_{n},\mu) - 2\gamma\left(\F(\nu_n)- \F(\mu) +K(\rho^n)\right) +\gamma^2 D(\nu_n)
	\end{align}
	Now we introduce a term involving $\F(\nu_{n+1})$. The above inequality becomes:
	\begin{align}
	W_2^2(\nu_{n+1},\mu)\leq & W_2^2(\nu_{n},\mu) - 2\gamma\left(\F(\nu_{n+1})- \F(\mu) +K(\rho^n)\right) \\
	&+\gamma^2 D(\nu_n) -2\gamma (\F(\nu_n)-\F(\nu_{n+1}))
	\label{eq:main_ineq_2}
	\end{align}
	It is possible to upper-bound the last two terms on the r.h.s. by a negative quantity when the step-size is small enough. This is mainly a consequence of the smoothness of the functional $\F$ and the fact that $\nu_{n+1}$ is obtained by following the steepest direction of $\F$ starting from $\nu_n$. \cref{prop:decreasing_functional} makes this statement more precise and enables to get the following inequality:
	\begin{align}
	\gamma^2 D(\nu_n) -2\gamma (\F(\nu_n)-\F(\nu_{n+1})\leq -\gamma^2 (1-3\gamma L)D(\nu_n),
	\label{eq:decreasing_functional}
	\end{align}
	where $L$ is the Lispchitz constant of $\nabla k$. Combining  \cref{eq:main_ineq_2} and \cref{eq:decreasing_functional} we finally get:
	\begin{align}
	2\gamma(\F(\nu_{n+1})-\F(\mu))+\gamma^2(1-3\gamma L)D(\nu_n)
	\leq 
	W_2^2(\nu_n,\mu)-W_2^2(\nu_{n+1},\mu)-2\gamma K(\rho^n).
	\label{eq:main_final}
	\end{align}
	and under the condition $\gamma\le 1/(3L)$ we recover the desired result.
\end{proof}

We can now give the proof of the \cref{th:rates_mmd}.

\begin{proof}[Proof of \cref{th:rates_mmd}]\label{proof:th:rates_mmd}
	Consider the Lyapunov function $L_j = j \gamma (\F(\nu_j) - \F(\mu)) + \frac12 W_2^2(\nu_j,\mu)$ for any iteration $j$. At iteration $j+1$, we have:
	\begin{align*}
	L_{j+1} &= j\gamma(\F(\nu_{j+1}) - \F(\mu)) + \gamma(\F(\nu_{j+1}) - \F(\mu)) + \frac12 W_2^2(\nu_{j+1},\mu)\\
	&\leq j\gamma(\F(\nu_{j+1}) - \F(\mu)) + \frac12 W_2^2(\nu_j,\mu)-\gamma K(\rho^j)\\
	&\leq j\gamma(\F(\nu_{j}) - \F(\mu)) + \frac12 W_2^2(\nu_j,\mu)-\gamma K(\rho^j) -j\gamma^2 (1-\frac{3}{2} \gamma L )\int \Vert \nabla f_{\mu, \nu_j}(X)\Vert^2 d\nu_j \\
	&\leq  L_j - \gamma K(\rho^j).
	\end{align*}
	where we used ~\cref{prop:evi} and \cref{prop:decreasing_functional} successively for the two first inequalities. We thus get by telescopic summation: 
	\begin{equation}
	 L_n \leq L_0 -\gamma \sum_{j = 0}^{n-1} K(\rho^j)
	\end{equation}
		Let us denote $\bar{K}$ the average value of $(K(\rho^j))_{0\leq j \leq n}$ over iterations up to $n$. We can now write the final result:
	\begin{equation}
	\F(\nu_{n}) - \F(\mu) \leq \frac{W_2^2(\nu_0, \mu)}{2 \gamma n} -\bar{K}
	\end{equation}
\end{proof}

\subsection{Lojasiewicz type inequalities}\label{subsection:Lojasiewicz}

Given a probability distribution $\nu$, the \textit{weighted Sobolev semi-norm} is defined for all squared integrable functions $f$ in $L_2(\nu)$ as $ \Vert f \Vert_{\dot{H}(\nu)} = \left(\int \left\Vert \nabla f(x) \right\Vert^2 \diff \nu(x) \right)^{\frac{1}{2}}$ with the convention $\Vert f \Vert_{\dot{H}(\nu)} = +\infty$ if $f$ does not have a square integrable gradient. The \textit{Negative weighted Sobolev distance} $ \Vert . \Vert_{\dot{H}^{-1}(\nu)} $ is then defined on distributions as the dual norm of $ \Vert .\Vert_{\dot{H}(\nu)} $. For convenience, we recall the definition of $ \Vert . \Vert_{\dot{H}^{-1}(\nu)} $:
\begin{definition}\label{def:neg_sobolev_appendix}
	Let $\nu\in \mathcal{P}_2(\x)$, with its corresponding \textit{weighted Sobolev semi-norm} $ \Vert . \Vert_{\dot{H}(\nu)} $. %
	The \textit{weighted negative Sobolev distance} $\Vert p - q \Vert_{\dot{H}^{-1}(\nu)}$ between any $p$ and $q$ in $\mathcal{P}_2(\x)$  is defined as
\begin{align}\label{eq:neg_sobolev}
	\Vert p - q \Vert_{\dot{H}^{-1}(\nu)} = \sup_{f\in L_2(\nu), \Vert f \Vert_{\dot{H}(\nu)} \leq 1 } \left\vert \int f(x)\diff p(x) - \int f(x)\diff q(x) \right\vert 
\end{align}	
with possibly infinite values.
\end{definition}
There are several possible choices for the set of test functions $f$. While it is often required that $f$ vanishes at the boundary (see \cite{Mroueh:2019}), we do not make such restriction and rather use the definition from \cite{Peyre:2011}. We refer to \cite{Shestakov:2009} for more discussion on the relationship between different choices for the set of test functions.

We provide now a proof for \cref{prop:lojasiewicz}.
\begin{proof}[Proof of \cref{prop:lojasiewicz}]\label{proof:prop:lojasiewicz}
	This proof follows simply from the definition of the negative Sobolev distance. Under \cref{assump:lipschitz_gradient_k}, the kernel has at most quadratic growth hence, for any $\mu,\nu \in \mathcal{P}_2(\X)^2$, $f_{\mu,\nu}\in L_2(\nu)$. Consider $g = \Vert f_{\mu, \nu_t}\Vert^{-1}_{\dot{H}(\nu_t)} f_{\mu, \nu_t}$, then $g\in L_2(\nu_t)$ and $\Vert g \Vert_{\dot{H}(\nu_t)}\leq 1$. Therefore, we directly have:
	\begin{align}\label{eq:loja1}
	\left\vert \int g \diff \nu_t - \int g \diff \mu  \right\vert \leq \left\Vert \nu_t - \mu\right\Vert_{\dot{H}^{-1}(\nu_t)} 
	\end{align}
	Now, recall the definition of $g$, which implies that
	\begin{equation}\label{eq:loja2}
	\left\vert \int g \diff \nu_t - \int g \diff \mu  \right\vert = \left\Vert \nabla f_{\mu, \nu_t}\right\Vert^{-1}_{L_2(\nu_t)} \left\vert \int f_{\mu, \nu_t}\diff \nu_t-\int f_{\mu, \nu_t} \diff \mu \right\vert.
	\end{equation}
	Moreover,  we have that $\int f_{\mu, \nu_t}\diff \nu_t-\int f_{\mu,\nu_t}\diff \mu = \Vert f_{\mu, \nu_t}\Vert^2_{\kH}$, since $f_{\mu, \nu_t}$  is the unnormalised witness function between $\nu_t$ and $\mu$. Combining \cref{eq:loja1} and \cref{eq:loja2} we thus get the desired Lojasiewicz inequality on $f_{\mu,\nu_t}$:
	\begin{equation}
	\Vert f_{\mu,\nu_t} \Vert^2_{\mathcal{H}} \leq \Vert f_{\mu,\nu_t} \Vert_{\dot{H}(\nu_t)} \Vert  \mu -\nu_t\Vert_{\dot{H}^{-1}(\nu_t)}  
	\end{equation}
	where $\Vert f_{\mu,\nu_t} \Vert_{\dot{H}(\nu_t)}=\Vert \nabla f_{\mu, \nu_t} \Vert_{L_2(\nu_t)}$ by definition. Then, 
	using \cref{prop:decay_mmd} and recalling by assumption that: $\Vert \mu - \nu_t \Vert^2_{\dot{H}^{-1}(\nu_t)} \le C$, we have:  
	\begin{align}\label{eq:PL_inequality}
	\dot{\F}(\nu_t) = - \Vert \nabla f_{\mu, \nu_t} \Vert^2_{L_2(\nu_t)} \leq -\frac{1}{C}\Vert f_{\mu,\nu_t} \Vert^4_{\mathcal{H}}= -\frac{4}{C}\F(\nu_t)^2	
	\end{align}
	It is clear that if $\mathcal{F}(\nu_0)>0$ then $\F(\nu_t)>0$ at all times by uniqueness of the solution. Hence, one can divide by $\F(\nu_t)^2$ and integrate the inequality from $0$ to some time $t$. The desired inequality is obtained by simple calculations.
Then, using \cref{prop:decreasing_functional} and \cref{eq:PL_inequality} where $\nu_t$ is replaced by $\nu_n$ it follows:
$$
\cF(\nu_{n+1}) - \cF(\nu_n) \leq -\gamma\left(1-\frac{3}{2} L\gamma\right)\|\nabla f_{\mu,\nu_n}\|_{L_2(\nu_n)}^2 \leq -\frac{4}{C}\gamma\left(1-\frac{3}{2}\gamma L\right)\cF(\nu_n)^2.
$$
Dividing by both sides of the inequality by  $ \F(\nu_n)\F(\nu_{n+1})$ and recalling that $\F(\nu_{n+1})\leq \F(\nu_n)$ it follows directly that: 
$$\frac{1}{\cF(\nu_n)} - \frac{1}{\cF(\nu_{n+1})} \leq -\frac{4}{C}\gamma\left(1-\frac{3}{2}\gamma L\right).$$ The proof is concluded by summing over $n$ and rearranging the terms. 
\end{proof}

\subsection{A simple example}\label{subsec:simple_example}
Consider a gaussian target distribution $\mu(x)  = \mathcal{N}(a,\Sigma) $ and initial distribution $\nu_0 = \mathcal{N}(a_0,\Sigma_0)$. In this case it is sufficient to use a kernel that captures the first and second moments of the distribution. We simply consider a kernel of the form $k(x,y)= (x^{\top}y)^2 + x^{\top}y$. In this case, it is easy to see by simple computations that the following equation holds:
\begin{align}\label{eq:example_1_mckean_vlassov}
	\dot{X}_t = - (\Sigma_t-\Sigma  + a_t a_t^{\top}-a a^{\top} )X_t - (a_t-a),\qquad \forall t \geq 0
\end{align}
Where $a_t$ and $\Sigma_t$ are the mean and covariance matrix of $\nu_t$ and satisfy the equations:
\begin{align}
	\dot{\Sigma}_t &= - (S_t \Sigma_t +  \Sigma_t S_t )\\ 
	\dot{a}_t  &= - S_t a_t -(a_t-a).
\end{align}
Where we introduced $S_t =  \Sigma_t-\Sigma  + a_t a_t^{\top}-aa^{\top}$ for simplicity.
\cref{eq:example_1_mckean_vlassov} implies that $\nu_t$ is in fact a gaussian distribution since $X_t$ is obtained by summing gaussian increments. The same conclusion can be reached by solving the corresponding continuity equation. Thus we will be only interested in the behavior of $a_t$ and $\Sigma_t$. First we can express the squared MMD in terms of those parameters:
\begin{align}\label{eq:example_1_MMD}
	MMD^2(\mu,\nu_t) = \Vert S_t \Vert^2 + \Vert a_t-a \Vert^2.
\end{align}
Since $a_t$ and $\Sigma_t$ are obtained from the gradient flow of the MMD, it follows that $\Vert a_t-a \Vert^2$ and $\Vert S_t \Vert^2$ remain bounded.
Moreover, the Negative Sobolev distance is obtained by solving a finite dimensional quadratic problem and can be simply written as:
\begin{align}
	D(\mu,\nu_t) = tr(Q_t \Sigma_t Q_t)  + \Vert a_t-a\Vert^2 
\end{align}
where $Q_t$ is the unique solution of the Lyapounov equation:
\begin{align}\label{eq:Lyapounov}
	\Sigma_t Q_t + Q_t \Sigma_t = \Sigma_t- \Sigma + (a_t-a)(a_t-a)^{\top}:=G_t.
\end{align}
We first consider the one dimensional case, for which \cref{eq:Lyapounov} has a particularly simple solution and allows to provide a closed form expression for the negative Sobolev distance:
\begin{align}
	Q_t= \frac{G_t}{2\Sigma_t}, \qquad D(\mu,\nu_t) = \frac{G_t^2}{4\Sigma_t} + (a_t-a)^2.
\end{align}
Recalling \cref{eq:example_1_MMD} and that $MMD^2(\mu,\nu_t)$ is  bounded at all times by definition of $\nu_t$, it follows that  both $G_t$ and $a_t-a$ are also bounded. Hence, it is easy to see that  $D(\mu,\nu_t)$ will remain bounded iff $\Sigma_t$ remains bounded away from $0$. This analysis generalizes the higher dimensions using \cite[Lemma 3.2 (iii)]{Behr:2018} which provides an expression for $Q_t$ in terms of $G_t$ and the singular value decomposition of $\Sigma_t = U_t D_t U_t^{\top}$:
\begin{align}
	Q_t = U_t \left( \left(\frac{1}{(D_t)_i + (D_t)_j }\right)\odot U_t^{\top} G_t U_t\right) U_t^{\top}.
\end{align}
Here, $\odot$ denotes the Hadamard product of matrices. It is easy to see from this expression that $D(\mu,\nu_t)$ will be bounded if all singular values $((D_t)_i)_{1\leq i \leq d}$ of $\Sigma_t$ remain bounded away from $0$.

\subsection{Lojasiewicz-type inequalities for $\F$ under different metrics}\label{subsec:Lojasiewicz_different_metrics}

The Wasserstein gradient flow of $\F$ can be seen as the continuous-time limit of the so called minimizing movement scheme \cite{ambrosio2008gradient}. Such proximal scheme is defined using an initial distribution $\nu_0$, a step-size $\tau$, and an iterative update equation:
\begin{align}\label{eq:minimizing_movement_scheme}
	\nu_{n+1} \in \arg \min _{\nu} \F(\nu) + \frac{1}{2\tau} W_2^2(\nu,\nu_n).
\end{align}
In \cite{ambrosio2008gradient}, it is shown that the continuity equation $\partial_t \nu_t = div(\nu_t \nabla  f_{\mu,\nu_t})$ can be obtained as the limit when $\tau \rightarrow 0$ of \cref{eq:minimizing_movement_scheme} using suitable interpolations between the elements $\nu_n$.
In \cite{rotskoff2019global}, a different transport equation that includes a birth-death term is considered:
\begin{align}\label{eq:birth_death}
	\partial_t \nu_t = \beta div(\nu_t \nabla f_{\mu,\nu_t}) + \alpha(f_{\mu,\nu_t} - \int f_{\mu,\nu_t}(x)\diff \nu_t(x) )\nu_t
\end{align}
When $\beta=0$ and $\alpha =1$, it is shown formally in \cite{rotskoff2019global} that the above dynamics corresponds to the limit of a proximal scheme using the KL instead of the Wasserstein distance.
For general $\beta$ and $\alpha$, \cref{eq:birth_death} corresponds to the limit of a different proximal scheme where $W_2^2(\nu,\nu_n)$ is replaced by the Wasserstein-Fisher-Rao distance $d^2_{\alpha,\beta}(\nu,\nu_n)$ (see \cite{Chizat:2015,liero2016optimal,kondratyev2016new}). $d^2_{\alpha,\beta}(\nu,\nu_n)$ is an interpolation between the squared Wasserstein distance  ($\beta =1$ and $\alpha =0$) and the squared Fisher-Rao distance as defined in \cite[Definition 6]{Chizat:2015} ($\beta=0$ and $\alpha = 1$). Such scheme is consistent with the one proposed in \cite{rotskoff2019global} and  which uses the $KL$. In fact, as we will show later, both the $KL$ and the Fisher-Rao distance have the same local behavior therefore both proximal schemes are expected to be  equivalent in the limit when $\tau \rightarrow 0$.

Under \cref{eq:birth_death}, the time evolution of $\F$ is given by \cite[Proposition 3.1]{rotskoff2019global}: 
\begin{align}\label{eq:birth_death}
\dot{\F}(\nu_t) = -\beta\int \Vert \nabla f_{\mu,\nu_t}\Vert^2 \diff \nu_t(x) -\alpha \int \left\vert f_{\mu,\nu_t}(x)-\int f_{\mu,\nu_t}(x')\diff \nu_t(x')\right\vert^2\diff \nu_t(x) 
\end{align}
We would like to apply the same approach as in \cref{sec:Lojasiewicz_inequality} to provide a condition on the convergence of \cref{eq:birth_death}. 
Hence we first introduce an analogue to the Negative Sobolev distance in \cref{def:neg_sobolev} by duality:
\begin{align}
	D_{\nu}(p,q) =\sup_{\substack{g\in L_2(\nu)\\ \beta \Vert \nabla g \Vert^2_{L_2(\nu)} +\alpha  \Vert g- \bar{g}  \Vert^2_{L_2(\nu)}  \leq 1 }} \left\vert \int g(x)\diff p(x) - \int g(x) \diff q(x)\right\vert 
\end{align}
where $\bar{g}$ is simply the expectation of $g$ under $\nu$.
Such quantity defines a distance, since it is the dual of a semi-norm. Now using the particular structure of the MMD, we recall that $f_{\mu,\nu}\in L_2(\nu)$ and that $\beta \Vert \nabla f \Vert^2_{L_2(\nu)} +\alpha  \Vert f- \bar{f}  \Vert^2_{L_2(\nu)}<\infty$. Hence for a particular $g$ of the form:
\[
g = \frac{f_{\mu,\nu}}{\left(\beta \Vert \nabla f_{\mu,\nu} \Vert^2_{L_2(\nu)} +\alpha  \Vert f_{\mu,\nu}- \bar{f}_{\mu,\nu}  \Vert^2_{L_2(\nu)} \right)^\frac{1}{2}}
\]
the following inequality holds:
\[
D_{\nu}(\mu,\nu) \geq \frac{\left\vert \int f_{\mu,\nu} \diff \nu(x) -   \int f_{\mu,\nu} \diff\mu(x)\right\vert }{\left(\beta \Vert \nabla f_{\mu,\nu} \Vert^2_{L_2(\nu)} +\alpha  \Vert f_{\mu,\nu}- \bar{f}_{\mu,\nu}  \Vert^2_{L_2(\nu)}\right)^{\frac{1}{2}} }.
\]
But since $f_{\mu,\nu}$ is the unnormalised witness function between $\mu$ and $\nu$ we have that $2\F(\nu) = \left\vert \int f_{\mu,\nu} \diff \nu(x) - \int f_{\mu,\nu} \diff \mu(x)\right\vert $. Hence one can write that:
\begin{align}
	D^2_{\nu}(\mu,\nu)\left(\beta \Vert \nabla f_{\mu,\nu} \Vert^2_{L_2(\nu)} +\alpha  \Vert f_{\mu,\nu}- \bar{f}_{\mu,\nu}  \Vert^2_{L_2(\nu)}\right) \geq 4\F^2(\nu)
\end{align}
Now provided that $D^2_{\nu}(\mu,\nu_t)$ remains bounded at all time $t$ by some constant $C>0$ one can easily deduce a rate of convergence for $\F(\nu_t)$ just as in \cref{prop:lojasiewicz}. In fact, in the case when $\beta = 1$ and $\alpha =0$ one recovers \cref{prop:lojasiewicz}. Another interesting case is when $\beta =0$ and $\alpha=1$. In this case, $D_{\nu}(p,q)$ is defined for $p$ and $q$ such that the difference $p-q$ is absolutely continuous w.r.t. $\nu$. Moreover, $D_{\nu}(p,q)$ has the simple expression:
\[
D_{\nu}(p,q) = \int \left(\frac{p-q }{\nu }(x)\right)^2 \diff \nu(x)
\] 
where $\frac{ p-q }{ \nu }$ denotes the radon nikodym density of $p-q$ w.r.t. 
$\nu$. More importantly, $D^2_{\nu}(\mu,\nu)$ is exactly equal to $\chi^2(\mu\Vert \nu)^{\frac{1}{2}}$. 
As we will show now, $(\chi^2)^{\frac{1}{2}}$ turns out to be a linearization of $\sqrt{2} KL^{\frac{1}{2}}$ and the Fisher-Rao distance.
\paragraph{Linearization of the KL and the Fisher-Rao distance.} We first show the result for the KL. Given a probability distribution $\nu'$ that is absolutely continuous  w.r.t to $\nu$ and for $0<\epsilon< 1$ denote by $G(\epsilon) := KL(\nu \Vert (\nu+\epsilon(\nu'-\nu) )$. It can be shown that $G(\epsilon) = \frac{1}{2}\chi^2(\nu'\Vert\nu)\epsilon^2 +o(\epsilon^2)$. To see this, one needs to perform a second order Taylor expansion of $G(\epsilon)$ at $\epsilon=0$. Exchanging the derivatives and the integral, $\dot{G}(\epsilon)$ and $\ddot{G}(\epsilon)$ are both given by:
\begin{align*}
	\dot{G}(\epsilon) = -\int \frac{\mu-\nu}{\nu +\epsilon(\mu-\nu)}\diff \nu\\
	\ddot{G}(\epsilon) = \int \frac{(\nu-\mu)^2}{(\nu+\epsilon(\mu-\nu))^2} \diff \nu
\end{align*}
Hence, we have for $\epsilon=0$:  $\dot{G}(0) = 0$ and $\ddot{G}(0) = \chi^2(\mu\Vert \nu)$. Therefore, it follows:
$G(\epsilon) =\frac{1}{2} \chi^2(\mu\Vert \nu) \epsilon^2 + o(\epsilon^2)$, which means that  
\begin{align*}
	\lim_{\epsilon \rightarrow 0} \frac{1}{\epsilon}\left[2KL\left(\nu \Vert \nu+\epsilon(\nu'-\nu) \right) \right]^{\frac{1}{2}}   = \chi^2(\nu'\Vert \nu)^\frac{1}{2}.
\end{align*}
The same approach can be used for the Fisher-Rao distance $d_{0,1}(\nu,\nu')$. From \cite[Theorem 3.1]{Chizat:2015} we have that:
\[
d^2_{0,1}(\nu,\nu') = 2\int (\sqrt{\nu(x)}-\sqrt{\nu'(x)})^2\diff x 
\]
where $\nu$ and $\nu'$ are assumed to have a density w.r.t. Lebesgue measure. Using the exact same approach as for the KL one easily show that $	\lim_{\epsilon \rightarrow 0} \frac{1}{\epsilon}\left[2d^2_{0,1}\left(\nu \Vert \nu+\epsilon(\nu'-\nu) \right) \right]^{\frac{1}{2}}   = \chi^2(\nu'\Vert \nu)^\frac{1}{2}.$

\paragraph{Linearization of the $W_2$.} Similarly, it can be shown that the  \textit{Negative weighted Sobolev distance} is a linearization of the $W_2$ under suitable conditions. We recall here \cite[Theorem 7.26]{Villani:2003} which relates the two quantities: 
\begin{theorem}\label{thm:villani}
	Let $\nu\in \mathcal{P}(\X)$ be a probability measure with finite second moment, absolutely continuous w.r.t the Lebesgue measure and let $h\in L^{\infty}(\X)$ with $\int h(x)\diff \nu(x)=0$. Then
	\[
	\Vert h \Vert_{\dot{H}^{-1}(\nu)}\leq \lim\inf_{\epsilon \rightarrow 0} \frac{1}{\epsilon} W_2(\nu,(1+\epsilon h )\nu). 
	\]
\end{theorem}
\cref{thm:villani} implies that for any probability distribution $\nu'$ that has a bounded density w.r.t. to $\nu$ one has:
\[
\Vert \nu'-\nu \Vert_{\dot{H}^{-1}(\nu)}\leq \lim\inf_{\epsilon \rightarrow 0} \frac{1}{\epsilon} W_2(\nu,\nu+\epsilon (\nu'-\nu)).
\]
To get the converse inequality, one needs to assume that the support of  $\nu$ is $\X$. \cref{prop:converse_sobolev_wasserstein} provides such inequality and uses techniques from \cite{Peyre:2011}.
\begin{proposition}\label{prop:converse_sobolev_wasserstein}
	Let $\nu\in \mathcal{P}(\X)$ be a probability measure with finite second moment, absolutely continuous w.r.t the Lebesgue measure with support equal to $\X$ and let $h\in L^{\infty}(\X)$ with $\int h(x)\diff \nu(x)=0$ and $1+h\geq 0$. Then
	\[
	\lim\sup_{\epsilon \rightarrow 0} \frac{1}{\epsilon} W_2(\nu,(1+\epsilon h )\nu)\leq \Vert h \Vert_{\dot{H}^{-1}(\nu)}
	\]
\end{proposition}
\begin{proof}
Consider the elliptic equation: $\nu h + div(\nu \nabla F) = 0$ with Neumann boundary condition on $\partial \X$. Such equation admits a unique solution $F$ in $\dot{H}(\nu)$ up to a constant since $\nu$ is supported on all of $\X$ (see \cite[Section 7 (Linearizations)]{Otto:2000}). Moreover, we have that $ \int F(x)h(x)\diff \nu(x) = \int \Vert \nabla F(x) \Vert^2 \diff \nu(x)$ which implies that $\Vert h \Vert_{\dot{H}^{-1}(\nu)} \geq  \Vert F \Vert_{\dot{H}(\nu)}$. Now consider the path: $s_u = (1 + u\epsilon h)\nu$ for $u\in[0,1]$. $s_u$ is a probability distribution for all $u\in [0,1]$ with $s_0=  \nu$ and $s_1 = 	(1+\epsilon h)\nu$. It is easy to see that $s_u$ satisfies the continuity equation:
\[
\partial_u s_u +div(s_u V_u )=0
\]
 with $V_u =  \frac{\epsilon\nabla F}{1+u\epsilon h}$. Indeed, for any smooth test function $f$ one has:
 \begin{align*}
 	 \frac{\diff}{\diff u}\int f(x)\diff s_u(x) = \epsilon \int f(x)h(x)\diff \nu(x) = \epsilon \int \nabla f(x).\nabla F(x) \diff \nu(x) = \int \nabla f(x).V_u(x)\diff s_u(x).
 \end{align*}
We used the definition of $F$ for the second equality and that $\nu$ admits a density w.r.t. to $s_u$ provided that $\epsilon$ is small enough. Such density is given by $1/(1+u\epsilon h)$ and is positive and bounded when $\epsilon\leq \frac{1}{2\Vert h \Vert_{\infty} }$.   
Now, using the Benamou-Brenier formula for $W_2(\nu,(1+\epsilon h)\nu)$ one has in particular that:
\[
W_2(\nu,(1+\epsilon h)\nu)\leq \int \Vert V_u \Vert_{L^2(s_u)} \diff u
\]
Using the expressions of $V_u$ and $s_u$, one gets by simple computation:
\begin{align*}
W_2(\nu,(1+\epsilon h)\nu)\leq & \epsilon\int \left(\int\frac{\Vert \nabla F(x) \Vert^2}{1-u\epsilon + u\epsilon (h+1)} \diff \nu(x) \right)^{\frac{1}{2}} \diff u\\
&\leq \epsilon \left( \int \Vert \nabla F(x) \Vert^2\diff \nu(x) \right)^{\frac{1}{2}} \int_0^1 (1-u\epsilon)^{-\frac{1}{2}}\diff u.
\end{align*}
Finally, $\epsilon \int_0^1 (1-u\epsilon)^{-\frac{1}{2}}\diff u = 2(1-\sqrt{1 - \epsilon}) \rightarrow 1$ when $\epsilon\rightarrow 0$, hence:
\[
\lim\sup_{\epsilon\rightarrow 0} W_2(\nu,(1+\epsilon h)) \leq \Vert F\Vert_{\dot{H}(\nu)}\leq \Vert h \Vert_{\dot{H}^{-1}(\nu)}.  
\]
\end{proof}
\cref{thm:villani} and \cref{prop:converse_sobolev_wasserstein} allow to conclude that $
\lim_{\epsilon \rightarrow 0} \frac{1}{\epsilon} W_2(\nu,\nu +\epsilon(\nu'-\nu)) = \Vert \nu - \nu' \Vert_{\dot{H}^{-1}(\nu)}$
for any $\nu'$ that has a bounded density w.r.t. $\nu$.

By analogy, one could wonder if $D$ is also a linearization of the the Wasserstein-Fisher-Rao distance. We leave such question for future work.

\section{Algorithms}\label{sec:appendix_algorithms}

\subsection{Noisy Gradient flow of the MMD}
\begin{proof}[Proof of \cref{thm:convergence_noisy_gradient}]\label{proof:thm:convergence_noisy_gradient}
To simplify notations, we write $\mathcal{D}_{\beta_n}(\nu_n)  = \int \Vert V(x+\beta_n u) \Vert^2 g(u)\diff \nu_n \diff u $ where  $V := \nabla f_{\mu,\nu_n}$ and $g$ is the density of a standard gaussian. The symbol $\otimes$ denotes the product of two independent probability distributions. Recall that a sample $x_{n+1}$ from $\nu_{n+1}$ is obtained using  $x_{n+1} = x_n - \gamma V(x_n+ \beta_n u_n)$
	where $x_n$ is a sample from $\nu_n$ and $u_n$ is a sample from a standard gaussian distribution that is independent from $x_n$. Moreover, by assumption $\beta_n$ is a non-negative scalar satisfying:
	\begin{align}\label{eq:control_noise_level_bis}
		8\lambda^2\beta_n^2 \F(\nu_n) \leq \mathcal{D}_{\beta_n}(\nu_n)  
	\end{align}
	 Consider now the map $(x,u)\mapsto s_t(x)= x - \gamma tV(x+\beta_n u)$ for $0\leq t\leq 1$, then $\nu_{n+1}$ is obtained as a push-forward of $\nu_n\otimes g$ by $s_1$: $\nu_{n+1} = (s_1)_{\#}(\nu_n\otimes g)$. Moreover, the curve $\rho_t = (s_t)_{\#}(\nu_n\otimes g)$ is a path from $\nu_n$ to $\nu_{n+1}$. We know by \cref{prop:grad_witness_function} that $\nabla f_{\mu,\nu_n}$ is $2L$-Lipschitz, thus using $\phi(x,u) = -\gamma V(x+\beta_n u)$, $\psi(x,u) = x$ and $q = \nu_n\otimes g $ in \cref{lem:derivative_mmd_augmented} it follows that $\F(\rho_t)$ is differentiable in $t$ with:
	 \begin{equation*}
	 \dot{\F}(\rho_t)=\int \nabla f_{\mu,\rho_t}(s_t(x)).(-\gamma V(x+\beta_n u))g(u)\diff \nu_n(x)\diff u
	 \end{equation*} 
	 Moreover, $\dot{\F}(\rho_0)$ is given by $\dot{\F}(\rho_0)= -\gamma \int V(x).V(x+\beta_n u) g(u)\diff\nu_n(x)\diff u$ and the following estimate holds:
	 \begin{align}\label{eq:estimate_gradient}
	 	\vert \dot{\F} (\rho_t) -\dot{\F}(\rho_0)\vert \leq 3\gamma^2 L t \int \Vert V(x+\beta_n u) \Vert^2 g(u)\diff\nu_n(x)\diff u = 3\gamma^2 Lt \mathcal{D}_{\beta_n}(\nu_n).
	 \end{align}
	Using the absolute continuity of $\F(\rho_t)$, one has $\mathcal{F}(\nu_{n+1})-\mathcal{F}(\nu_{n})
	=\dot{\F}(\rho_0)+ \int_0^1 \dot{\F} (\rho_t) -  \dot{\F} (\rho_0) \diff t $. Combining with  \cref{eq:estimate_gradient} and using the expression of $\dot{\F}(\rho_0)$, it follows that:
	\begin{align}\label{eq:taylor_expansion}
	\mathcal{F}(\nu_{n+1})-\mathcal{F}(\nu_{n})
	\leq -\gamma \int V(x).V(x+\beta_n u) g(u)\diff\nu_n(x)\diff u + \frac{3}{2}\gamma^2L \mathcal{D}_{\beta_n}(\nu_n).
	\end{align} 
Adding and subtracting  $\gamma \mathcal{D}_{\beta_n}(\nu_n)$ in \cref{eq:taylor_expansion} it follows directly that:
\begin{align}\label{eq:penultimate}
\begin{split}
			\mathcal{F}(\nu_{n+1})-\mathcal{F}(\nu_{n} )\leq &   -\gamma (1-\frac{3}{2}\gamma L )\mathcal{D}_{\beta_n}(\nu_n)
 \\
 &+ \gamma\int  (V(x+\beta_n u) -V(x)).V(x+\beta_n u) g(u)\diff\nu_n(x)\diff u	
\end{split}
\end{align}
We shall control now the last term in \cref{eq:penultimate}. Recall now that for all $1\le i\le d$, $ V_i(x) = \partial_i f_{\mu,\nu_n}(x) = \langle f_{\mu,\nu_n} , \partial_i k(x,.)\rangle $ where we used the reproducing property for the derivatives of $f_{\mu,\nu_n}$ in $\kH$ (see \cref{sec:rkhs}). Therefore, it follows by Cauchy-Schwartz in $\kH$ and using \cref{assump:Lipschitz_grad_rkhs}:
\begin{align*}
\Vert V(x+\beta_n u) -V(x)\Vert^2
&\leq 
\Vert f_{\mu,\nu_n} \Vert_{\mathcal{H}}^2  \left( \sum_{i=1}^{d}\Vert \partial_i k(x+\beta_n u,.) -\partial_i k(x,.)\Vert^2_{\mathcal{H}}\right)\\
&\leq \lambda^2\beta_n^2
\Vert   f_{\mu,\nu_n}\Vert_{\mathcal{H}}^2\Vert u \Vert^2 
\end{align*}
for all $ x,u \in \X$. Now integrating both sides w.r.t. $\nu_n$ and $g$ and recalling that $g$ is a standard gaussian, we have:
\begin{align}
	 \int  \Vert V(x+\beta_n u) -V(x)\Vert^2 g(u)\diff\nu_n(x)\diff u
\leq 
	\lambda^2\beta^2_n\Vert f_{\mu,\nu_n} \Vert_{\mathcal{H}}^2
\end{align}
Getting back to \cref{eq:penultimate} and applying Cauchy-Schwarz in $L_2(\nu_n\otimes g)$ it follows:
\begin{align}
	\mathcal{F}(\nu_{n+1})-\mathcal{F}(\nu_{n} )\leq &   -\gamma (1-\frac{3}{2}\gamma L )\mathcal{D}_{\beta_n}(\nu_n) +\gamma \lambda\beta_n\Vert f_{\mu,\nu_n} \Vert_{\mathcal{H}}\mathcal{D}^{\frac{1}{2}}_{\beta_n}(\nu_n)
\end{align}
It remains to notice that $\Vert f_{\mu,\nu_n} \Vert_{\mathcal{H}}^2 = 2\F(\nu_n)$ and that $\beta_n$ satisfies \cref{eq:control_noise_level_bis} to get:
\[
\F(\nu_{n+1}) -\F(\nu_n) \leq -\frac{\gamma}{2}(1-\frac{3}{2}\gamma L)\mathcal{D}_{\beta_n}(\nu_n).
\]
We introduce now $\Gamma = 4\gamma(1-\frac{3}{2}\gamma L)\lambda^2$ to simplify notation and prove the second inequality. Using \cref{eq:control_noise_level_bis} again in the above inequality we directly have: $\F(\nu_{n+1}) -\F(\nu_n) \leq - \Gamma\beta_n^2 \F(\nu_n)$. One can already deduce that $\Gamma\beta_n^2$ is necessarily smaller than $1$. Hence, taking $\F(\nu_n)$ to the r.h. side and iterating over $n$ it follows that: 
\[
\F(\nu_{n}) \leq \F(\nu_0)\prod_{i=0}^{n-1}(1- \Gamma\beta_n^2)
\]
Simply using that $1-\Gamma\beta_n^2\leq e^{-\Gamma\beta_n^2}$ leads to the desired upper-bound $\F(\nu_{n}) \leq \F(\nu_0)e^{-\Gamma \sum_{i=0}^{n-1} \beta_n^2}$.
\end{proof}

\subsection{Sample-based approximate scheme}
\begin{proof}[Proof of \cref{prop:convergence_euler_maruyama}]\label{proof:propagation_chaos}
Let $(u_{n}^{i})_{1\leq i\leq N}$ be i.i.d standard gaussian variables and $(x_{0}^{i})_{1\leq i\leq N}$ i.i.d. samples from $\nu_0$. We consider $(x_n^i)_{1\leq i\leq N}$ the particles obtained using the approximate scheme \cref{eq:euler_maruyama}: $x_{n+1}^{i}=x_{n}^{i}-\gamma\nabla f_{\hat{\mu},\hat{\nu}_{n}}(x_{n}^{i}+\beta_{n}u_{n}^{i})$ starting from $(x_{0}^{i})_{1\leq i\leq N}$, where $\hat{\nu_n}$ is the empirical distribution of these $N$ interacting particles. Similarly, we denote by $(\bar{x}_{n}^{i})_{1\leq i\leq N}$ the particles obtained using the exact update equation \cref{eq:discretized_noisy_flow}: $\bar{x}_{n+1}^{i}=\bar{x}_{n}^{i}-\gamma\nabla f_{\mu,\nu_{n}}(\bar{x}_{n}^{i}+\beta_{n}u_{n}^{i})$ also starting from $(x_{0}^{i})_{1\leq i\leq N}$.
By definition of $\nu_n$ we have that  $(\bar{x}_{n}^{i})_{1\leq i\leq N}$ are i.i.d. samples drawn from $\nu_n$ with empirical distribution denoted by $\bar{\nu}_{n}$.
We will control the expected error $c_{n}$ defined as  $c^2_{n}= \frac{1}{N}\sum_{i=1}^N \mathbb{E}\left[\Vert x_{n}^{i}-\bar{x}_{n}^{i}\Vert^{2}\right]$. By recursion, we have:
\begin{align*}
c_{n+1} = & \frac{1}{\sqrt{N}}\left(\sum_{i=1}^{N}\mathbb{E}\left[\left\Vert x_{n}^{i}-\bar{x}_{n}^{i}-\gamma\left(\nabla f_{\hat{\mu},\hat{\nu}_{n}}(x_{n}^{i}+\beta_{n}u_{n}^{i})-\nabla f_{\mu,\nu_{n}}(\bar{x}_{n}^{i}+\beta_{n}u_{n}^{i})\right)\right\Vert^{2}\right]\right)^{\frac{1}{2}}\label{eq:main_inequality_c_n_1}\\
\leq &  c_{n} +\frac{\gamma}{\sqrt{N}}\left[\sum_{i=1}^{N}\mathcal{E}_{i}\right]^{\frac{1}{2}}+\frac{\gamma}{\sqrt{N}}\left[\sum_{i=1}^{N}\mathcal{G}_{i}\right]^{\frac{1}{2}}  \\
& +\frac{\gamma}{\sqrt{N}}\left(\sum_{i=1}^{N}\mathbb{E}\left[\left\Vert\nabla f_{\mu,\hat{\nu}_{n}}\left(x_{n}^{i}+\beta_{n}u_{n}^{i}\right)-\nabla f_{\mu,\bar{\nu}_{n}}\left(\bar{x}_{n}^{i}+\beta_{n}u_{n}^{i}\right)\right\Vert^{2}\right]\right)^{\frac{1}{2}}\\
  \leq & c_{n}+2\gamma L\left(c_{n}+\mathbb{E}\left[W_{2}(\hat{\nu}_{n},\bar{\nu}_{n})^{2}\right]^{\frac{1}{2}}\right)+\frac{\gamma}{\sqrt{N}}\left[\sum_{i=1}^{N}\mathcal{E}_{i}\right]^{\frac{1}{2}}+\frac{\gamma}{\sqrt{N}}\left[\sum_{i=1}^{N}\mathcal{G}_{i}\right]^{\frac{1}{2}}
\end{align*}
where the second line follows from a simple triangular inequality and the last line is obtained recalling that $\nabla f_{\mu,\nu}(x)$
is jointly $2L$ Lipschitz in $x$ and $\nu$ by \cref{prop:grad_witness_function}. Here, $\mathcal{E}_{i}$ represents the error between $\bar{\nu}_n$ and $\nu_n$ while $\mathcal{G}_{i}$  represents the error between $\hat{\mu}$ and $\mu$ and are given by: 
\begin{align*}
\mathcal{E}_{i} & =\mathbb{E}\left[\left\Vert\nabla f_{\mu,\bar{\nu}_{n}}(\bar{x}_{n}^{i}+\beta_{n}u_{n}^{i})-\nabla f_{\mu,\nu_{n}}(\bar{x}_{n}^{i}+\beta_{n}u_{n}^{i})\right\Vert^{2}\right]\\
\mathcal{G}_{i} & =\mathbb{E}\left[\left\Vert\nabla f_{\hat{\mu},\hat{\nu}_{n}}(x_{n}^{i}+\beta_{n}u_{n}^{i})-\nabla f_{\mu,\hat{\nu}_{n}}(x_{n}^{i}+\beta_{n}u_{n}^{i})\right\Vert^{2}\right]
\end{align*}
We will first control the error term $\mathcal{E}_i$. To simplify
notations, we write $y^{i}=\bar{x}_{n}^{i}+\beta_{n}u_{n}^{i}$. Recalling the expression of $\nabla f_{\mu,\nu}$ from \cref{prop:grad_witness_function} and expanding the squared norm in $\mathcal{E}_i$, it follows:
\begin{align*}
\mathcal{E}_{i} & =\mathbb{E}\left[\left\Vert\frac{1}{N}\sum_{j=1}^{N}\nabla k(y^{i},\bar{x}_{n}^{j})-\int\nabla k(y^{i},x)d\nu_{n}(x)\right\Vert^{2}\right]\\
 & =\frac{1}{N^{2}}\sum_{j=1}^{N}\mathbb{E}\left[\left\Vert\nabla k(y^{i},\bar{x}_{n}^{j})-\int\nabla k(y^{i},x)d\nu_{n}(x)\right\Vert^{2}\right]\\
 & \leq\frac{L^{2}}{N^{2}}\sum_{j=1}^{N}\mathbb{E}\left[\left\Vert\bar{x}_{n}^{j}-\int xd\nu_{n}(x)\right\Vert^{2}\right]=\frac{L^{2}}{N}var(\nu_{n}).
\end{align*}
The second line is obtained using the independence of the auxiliary samples $(\bar{x}^{i}_n)_{1\le i\le N}$ and recalling that they are distributed according to $\nu_{n}$. The last line uses the fact that $\nabla k(y,x)$ is $L$-Lipshitz in $x$ by \cref{assump:lipschitz_gradient_k}. To control the variance $var(\nu_n)$ we use  \cref{lem:Control_variance} which implies that $var(\nu_{n})^{\frac{1}{2}}\leq(B+var(\nu_{0})^{\frac{1}{2}})e^{LT}$ for all $n\leq\frac{2T}{\gamma}$.
For $\mathcal{G}_{i}$, it is sufficient to expand again the squared norm and recall that $\nabla k(y,x)$ is $L$-Lipschitz in $x$ which then implies that $\mathcal{G}_{i}\leq\frac{L^{2}}{M}var(\mu)$. Finally, one can observe that  $\mathbb{E}[W_{2}^{2}(\hat{\nu}_{n},\bar{\nu}_{n})]\leq\frac{1}{N}\sum_{i=1}^{N}\mathbb{E}\left[\Vert x_{n}^{i}-\bar{x}_{n}^{i}\Vert^{2}\right]=c_{n}^{2}$, hence $c_n$ satisfies the recursion:
\[
c_{n+1}\leq(1+4\gamma L)c_{n}+\frac{\gamma L}{\sqrt{N}}(B+var(\nu_{0})^{\frac{1}{2}})e^{2LT}+\frac{\gamma L}{\sqrt{M}}var(\mu).
\]
Using \cref{lem:Discrete-Gronwall-lemma} to solve the above inequality, it follows that:
\[
c_{n}\leq\frac{1}{4}\left(\frac{1}{\sqrt{N}}(B+var(\nu_{0})^{\frac{1}{2}})e^{2LT}+\frac{1}{\sqrt{M}}var(\mu))\right)(e^{4LT}-1)
\]
\end{proof}

\begin{lemma}\label{lem:Control_variance}
Consider an initial distribution
$\nu_{0}$ with finite variance, a sequence $(\beta_n)_{ n \geq 0}$ of non-negative numbers bounded by $B<\infty$ and define the sequence of probability distributions $\nu_n$ of the process \cref{eq:discretized_noisy_flow}:
\[
x_{n+1}=x_{n}-\gamma\nabla f_{\mu,\nu_{n}}(x_{n}+\beta_{n}u_{n}) \qquad x_0 \sim \nu_0
\]
where $(u_n)_{n\geq 0}$ are standard gaussian variables. 
Under \cref{assump:lipschitz_gradient_k},  the variance of
$\nu_{n}$ satisfies for all $T>0$ and $n\leq\frac{T}{\gamma}$ the following inequality:
\[
var(\nu_{n})^{\frac{1}{2}}\leq(B+var(\nu_{0})^{\frac{1}{2}})e^{2TL}
\]
\end{lemma}
\begin{proof}
Let $g$ be the density of a standard gaussian. Denote by $(x,u)$ and $(x',u')$ two independent samples from $\nu_n\otimes g$. The idea is to find a recursion from $var(\nu_{n})$ to $var(\nu_{n+1})$: 
\begin{align*}
var(\nu_{n+1})^{\frac{1}{2}} 
 & =\left(\mathbb{E}\left[\left\Vert x -\mathbb{E}\left[x'\right] -\gamma\nabla f_{\mu,\nu_{n}}(x+\beta_{n}u)+\gamma\mathbb{E}\left[\nabla f_{\mu,\nu_{n}}(x'+\beta_{n}u')\right]\right\Vert^2\right]\right)^{\frac{1}{2}}\\
 & \leq var(\nu_{n})^{\frac{1}{2}}+\gamma\left(\mathbb{E}\left[\left\Vert\nabla f_{\mu,\nu_{n}}(x+\beta_{n}u)-\mathbb{E}\left[\nabla f_{\mu,\nu_{n}}(x'+\beta_{n}u')\right]\right\Vert^{2}\right]\right)^{\frac{1}{2}}\\
 & \leq var(\nu_{n})^{\frac{1}{2}}+2\gamma L\mathbb{E}_{\substack{x,x'\sim\nu_{n}\\ u,u'\sim g}}\left[\left\Vert x+\beta_{n}u-x'+\beta_{n}u'\right\Vert^{2}\right]^{\frac{1}{2}}\\
 & \leq var(\nu_{n})^{\frac{1}{2}}+2\gamma L(var(\nu_{n})^{\frac{1}{2}}+\beta_{n})
\end{align*}
The second and last lines are obtained using a triangular inequality while the third line uses that $\nabla f_{\mu,\nu_n}(x)$ is $2L$-Lipschitz in $x$ by  \cref{prop:grad_witness_function}. Recalling that $\beta_{n}$ is bounded by $B$ it is easy to conclude using \cref{lem:Discrete-Gronwall-lemma}. 
\end{proof}

\section{Connection with Neural Networks}\label{subsec:training_neural_networks}
In this sub-section we establish a formal connection between the MMD gradient flow defined in \cref{eq:continuity_mmd} and neural networks optimization. Such connection holds in the limit of infinitely many neurons and is based on the formulation in \cite{rotskoff2018neural}. To remain consistent with the rest of the paper, the parameters of a network will be denoted by $x\in \X$ while the input and outputs will be denoted as $z$ and $y$.
 Given a neural network or any parametric function $(z,x)\mapsto \psi(z,x)$ with parameter $x \in \X $ and input data $z$ we consider the supervised learning problem:
\begin{align}\label{eq:regression_network}
	\min_{(x_1,...,x_m )\in \X} \frac{1}{2}\mathbb{E}_{(y,z)\sim p  } \left[ \left\Vert y - \frac{1}{m}\sum_{i=1}^m\psi(z,x_i) \right\Vert^2 \right ]
\end{align}
where $(y,z) \sim p$ are samples from the data distribution and the regression function is an average of $m$ different networks. The formulation in \cref{eq:regression_network} includes any type of networks. Indeed, the averaged function can itself be seen as one network with augmented parameters $(x_1,...,x_m)$ and any network can be written as an average of sub-networks with potentially shared weights. In the limit $m\rightarrow \infty$, the average can be seen as an expectation over the parameters under some probability distribution $\nu$. This leads to an expected network $\Psi(z,\nu) =  \int \psi(z,x) \diff \nu(x) $ and the optimization problem in \cref{eq:regression_network} can be lifted to an optimization problem in $\mathcal{P}_2(\X)$ the space of probability distributions:
\begin{align}\label{eq:lifted_regression}
	\min_{\nu \in \mathcal{P}_2(\X)}  \mathcal{L}(\nu) :=  \frac{1}{2}\mathbb{E}_{(y,z)\sim p} \left [ \left\Vert y - \int \psi(z,x) \diff \nu(x) \right\Vert^2 \right ]
\end{align} 
For convenience, we consider $\bar{\mathcal{L}}(\nu)$ the function obtained by subtracting the variance of $y$ from $\mathcal{L}(\nu)$, i.e.: $\bar{\mathcal{L}}(\nu) = \mathcal{L}(\nu) - var(y) $. When the model is well specified, there exists $\mu \in \mathcal{P}_2(\X) $ such that $\mathbb{E}_{y\sim \mathbb{P}(.|z)}[y] =  \int \psi(z,x) \diff \mu(x)$. In that case, the cost function $\bar{\mathcal{L}}$ matches  the functional $\F$ defined in \cref{eq:mmd_as_free_energy}  for a particular choice of the kernel $k$. More generally, as soon as a global minimizer for  \cref{eq:lifted_regression} exists,  \cref{prop:inequality_mmd_loss} relates the two losses $\bar{\mathcal{L}}$ and $\mathcal{F}$.
\begin{proposition}\label{prop:inequality_mmd_loss}
	Assuming a global minimizer of \cref{eq:lifted_regression} is achieved by some $\mu\in \mathcal{P}_2(\X)$, the following inequality holds for any $\nu \in \mathcal{P}_2(\X)$:
	\begin{align}\label{eq:inequality_mmd_nn}
		\left(\bar{\mathcal{L}}(\mu)^{\frac{1}{2}} + \F^{\frac{1}{2}}(\nu)\right)^2
		\geq 
		\bar{\mathcal{L}}(\nu)
		\geq
		\mathcal{F}(\nu) + \bar{\mathcal{L}}(\mu)
	\end{align}
	where $\F(\nu)$ is defined by \cref{eq:mmd_as_free_energy} with  a kernel $k$  constructed from the data as an expected product of networks:
\begin{align}\label{eq:kernel_NN}
	k(x,x') = \mathbb{E}_{z\sim \mathbb{P}} \left[\psi(z,x)^T\psi(z,x')\right]
\end{align}
Moreover, $\bar{\mathcal{L}} = \F$ iif $\bar{\mathcal{L}}(\mu)=0$, which means that the model is well-specified. 
\end{proposition}
The framing \cref{eq:inequality_mmd_nn} implies that optimizing $\mathcal{F}$ can decrease  $\mathcal{L}$ and vice-versa. 
Moreover, in the well specified case, optimizing  $\mathcal{F}$ is equivalent to optimizing $\mathcal{L}$. Hence one can use the gradient flow of the MMD defined in \cref{eq:continuity_mmd} to solve \cref{eq:lifted_regression}.
One particular setting when \cref{eq:lifted_regression} is well-specified is the student-teacher problem as in \cite{Chizat:2018a}. In this case, a teacher network of the form $\Psi_T(z,\mu)$ produces a deterministic output $y = \Psi_T(z,\mu)$ given an input $z$ while a student network $\Psi_S(z,\nu)$ tries to learn the mapping $z\mapsto \Psi_T(z,\mu)$ by minimizing \cref{eq:lifted_regression}. In practice $\mu$ and $\nu$ are given as empirical distributions on some particles $\Xi = (\xi^1,...,\xi^M)$ and $X=(x^1,...,x^N)$ with $\mu = \frac{1}{M} \sum_{j=1}^M \delta_{\xi^j}$  and $\nu= \frac{1}{N} \sum_{i=1}^N\delta_{x^i}$. The particles $(x^i)_{1\leq i \leq N}$ are then optimized using gradient descent starting from an initial configuration $(x_0^i)_{1\leq i \leq N}$. This leads to the update equation:
\begin{align}\label{eq:update_equation_student_teacher}
	x^i_{n+1} = x^i_n - \gamma \mathbb{E}_{z\sim p }\left[ \left(\frac{1}{N}\sum_{j=1}^N \psi(z,x_n^{j})-\frac{1}{M}\sum_{j=1}^M \psi(z,\xi^{j})\right)\nabla_{x_n^{i}}\psi(z,x_n^{i})\right],
\end{align}
 where $(x_n^{i})_{1\leq i\leq N}$ are the particles at iteration $n$ with empirical distribution $\nu_n$. Here, the gradient is rescaled by the number of particles $N$.
Re-arranging terms and recalling that $k(x,x') = \mathbb{E}_{z\sim p}[\psi(z,x)^T\psi(z,x')]$,  equation \cref{eq:update_equation_student_teacher} becomes:
\[
x^i_{n+1} = x^i_n - \gamma \nabla f_{\mu,\nu_n}(x_n^i).
\]
with $\nabla f_{\mu,\nu_n}(x_n^i) = \left(\frac{1}{N}\sum_{j=1}^N \nabla_2 k(x_n^{j},x_n^{i})-\frac{1}{M}\sum_{j=1}^M \nabla_2 k(\xi^{j},x_n^{i})\right)$.
The above equation is a discretized version of the gradient flow of the MMD defined in \cref{eq:continuity_mmd}. Such discretization is obtained from \cref{eq:euler_maruyama} by setting the noise level  $\beta_n$ to $0$.
Hence, in the limit when $N\rightarrow \infty$ and $\gamma\rightarrow 0$, one recovers the gradient flow defined in \cref{eq:euler_scheme_particles}.
In general the kernel $k$ is intractable and can be approximated using  $n_b$ samples $(z_1,...,z_{n_b})$ from the data distribution: $\hat{k}(x,x') = \frac{1}{n_b} \sum_{b=1}^{n_b} \psi(z_b,x)^T \psi(z_b,x')$. This finally leads to an approximate update:
\[
x^i_{n+1} = x^i_n - \gamma \nabla \hat{f}_{\mu,\nu_n}(x_n^i).
\]
where $\nabla \hat{f}_{\mu,\nu_n}$ is given by:
\[
\nabla\hat{f}_{\mu,\nu_n}(x_n^i) = \frac{1}{n_b} \sum_{b=1}^{n_b} \left(\frac{1}{N}\sum_{j=1}^N \psi(z_b,x_n^{j})-\frac{1}{M}\sum_{j=1}^M \psi(z_b,\xi^{j})\right)\nabla_{x_n^{i}}\psi(z_b,x_n^{i})).
\]
We provide now a proof for \cref{prop:inequality_mmd_loss}:
\begin{proof}[Proof of \cref{prop:inequality_mmd_loss}]\label{proof:prop:inequality_mmd_loss}Let $\Psi(z,\nu)$=$\int \psi(z,x)\diff \nu(x)$. By  \cref{eq:kernel_NN}, we have: $k(x,x') =\int_{z}\psi(z,x)^T\psi(z,x')\diff s(z)$ where $s$ denotes the distribution of $z$. It is easy to see that $\F(\nu) = \frac{1}{2} \int \Vert \Psi(z,\nu) -\Psi(z,\mu)  \Vert^2 \diff s(z) $. Indeed expanding the square in the l.h.s and exchanging the order of integrations w.r.t $p$ and $(\mu\otimes \nu)$ one gets $\F(\nu)$.
	Now, introducing $\Psi(z,\mu)$ in the expression of $\mathcal{L}(\nu)$, it follows by a simple calculation that:
\begin{align}\label{eq:main_eq_nn}
 \mathcal{L}(\nu)&= \mathcal{L}(\mu)+ \mathcal{F}(\nu)+ \int \left\langle\Psi(z,\mu)-m(z),\Psi(z,\nu)-\Psi(z,\mu)\right\rangle\diff p(z)
\end{align}
where $m(z)$ is the conditional mean of $y$, i.e.: $m(z)=\int y \diff p(y|z)$. On the other hand we have that $2\mathcal{L}(\mu) = var(y) + \int \Vert \Psi(z,\mu)-m(z)\Vert^2\diff p(z)$, so that $ \int \Vert \Psi(z,\mu)-m(z)\Vert^2\diff p(z) = 2\bar{\mathcal{L}}(\mu)$. Hence, using Cauchy-Schwartz for the last term in \cref{eq:main_eq_nn}, one gets the upper-bound:
\[
 \mathcal{L}(\nu)\leq \mathcal{L}(\mu)+ \mathcal{F}(\nu) + 2 \bar{\mathcal{L}}(\mu)^{\frac12}\mathcal{F(\nu)}^{\frac12}.
\]
This in turn gives an upper-bound on $\bar{\mathcal{L}}(\nu)$ after subtracting $var(y)/2$ on both sides of the inequality.
To get the lower bound on $\bar{\mathcal{L}}$ one needs to use the global optimality condition of $\mu$ for $\mathcal{L}$ from \cite[Proposition 3.1]{chizat2018global}. Indeed, for any $0<\epsilon\leq 1$ it is easy to see that:
\[
\epsilon^{-1}( \mathcal{L}(\mu +\epsilon(\nu-\mu))-\mathcal{L}(\mu)) = \int \left\langle \Psi(z,\mu)-m(z),\Psi(z,\nu)-\Psi(z,\mu)\right\rangle\diff p(z) +o(\epsilon).
\]
Taking the limit $\epsilon\rightarrow 0$ and recalling that the l.h.s is always non-negative by optimality of $\mu$, it follows that $\int \ps{\Psi(z,\mu)-m(z),\Psi(z,\nu)-\Psi(z,\mu)}\diff p(z)$ must also be non-negative. Therefore, from \cref{eq:main_eq_nn} one gets that $\mathcal{L}(\nu) \geq  \mathcal{L}(\mu)+ \mathcal{F}(\nu)$. The final bound is obtained by subtracting $var(y)/2$ again from both sides of the inequality.
\end{proof}

\section{Numerical Experiments}\label{sec:experiments}
\subsection{Student-Teacher networks}\label{sec:experiments_neural_network}
We consider a student-teacher network setting similar to \cite{Chizat:2018a}. More precisely, using the notation from \cref{subsec:training_neural_networks}, we denote by $\Psi(z,\nu)$  the neural network of the form: $\Psi (z,\nu) = \int  \psi(z,x)\diff \nu(x) $ where $z$ is an input vector in $\R^{p}$ and  $\nu$ is a probability distribution over the parameters $x$. Hence $\Psi$ is an expectation over sub-networks $\psi(z,x)$ with parameters $x$. Here, we choose $\psi$ of the form:
\begin{align}
	\psi(z,x) = G\left(b^{1}+W^{1}\sigma(W^{0}z+b^{0})\right).
\end{align}
where $x$ is obtained as the concatenation of the parameters $(b^{1},W^{1},b^{0},W^{0})\in \X$, $\sigma$ is the ReLU non-linearity while $G$ is a fixed function and is defined later.
Note that using $x$ to denote the parameters of a neural network is unusual, however, we prefer to keep a notation which is consistent with the rest of the paper.
We will only consider the case when $\nu$ is given by an empirical distribution of $N$ particles  $X  = (x^{1},...x^{N})$ for some $N\in\mathbb{N}$. In that case, we denote by  $\nu_{X}$ such distribution to stress the dependence on the particles $X$, i.e.: $ \nu := \nu_{X}= \frac{1}{N} \sum_{i=1}^N \delta_{x^{i}}$. 
The teacher network $\Psi_{T}(z,\nu_{\Xi})$ is given by $M$ particles $\Xi = (\xi_1,...,\xi_M)$ which are fixed during training and are initially drawn  according to a normal distribution $\mathcal{N}(0,1)$. Similarly, the student network $\Psi_{S}(z,\nu_{X})$ has $N$ particles $X = (x^{1},...,x^{N})$ that are initialized according to a normal distribution $\mathcal{N}(10^{-3},1)$. 
Here we choose $M=1$ and $N=1000$.
  The inputs $z$ are drawn from a uniform distribution $\mathbb{S}$ on the sphere in $\R^p$ as in \cite{Chizat:2018a} with $p=50$. The number of hidden layers $H$ is set to $3$ and the output dimension is $1$.   
The parameters of the student networks are trained to minimize the risk in \cref{eq:student_teacher_problem} using SGD with mini-batches of size $n_b  = 10^2$ and optimal step-size $\gamma$ selected from: $\{10^{-3},10^{-2},10^{-1}\}$.
\begin{align}\label{eq:student_teacher_problem}
	\min_{X} \mathbb{E}_{z\sim \mathbb{S} }\left[(\Psi_T(z,\nu_{\Xi} )- \Psi_S(z,\nu_{X}))^2\right]
\end{align}
When $G$ is simply the identity function and no bias is used, one recovers the setting in \cite{chizat2018global}. In that case the network is partially $1$-homogeneous and \cite[Theorem 3.5]{chizat2018global} applies ensuring global optimality. Here, we are interested in the case when global optimality is not guaranteed by the homogeneity structure, hence we choose $G$ to be a gaussian with fixed bandwidth $\sigma=2$. 
As shown in \cref{subsec:training_neural_networks}, performing gradient descent to minimize \cref{eq:student_teacher_problem} can be seen as a particle version of the gradient flow of the MMD with a kernel given by  $k(x,x') = \mathbb{E}_{z\sim\mathbb{S}}[\psi(z,x)\psi(z,x')]$  and target distribution $\mu$ given by $\mu= \nu_{\Xi}$. Hence one can use the noise injection algorithm defined in \cref{eq:euler_maruyama} to train the parameters of the student network. Since $k$ is defined through an expectation over the data, it can be approximated using $n_{b}$ data samples  $\{z_{1},...,z_{B}\}$:
\begin{align}\label{eq:random_features_kernel}
	\hat{k}(x,x') = \frac{1}{n_b} \sum_{b=1}^{n_b} \psi(z_b,x)\psi(z_b,x').
\end{align}
	
Such approximation of the kernel leads to a simple expression for the gradient of the unnormalised witness function between  $\nu_{\Xi}$ and $\nu_{X}$:
\begin{align}\label{eq:approximate_gradient_witness_student_teacher}
	\nabla \hat{f}_{\nu_{\Xi},\nu_{X}}(x) = \frac{1}{n_b}\sum_{b=1}^{n_b}\left( \frac{1}{M}\sum_{j=1}^M\psi(z_b,\xi^j) -  \frac{1}{N}\sum_{i=1}^N\psi(z_b , x^i)\right)\nabla_{x}\psi(z_b,x), \qquad \forall x \in \X.
\end{align}
\cref{euclid_student_teacher}, provides the main steps to train the parameters of the student network using the noisy gradient flow of the MMD proposed in \cref{eq:euler_maruyama}. It can be easily implemented using automatic differentiation packages like \verb+PyTorch+. Indeed, one only needs to compute an auxiliary loss function $\F_{aux}$ instead of the actual MMD loss $\F$ and perform gradient descent using $\F_{aux}$. Such function is given by:
\[
\F_{aux} = \frac{1}{n_b}\sum_{i=1}^N\sum_{b=1}^{n_b} \left(\verb+NoGrad+\left(y_S^b\right) - y_T^b \right)\psi(z^b,\widetilde{x}_n^{i})
\]
To compute $\F_{aux}$, two forward passes on the student network are required.
A first forward pass using the current parameter values  $X_n = (x_n^1,...,x_n^{N})$ of the student network is used to compute the predictions $y_S^b$ given an input $z^b$. For such forward pass, the gradient w.r.t to the parameters $X_n$ is not used. This is enforced, here, formally by calling the function \verb+NoGrad+.  
 The second forward pass is performed using the noisy parameters $\widetilde{x}_n^{i} = x_n^i + \beta_n u_n^{i}$  and  requires implementing special layers which can inject noise to the weights. This second forward pass will be used to provide a gradient to update the particles using back-propagation. Indeed, it is easy to see that $\nabla_{x_n^{i}} \F_{aux}$ gives exactly the gradient $\nabla \hat{f}_{\nu_{\Xi},\nu_X}(\widetilde{x}_n^i)$ used in  \cref{euclid_student_teacher}.

\subsection{Learning gaussians}\label{sec:experiments_gaussian}
\begin{figure}[ht]
	\centering
	\includegraphics[width=0.8\linewidth]{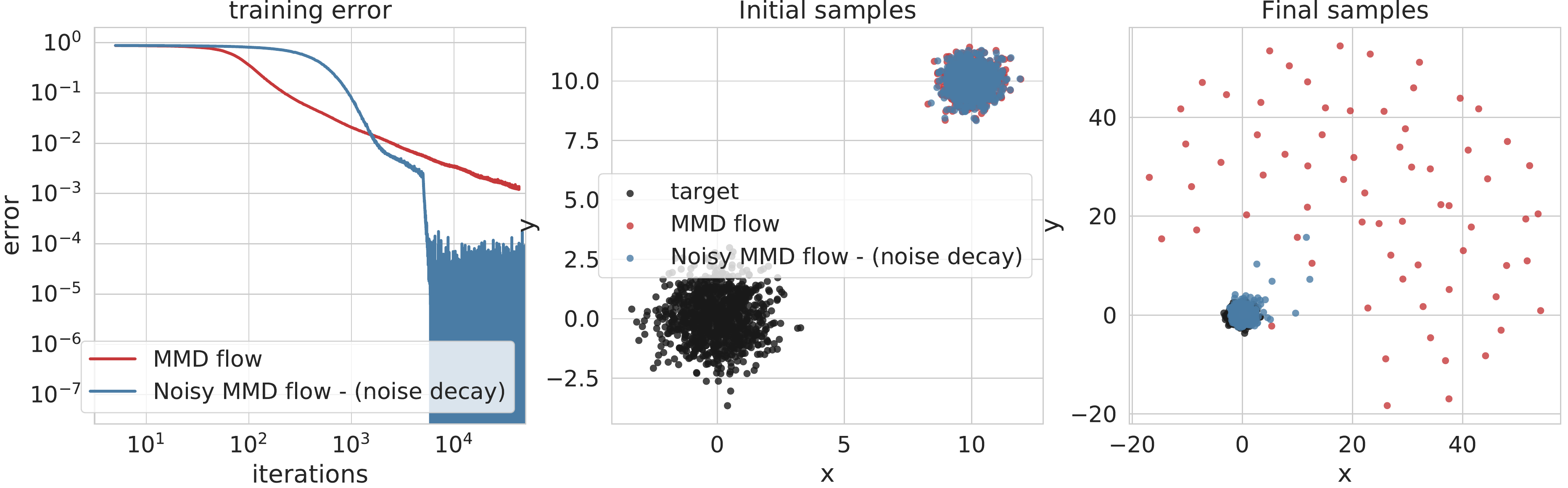}
	\caption{Gradient flow of the $MMD$ from a gaussian initial distributions $\nu_0\sim \mathcal{N}(10,0.5)$  towards a target distribution $\mu\sim \mathcal{N}(0,1)$ using $N=M=1000$ samples from $\mu$ and $\nu_0$ and a gaussian kernel with bandwidth $\sigma = 2 $. \cref{eq:euler_maruyama} is used 
	without noise $\beta_n = 0$ in red and  with noise $\beta_n = 10$ up to $n=5000$, then $\beta_n = 0$ afterwards in blue. 
	The left figure shows the evolution of the $MMD$ at each iteration. The middle figure shows the initial samples (black for $\mu$), and the right figure shows the final samples after $10^5$ iterations with step-size $\gamma = 0.1$.}
	\label{fig:experiments}
\end{figure}
\cref{fig:experiments} illustrates the behavior of the proposed algorithm \cref{eq:euler_maruyama} in a simple setting, and compares it with the gradient flow of the MMD without noise injection. In this setting, the MMD flow  fails to converge to the global optimum. Indeed, as shown in \cref{fig:experiments}(right), some of the final samples (in red) obtained using noise-free gradient updates tend to get further away from the target samples (in black). Most of the remaining samples collapse to a unique point at the center near the origin. This can also be seen from \cref{fig:experiments}(left) where the training error fails to decrease below $10^{-3}$. On the other hand, adding noise to the gradient seems to lead to global convergence, as seen visually from the samples. The training error decreases below $10^{-4}$ and oscillates between $10^{-8}$ and $10^{-4}$. The oscillation is due to the step-size, which remained fixed while the noise was set to $0$ starting from iteration $5000$. It is worth noting that adding noise to the gradient slows the speed of convergence, as one can see from \cref{fig:experiments}(left). This is expected since the algorithm doesn't follow the path of steepest descent. The noise helps in escaping local optima, however, as  illustrated here.

\begin{algorithm}\label{subsec:pseudocode}
	\setstretch{1.35}
	\caption{Noisy gradient flow of the MMD}\label{euclid}
	\begin{algorithmic}[1]
		\State \textbf{Input} $N$, $n_{iter}$, $\beta_0$, $\gamma$
		\State \textbf{Output} $(x^{i}_{n_{iter}})_{1\le i\le N}$
		\State \textit{Initialize $N$ particles from initial distribution $\nu_0$} : $x_{0}^{i}\mysim \nu_0$
		\State \textit{Initialize the noise level}: $\beta=\beta_0$
		\For{$n=0,\dots, n_{iter}$}
		\State \textit{Sample $M$ points from the target $\mu$}: $\{y^1,...,y^M\}$.
		\State \textit{Sample $N$ gaussians }: $\{u_n^{1},...,u_n^N\}$
		\For{$i=1,\dots,N$}
		\State \textit{Compute the noisy values}: $\widetilde{x}_n^{i} = x_n^i+\beta_n u_n^i$
		\State \textit{Evaluate  vector field}:$\nabla f_{\hat{\mu},\hat{\nu}_n}(\widetilde{x}_n^i) = \frac{1}{N}\sum\limits_{j=1}^N \nabla_2 k(x_n^j,\widetilde{x}_n^{i})-\frac{1}{M}\sum\limits_{m=1}^M \nabla_2 k(y^m,\widetilde{x}_n^{i})$
		\State \textit{Update the particles}: $x_{n+1}^{i} = x_n^i -\gamma \nabla f_{\hat{\mu},\hat{\nu}_n}(\widetilde{x}_n^i)$
		\EndFor
		\State \textit{Update the noise level using an update rule $h$}: $\beta_{n+1}=h(\beta_{n}, n)$.
		\EndFor
	\end{algorithmic}
\end{algorithm}

\begin{algorithm}\label{subsec:pseudocode_student_teacher}
	\setstretch{1.35}
	\caption{Noisy gradient flow of the MMD for student-teacher learning}\label{euclid_student_teacher}
	\begin{algorithmic}[1]
		\State \textbf{Input} $N$, $n_{iter}$, $\beta_0$, $\gamma$, $n_{b}$, $\Xi = (\xi^j)_{1\leq j\leq M}$.
		\State \textbf{Output} $(x^{i}_{n_{iter}})_{1\le i\le N}$.
		\State \textit{Initialize $N$ particles from initial distribution $\nu_0$} : $x_{0}^{i}\mysim \nu_0$.
		\State \textit{Initialize the noise level}: $\beta=\beta_0$.
		\For{$n=0,...,n_{iter}$}
		\State \textit{Sample minibatch of $n_{b}$ data points}: $\{z^1,...,z^{n_{b}}\}$.
		\For{$b=1,...,n_{b}$} 
		\State \textit{Compute teacher's output}: $y_{T}^b = \frac{1}{M}\sum_{j=1}^M \psi(z^b,\xi^{j})$.
		\State \textit{Compute students's output}: $y_{S}^b = \frac{1}{N}\sum_{i=1}^N \psi(z^b,x_n^i)$.
		\EndFor
		\State \textit{Sample $N$ gaussians }: $\{u_n^{1},...,u_n^{N}\}$.
		\For{$i=1,...,N$}
		\State \textit{Compute noisy particles}: $\widetilde{x}_n^{i} = x_n^i +\beta_n u_n^{i}$ 
		\State \textit{Evaluate vector field}: $  \nabla \hat{f}_{\nu_{\Xi},\nu_{X_n}}(\widetilde{x}_n^{i}) =  \frac{1}{n_{b}}\sum_{b=1}^{n_{b}} ( y_{S}^b - y_{T}^b ) \nabla_{x_n^{i}} \psi(z^b,\widetilde{x}_n^{i})$
		\State \textit{Update particle $i$}: $x_{n+1}^{i} = x_{n}^{i} -\gamma \nabla \hat{f}_{\nu_{\Xi},\nu_{X_n}}(\widetilde{x}_n^{i})$
		\EndFor
		\State \textit{Update the noise level using an update rule $h$}: $\beta_{n+1}=h(\beta_{n}, n)$.
		\EndFor
	\end{algorithmic}
\end{algorithm}

\section{Auxiliary results}\label{sec:auxiliary_results}

\begin{proposition}\label{prop:grad_witness_function}
Under \cref{assump:lipschitz_gradient_k}, the unnormalised witness function $f_{\mu,\nu}$ between any probability distributions $\mu$ and $\nu$ in $\mathcal{P}_2(\X)$ is differentiable and satisfies:
\begin{equation}\label{eq:gradient_witness}
\nabla f_{\mu,\nu}(z) = \int \nabla_1 k(z,x)\diff \mu(x) - \int \nabla_1 k(z,x)\diff \nu(x) \qquad \forall z\in \X
\end{equation}
where $z \mapsto \nabla_1 k(x,z)$ denotes the gradient of $z\mapsto k(x,z)$ for a fixed $x \in \X$.
 Moreover, the map $(z,\mu,\nu)\mapsto f_{\mu,\nu}(z)$ is Lipschitz with:
\begin{equation}\label{eq:lipschitz_grad_witness}
\Vert \nabla f_{\mu,\nu}(z) - \nabla f_{\mu',\nu'}(z')\Vert \leq 2L (\Vert z-z' \Vert + W_2(\mu,\mu') + W_2(\nu,\nu')) 
\end{equation}
Finally, each component of $\nabla f_{\mu,\nu}$ belongs to $\kH$.
\end{proposition}
\begin{proof}
	The expression of the unnormalised witness function is given in \cref{eq:witness_function}. To establish \cref{eq:gradient_witness}, we simply need to apply the differentiation lemma \cite[Theorem 6.28]{Klenke:2008}. By \cref{assump:lipschitz_gradient_k}, it follows that $ (x,z)\mapsto \nabla_1 k(z,x)$ has at most a linear growth. Hence on any bounded neighborhood of $z$, $x\mapsto \Vert \nabla_1 k(z,x) \Vert $ is upper-bounded by an integrable function w.r.t. $\mu$ and $\nu$. Therefore, the differentiation lemma applies and  $\nabla f_{\mu,\nu}(z)$ is differentiable with gradient given by \cref{eq:gradient_witness}.
	
	To prove the second statement, we will consider two optimal couplings: $\pi_1$ with marginals $\mu$ and $\mu'$ and $\pi_2$ with marginals $\nu$ and $\nu'$.  We use \cref{eq:gradient_witness} to write:
	\begin{align*}
		\Vert \nabla f_{\mu,\nu}(z) - \nabla f_{\mu',\nu'}(z')\Vert 
		&= \left\Vert \mathbb{E}_{\pi_1}\left[ \nabla_1 k(z,x)-\nabla_1 k(z',x') \right] - \mathbb{E}_{\pi_2}\left[\nabla_1 k(z,y)-\nabla_1 k(z',y')\right] \right\Vert\\
		& \leq
		\mathbb{E}_{\pi_1}\left[ \left\Vert  \nabla_1 k(z,x)-\nabla_1 k(z',x') \right\Vert \right] + \mathbb{E}_{\pi_2}\left[\left\Vert  \nabla_1 k(z,y)-\nabla_1 k(z',y') \right\Vert \right] \\
		&\leq
		L\left( \Vert  z-z' \Vert + \mathbb{E}_{\pi_1}[\Vert  x-x' \Vert]  +  \Vert  z-z' \Vert + \mathbb{E}_{\pi_2}[\Vert  y-y' \Vert ] \right)\\
		&\leq L(2\Vert z-z'\Vert + W_2(\mu,\mu')  + W_2(\nu,\nu') )
	\end{align*}
	The second line is obtained by convexity while the third one uses \cref{assump:lipschitz_gradient_k} and finally the last line relies on $\pi_1$ and $\pi_2$ being optimal. The desired bound is obtained by further upper-bounding the last two terms by twice their amount.
\end{proof}

\begin{lemma}\label{lem:derivative_mmd_augmented}
Let $U$ be an open set, $q$ a probability distribution in $\mathcal{P}_2(\X \times \mathcal{U})$ and $\psi$ and $\phi$ two measurable maps from  $\X \times\mathcal{U} $ to $\X$  which are square-integrable w.r.t $q$. Consider the path $\rho_t$ from  $(\psi)_{\#}q$ and $(\psi+\phi)_{\#}q$ given by: $\rho_t=  (\psi+t\phi)_{\#}q \quad \forall t\in [0,1]$. Under \cref{assump:lipschitz_gradient_k}, $\mathcal{F}(\rho_t)$ is differentiable in $t$ with
	\begin{align*}
		\dot{\F}(\rho_t)&=\int \nabla f_{\mu,\rho_t}(\psi(x,u)+t\phi(x,u)) \phi(x,u)\diff q(x,u)
	\end{align*}
where $f_{\mu,\rho_t}$ is the unnormalised witness function between $\mu$ and $\rho_t$ as defined in \cref{eq:witness_function}.	
Moreover:
\begin{align*}
		\left\vert \dot{\F}(\rho_t) - \dot{\F}(\rho_s) \right\vert \leq 3L\left\vert t-s \right\vert\int \left\Vert \phi(x,u) \right\Vert^2 dq(x,u)
\end{align*}
\end{lemma}
\begin{proof}
For simplicity, we write $f_t$ instead of $f_{\mu,\rho_t}$ and denote by $s_t(x,u)= \psi(x,u)+t\phi(x,u)$
The function $h: t\mapsto k(s_t(x,u),s_t(x',u')) - k(s_t(x,u),z) - k(s_t(x',u'),z)$ is differentiable for all $(x,u)$,$(x',u')$ in $\X\times \mathcal{U}$ and $z\in \X$. 
Moreover, by \cref{assump:lipschitz_gradient_k}, a simple computation shows that for all $0\leq t\leq 1$:%
\[
\left\vert \dot{h} \right\vert \leq L\left[ \left(\left\Vert z - \phi(x,u)\right\Vert + \left\Vert \psi(x,u)\right\Vert\right) \left\Vert \phi(x',u')\right\Vert +  
\left(\left\Vert z - \phi(x',u')\right\Vert + \left\Vert \psi(x',u')\right\Vert \right)\left\Vert \phi(x,u)\right\Vert \right]
\]
The right hand side of the above inequality is integrable when $z$,  $(x,u)$ and  $(x',u')$ are independent and such that $z\sim \mu$ and both $(x,u)$ and $(x',u')$ are distributed according to $q$. Therefore, by the differentiation lemma \cite[Theorem 6.28]{Klenke:2008} it follows that $\F(\rho_t)$ is differentiable and:
\begin{align}\label{eq:time_derivative_mmd}
\dot{\F}(\rho_t) = \mathbb{E}\left[(\nabla_1 k(s_t(x,u),s_t(x',u'))-\nabla_1 k(s_t(x,u),z)).\phi(x,u)\right].
\end{align}
By \cref{prop:grad_witness_function}, we directly get $\dot{\F}(\rho_t) = \int \nabla f_{\mu,\rho_t}(\psi(x,u)+t\phi(x,u)) \phi(x,u)\diff q(x,u)$.
 We shall control now the difference $\vert \dot{F}(\rho_t)-\dot{\F}(\rho_{t'})\vert$ for $0\leq t,t'\leq 1$. Using \cref{assump:lipschitz_gradient_k} and recalling that $s_t(x,u)-s_{t'}(x,u)= (t-t')\phi(x,u)$ a simple computation shows:
\begin{align*}
	\left\vert\dot{\F}(\rho_t)-\dot{\F}(\rho_{t'}) \right\vert 
	&\leq L\left\vert t-t' \right\vert \mathbb{E}\left[\left(2\Vert \phi(x,u) \Vert + \Vert \phi(x',u')\Vert \right)\Vert \phi(x,u)\Vert  \right]\\
	  &\leq L\vert t-t'\vert(2\mathbb{E}\left[\Vert \phi(x,u)\Vert^2 \right]  + \mathbb{E}\left[\Vert \phi(x,u)\Vert \right]^2)\\
	 &\leq 3L\vert t-t'\vert\int\Vert \phi(x,u)\Vert^2 \diff q(x,u).
\end{align*}
which gives the desired upper-bound.
\end{proof}

We denote by $(x,y)\mapsto H_1 k(x,y)$ the Hessian of $x\mapsto k(x,y)$ for all $y\in \X$ and by $(x,y)\mapsto \nabla_1\nabla_2 k(x,y)$ the upper cross-diagonal block of the hessian of $(x,y)\mapsto k(x,y)$.  
\begin{lemma}\label{lem:second_derivative_augmented_mmd}
Let $q$ be a probability distribution in $\mathcal{P}_2(\X \times \X)$ and $\psi$ and $\phi$ two measurable maps from  $\X \times\X $ to $\X$  which are square-integrable w.r.t $q$. Consider the path $\rho_t$ from  $(\psi)_{\#}q$ and $(\psi+\phi)_{\#}q$ given by: $\rho_t=  (\psi+t\phi)_{\#}q \quad \forall t\in [0,1]$. Under \cref{assump:diff_kernel,assump:lipschitz_gradient_k}, $\mathcal{F}(\rho_t)$ is twice differentiable in $t$ with
	\begin{align*}
		\ddot{\F}(\rho_t)=&\mathbb{E}\left[\phi(x,y)^T\nabla_1 \nabla_2 k(s_t(x,y),s_t(x',y')) \phi(x',y')\right] \\
		&+ \mathbb{E}\left[\phi(x,y)^T (H_1k(s_t(x,y),y_t')-H_1k(s_t(x,y),z)) \phi(x,y)\right]
	\end{align*}
where $(x,y)$ and $(x',y')$ are independent samples from $q$, $z$ is a sample from $\mu$ and  $s_t(x,y)= \psi(x,y)+t\phi(x,y)$.
Moreover, if \cref{assump:bounded_fourth_oder} also holds then:
\begin{align*}
		\ddot{\F}(\rho_t) \geq \mathbb{E}\left[\phi(x,y)^T\nabla_1 \nabla_2 k(s_t(x,y),s_t(x',y')) \phi(x',y')\right] - \sqrt{2}\lambda d \F(\rho_t)^{\frac{1}{2}}\mathbb{E}[\Vert \phi(x,y) \Vert^2]  
\end{align*}
where we recall that $\X\subset \mathbb{R}^d$.
\end{lemma}
\begin{proof}
The first part is similar to \cref{lem:derivative_mmd_augmented}. In fact we already know by \cref{lem:derivative_mmd_augmented} that $\dot{\F}(\rho_t)$ exists and is given by:
\[
\dot{\F}(\rho_t) = \mathbb{E}\left[(\nabla_1 k(s_t(x,y),s_t(x',y'))-\nabla_1 k(s_t(x,y),z)).\phi(x,y)\right]
\]
Define now the function $\xi : t\mapsto (\nabla_1 k(s_t(x,y),s_t(x',y'))-\nabla_1 k(s_t(x,y),z)).\phi(x,y)$ which is differentiable for all $(x,y)$,$(x',y')$ in $\X\times \X$ and $z\in \X$ by \cref{assump:diff_kernel}. Moreover, its time derivative is given by:
\begin{align}
	\dot{\xi} =& \phi(x',y')^T \nabla_2\nabla_1k(s_t(x,y),s_t(x',y'))\phi(x,y) \\
	&+ \phi(x,y)^T(H_1k(s_t(x,y),s_t(x',y') ) - H_1k(s_t(x,y),z ))\phi(x,y)  
\end{align}
By \cref{assump:lipschitz_gradient_k} it follows in particular that $\nabla_2\nabla_1k$ and $H_1k$ are bounded hence $\vert \dot{\xi} \vert$  is upper-bounded by $ (\Vert \phi(x,y) \Vert + \Vert\phi(x',u') \Vert)\Vert \phi(x,y)\Vert$ which is integrable.
Therefore, by the differentiation lemma \cite[Theorem 6.28]{Klenke:2008} it follows that $\dot{\F}(\rho_t)$ is differentiable and $\ddot{\F}(\rho_t) = \mathbb{E}\left[\dot{\xi}\right].$
We prove now the second statement. Bu the reproducing property, it is easy to see that the last term in the expression of $\dot{\xi}$ can be written as:
\[
\langle \phi(x,y)^TH_1 k(s_t(x,y),.)\phi(x,y), k(s_t(x',y'),.)-  k(z,.)\rangle_{\kH} 
\]
Now, taking the expectation w.r.t $x'$  ,$y'$ and $z$ which can be exchanged with the inner-product in $\kH$ since  $(x',y',z)\mapsto k(s_t(x',y'),.)-  k(z,.)$ is Bochner integrable \cite[Definition 1, Theorem 6]{Retherford:1978} and recalling that such integral is given by $f_{\mu,\rho_t}$  one gets the following expression:
\[
 \langle  \phi(x,y)^TH_1 k(s_t(x,y),.)\phi(x,y), f_{\mu,\rho_t} \rangle_{\kH}
\]
Using Cauchy-Schwartz and \cref{assump:bounded_fourth_oder} it follows that:
\[
\vert \left\langle  \phi(x,y)^TH_1 k(s_t(x,y),.)\phi(x,y), f_{\mu,\rho_t} \right\rangle_{\kH}\vert \leq \lambda d\Vert \phi(x,y)\Vert^2 \Vert f_{\mu,\rho_t}\Vert 
\]
One then concludes using the expression of $\ddot{\F}(\rho_t)$ and recalling that $\F(\rho_t) = \frac{1}{2}\Vert f_{\mu,\rho_t} \Vert^2$.
\end{proof}

\begin{lemma}\label{lem:integral_lambda_convexity}
Assume that for any geodesic $(\rho_{t})_{t\in[0,1]}$ between
$\rho_{0}$ and $\rho_{1}$ in $\mathcal{P}(\X)$ with velocity vectors $(V_t)_{t \in [0,1]}$ the following holds:
\[
\ddot{\F}(\rho_{t}) \geq \Lambda(\rho_t,V_t)
\]
for some admissible functional $\Lambda$ as defined in \cref{def:conditions_lambda}, then:
\begin{align*}
\F(\rho_{t})\leq(1-t)\F(\rho_{0})+t\F(\rho_{1})-\int_{0}^{1}\Lambda(\rho_{s},V_{s})G(s,t)ds
\end{align*}
with $G(s,t)=s(1-t) \mathbbm{1}\{s\leq t\}
+t(1-s) \mathbbm{1}\{s\geq t\}$ for $0\leq s,t\leq 1$.

\end{lemma}
\begin{proof}
	This is a direct consequence of the general identity (\cite{Villani:2009},
	Proposition 16.2). Indeed, for any continuous function $\phi$ on
	$[0,1]$ with second derivative $\ddot{\phi}$ that is bounded below
	in distribution sense the following identity holds:
	\[
	\phi(t)=(1-t)\phi(0)+t\phi(1)-\int_{0}^{1}\ddot{\phi}(s)G(s,t)ds.
	\]
	This holds a fortiori for $\F(\rho_{t})$ since $\F$ is smooth. By assumption, we have that $\ddot{\F}(\rho_{t}) \geq \Lambda(\rho_t,V_t)$, hence, it follows that:
	\[
	\F(\rho_{t})\leq(1-t)\F(\rho_{0})+t\F(\rho_{1})-\int_{0}^{1}\Lambda(\rho_{s},V_{s})G(s,t)ds.
	\]
\end{proof}

\begin{lemma}\label{lem:mixture_convexity}[Mixture convexity]
	The functional $\F$ is mixture convex: for any probability distributions $\nu_1$ and $\nu_2$ and scalar $1\leq \lambda\leq 1$:
	\begin{align*}
	\F(\lambda \nu_1+(1-\lambda)\nu_2)\leq \lambda \F(\nu_1)+ (1-\lambda)\F(\nu_2)
	\end{align*}
\end{lemma}
\begin{proof}
	Let $\nu$ and $\nu'$ be two probability distributions and $0\leq \lambda\leq 1$. Expanding the RKHS norm in $\F$ it follows directly that:
	\[
		\mathcal{F}(\lambda \nu + (1-\lambda)\nu') -\lambda \mathcal{F}(\nu) -(1-\lambda)\mathcal{F}(\nu') = -\frac{1}{2}\lambda(1-\lambda)MMD(\nu,\nu')^2 \leq 0.
	\]
	which concludes the proof.
\end{proof}

\begin{lemma}
\label{lem:Discrete-Gronwall-lemma}[Discrete Gronwall lemma]
Let $a_{n+1}\leq(1+\gamma A)a_{n}+b$ with $\gamma>0$, $A>0$,
$b>0$ and $a_0=0$, then: 
\[
a_{n}\leq\frac{b}{\gamma A}(e^{n\gamma A}-1).
\]
\end{lemma}
\begin{proof}
	Using the recursion, it is easy to see that for any $n>0$:
	\[
	a_n \leq (1+\gamma A)^n a_0 + b\left(\sum_{i=0}^{n-1}(1+\gamma A )^{k}\right) 
	\]
	One concludes using the identity $\sum_{i=0}^{n-1}(1+\gamma A )^{k} =\frac{1}{\gamma A}((1+\gamma A)^{n} -1)$ and recalling that $(1+\gamma A)^{n} \leq e^{n\gamma A}$.
\end{proof}

\end{document}